\documentclass[sigconf]{acmart}

\newcommand{\bu}{\boldsymbol{u}}
\newcommand{\bV}{\boldsymbol{V}}
\newcommand{\bv}{\boldsymbol{v}}
\newcommand{\bU}{\boldsymbol{U}}
\newcommand{\bX}{\boldsymbol{X}}

\newcommand{\bpi}{\boldsymbol{\pi}}
\newcommand{\bPi}{\boldsymbol{\Pi}}
\newcommand{\kl}{KŁ\xspace}
\newcommand{\myparagraph}[1]{\vspace{0.3\baselineskip}\noindent{\textbf{#1.}}~}

\newcommand{\bz}{\boldsymbol{z}}

\usepackage{mathtools}

\DeclareMathOperator*{\argmin}{argmin}
\newcommand{\bx}{\boldsymbol{x}}
\newcommand{\ml}{\mathcal{L}}
\newcommand{\fedapm}{\texttt{FedAPM}\xspace}
\newcommand{\fedalt}{\texttt{FedAlt}\xspace}
\newcommand{\fedsim}{\texttt{FedSim}\xspace}
\newcommand{\fedprox}{\texttt{FedProx}\xspace}
\newcommand{\fedavg}{\texttt{FedAvg}\xspace}
\newcommand{\cifar}{\textsf{CIFAR10}\xspace}
\newcommand{\crisis}{\textsf{CrisisMMD}\xspace}
\newcommand{\kuhar}{\textsf{KU-HAR}\xspace}
\newcommand{\crema}{\textsf{Crema-D}\xspace}
\newtheoremstyle{mydefinition} 
  {1pt} 
  {1pt} 
  {\itshape} 
  {} 
  {\bfseries} 
  {.} 
  { } 
  {} 

\theoremstyle{mydefinition}

\newtheorem{theorem}{Theorem}

\newtheorem{proposition}{Proposition}
\newtheorem{definition}{Definition}
\newtheorem{assumption}{Assumption}
\newtheorem{lemma}{Lemma}
\usepackage[linesnumbered,ruled,vlined,boxed,noend]{algorithm2e}

\usepackage{amsmath}
\usepackage{amsthm}
\usepackage{enumitem}
\usepackage{multirow}

\usepackage{xcolor}
\usepackage{colortbl} 
\usepackage{ragged2e}

\usepackage{etoolbox}

\setlist[itemize]{leftmargin=*}
\AtBeginDocument{%
  }

\copyrightyear{2025}
\acmYear{2025}
\setcopyright{acmlicensed}\acmConference[KDD '25]{Proceedings of the 31st ACM SIGKDD Conference on Knowledge Discovery and Data Mining V.2}{August 3--7, 2025}{Toronto, ON, Canada}
\acmBooktitle{Proceedings of the 31st ACM SIGKDD Conference on Knowledge Discovery and Data Mining V.2 (KDD '25), August 3--7, 2025, Toronto, ON, Canada}
\acmDOI{10.1145/3711896.3736954}
\acmISBN{979-8-4007-1454-2/2025/08}

\begin{document}


\title{FedAPM: Federated Learning via ADMM with Partial Model Personalization}


\settopmatter{authorsperrow=3}
\author{Shengkun Zhu}
\orcid{0009-0009-2850-1505}
\affiliation{%
  \institution{School of Computer Science\\Wuhan University}
  \city{Wuhan}
  \country{China}
}
\email{whuzsk66@whu.edu.cn}

\author{Feiteng Nie}
\orcid{}
\affiliation{%
  \institution{School of Computer Science\\Wuhan University}
  \city{Wuhan}
  \country{China}
}
\email{niefeiteng@whu.edu.cn}

\author{Jinshan Zeng$^*$}
\orcid{0000-0003-1719-3358}
\affiliation{%
  \institution{School of Management\\Xi'an Jiaotong University}
  \city{Xi'an}
  \country{China}
}
\email{jsh.zeng@gmail.com}

\author{Sheng Wang$^*$}
\orcid{0000-0002-5461-4281}
\affiliation{%
  \institution{School of Computer Science\\Wuhan University}
  \streetaddress{}
  \city{Wuhan}
  \country{China}
}
\email{swangcs@whu.edu.cn}

\author{Yuan Sun}
\orcid{0000-0003-2911-0070}
\affiliation{%
  \institution{La Trobe Business School\\La Trobe University}
  \city{Melbourne}
  \country{Australia}
}
\email{yuan.sun@latrobe.edu.au}

\author{Yuan Yao}
\orcid{}
\affiliation{%
  \institution{Hong Kong University of Science and Technology}
  \country{Hong Kong, China}
}
\email{yuany@ust.hk}

\author{Shangfeng Chen}
\affiliation{%
  \institution{School of Computer Science\\Wuhan University}
  \city{Wuhan}
  \country{China}
}
\email{brucechen@whu.edu.cn}
\thanks{*Corresponding authors: Sheng Wang and Jinshan Zeng}

\author{Quanqing Xu}
\affiliation{%
  \institution{Oceanbase\\Ant group}
  \city{Hangzhou}
  \country{China}
}
\email{xuquanqing.xqq@oceanbase.com}

\author{Chuanhui Yang}
\affiliation{%
  \institution{Oceanbase\\Ant group}
  \city{Hangzhou}
  \country{China}
}
\email{rizhao.ych@oceanbase.com}



\renewcommand{\shortauthors}{Shengkun Zhu et al.}
\begin{abstract}
In federated learning (FL), the assumption that datasets from different devices are independent and identically distributed (i.i.d.) often does not hold due to user differences, and the presence of various data modalities across clients makes using a single model impractical.
Personalizing certain parts of the model can effectively address these issues by allowing those parts to differ across clients, while the remaining parts serve as a shared model.
However, we found that \textit{partial model personalization may} \textit{exacerbate} \textit{client drift} (each client's local model diverges from the shared model), thereby reducing the effectiveness and efficiency of FL algorithms.
We propose an FL framework based on the alternating direction method of multipliers (ADMM), referred to as \fedapm, to mitigate client drift.
We construct the augmented Lagrangian function by incorporating first-order and second-order proximal terms into the objective, with the second-order term providing fixed correction and the first-order term offering compensatory correction between the local and shared models.
Our analysis demonstrates that \fedapm, by using explicit estimates of the Lagrange multiplier, is more stable and efficient in terms of convergence compared to other FL frameworks.
We establish the global convergence of \fedapm training from arbitrary initial points to a stationary point, achieving three types of rates: constant, linear, and sublinear, under mild assumptions.
We conduct experiments using four heterogeneous and multimodal datasets with different metrics to validate the performance of \fedapm.
Specifically, \fedapm achieves faster and more accurate convergence, outperforming the SOTA methods with average improvements of 12.3\% in test accuracy, 16.4\% in F1 score, and 18.0\% in AUC while requiring fewer communication rounds.

\end{abstract}

\ccsdesc[500]{Computing methodologies~Distributed artificial intelligence}
\ccsdesc[300]{Computing methodologies~Distributed algorithms}
\ccsdesc[100]{Computing methodologies~Regularization}


\keywords{Federated learning, partial model personalization, ADMM, client drift, global convergence.}


\maketitle
\section{Introduction}
With the widespread use of mobile devices, vast amounts of data are generated that fuel various machine learning applications \cite{Liang2018IoT, Wang2023Apache, Piyush2021Query}.
However, traditional cloud-based training methods face issues of privacy leakage and data transmission costs \cite{Peter2021Distributed, Smith2017CoCoA}. 
Additionally, with the implementation of privacy regulations like the California Consumer Privacy Act (CCPA) \cite{Wikipedia1} and General Data Protection Regulation (GDPR) \cite{Wikipedia2}, protecting user data has become even more critical. 
Federated learning (FL) is a distributed machine learning framework that protects privacy by enabling local model training without the need for data centralization \cite{mcmahan2017communication, Li2020Challenges}, addressing challenges like data silos and insufficient samples \cite{Li2022Experimental, Kairouz2021Advances}, and fostering advancements in various intelligent applications \cite{Ye2024OpenFedLLM, Kuang2024FederatedScope}.

In FL, it is typically assumed that the datasets from various devices are sampled from the same or highly similar distributions \cite{Li2020fedprox, Li2022Experimental, Luo2021Fear}. 
However, due to the distinct characteristics of users and the growing prominence of personalized on-device services, the assumption of independent and identically distributed (i.i.d.) data often does not hold in practical scenarios \cite{Ye2024Heterogeneous, zhao2018federated, Sattler2020Robust}. 
This leads to statistical heterogeneity, which can significantly impact the effectiveness of FL algorithms \cite{Ye2024Heterogeneous}.
Moreover, different clients may have data in various modalities \cite{Feng2023FedMultimodal, Chen2024FedMBridge, Chen2022FedMSplit}, such as images, videos, text, and audio, each requiring distinct models like convolutional neural networks (CNNs) for images \cite{Yann2015Deep} and recurrent neural networks (RNNs) for audio \cite{goodfellow2016deep}.
Therefore, relying on a single model for all cases is ineffective and impractical \cite{Krishna2022Partial, Sun2023FedPerfix}.

\textit{Partial model personalization} \cite{Krishna2022Partial} is a method designed to address the challenge of heterogeneous and multimodal data by dividing the model parameters into two groups: shared and personalized model parameters.
The dimensions of personalized models can vary among clients, enabling personalized components to differ in the number of parameters or even in their architecture, while the shared model is common among clients and maintains the same structure.
However, while personalized models can effectively address the impact of heterogeneous and multimodal data, the shared model cannot, leading to the common saying: \textit{it started with a bang and ended with a whimper}.
This limitation arises because each client's local objectives differ. When solving for the shared model, the result is often a stationary point with respect to the local objective, rather than a stationary point with respect to the global objective.
This issue is commonly referred to as \textit{client drift} \cite{Praneeth2020SCAFFOLD}.

Several studies consider using full model personalization to address client drift \cite{Dinh2020pfedme, Li2021Ditto, Lin2022Personalized}, which involves personalizing all model parameters. However, each model requires twice the memory footprint of the full model, which limits the size of trainable models.
Moreover, \citet{Krishna2022Partial} proposed that full model personalization may not be essential for modern deep learning architectures, which consist of multiple simple functional units often structured in layers or other complex interconnected designs. 
Focusing on personalizing the appropriate components, guided by domain expertise, can yield significant advantages with only a minimal increase in memory requirements. 
Current partial model personalization methods can be categorized into two types based on their iterative approach: one using Gauss-Seidel iteration \cite{Liam2021Exploiting, Krishna2022Partial, Singhal2021Federated}, and the other employing Jacobi iteration \cite{Liang2020Think, arivazhagan2019federated, hanzely2021personalized}, referred to as \fedalt and \fedsim, respectively.
\citet{Krishna2022Partial} theoretically demonstrated the convergence of both methods and experimentally validated that \fedalt achieves better test performance than \fedsim.
However, these methods do not fundamentally address client drifts, as multiple local updates can still lead to deviations of each client's local model from the optimal representation.
Moreover, we find that partial model personalization sometimes even exacerbates client drift, as shown in Figure \ref{fig: example_client_drift}.
The local updates in \fedalt and \fedsim, compared to \fedavg (a de facto FL framework without personalization) \cite{mcmahan2017communication}, bring the local model closer to the optimal local model while moving it further from the optimal shared model.

\begin{figure}
  \centering
  \includegraphics[width=\linewidth]{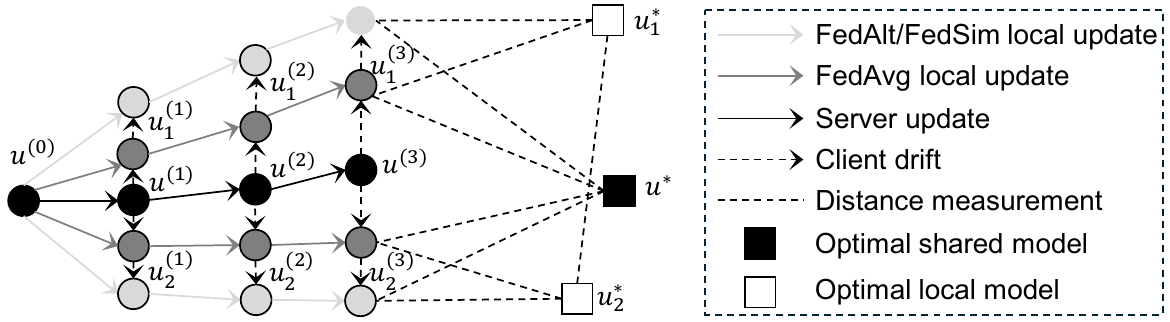}
  \caption{Client drift in \fedavg \cite{mcmahan2017communication} and \fedalt/\fedsim \cite{Krishna2022Partial} is illustrated using two clients with three local update steps. Partial model personalization reduces the gap between \(\bu_i\) and \(\bu_i^*\), but increases the gap between \(\bu_i\) and \(\bu^*\). This may cause the shared model \(\bu\) to drift further from \(\bu^*\), as \(\bu = (\bu_1 + \bu_2)/2\).}
  \label{fig: example_client_drift}
\end{figure}


Our main contribution in this paper is to address the issue that partial model personalization can exacerbate client drift.
Specifically, we propose an FL framework based on the alternating direction method of multipliers (ADMM): \fedapm.
We construct the augmented Lagrangian function by incorporating first-order and second-order proximal terms into the objective. 
The second-order term offers a fixed correction, while the first-order term provides a compensatory adjustment to fine-tune the local update.
Moreover, we demonstrate that \fedapm is more stable and efficient in terms of convergence compared to \fedalt and \fedsim, which results in improved communication efficiency. Specifically, we show that \fedalt and \fedsim can be interpreted as using the inexact penalty method, while \fedapm adopts ADMM, an augmented Lagrangian method (also known as the exact penalty method). 
In \fedapm, explicit estimates of the Lagrange multiplier are used to address the ill-conditioning that is inherent in quadratic penalty functions \cite{Numerical2006Wright}.

We theoretically establish the global convergence of \fedapm under weaker assumptions than those used in \cite{hanzely2021personalized, Krishna2022Partial, Liam2021Exploiting}.
Specifically, we establish global convergence with constant, linear, and sublinear convergence rates based on the \textit{Kurdyka-Łojasiewicz (KŁ)} \cite{kurdyka1998gradients} \textit{inequality framework}, as formulated in \cite{Attouch2013Convergence, wang2019global}. 
Our convergence analysis differs from existing work \cite{Attouch2013Convergence,wang2019global} in several key ways, enabling us to achieve the previously mentioned convergence results.
According to \cite{Attouch2013Convergence,wang2019global}, the conditions of \textit{sufficient descent}, \textit{relative error}, \textit{continuity}, and the \textit{KŁ property} guarantee the global convergence of a nonconvex algorithm. 
In contrast, we demonstrate sufficient descent and relative error conditions within an inexact ADMM optimization framework, while existing theories \cite{Attouch2013Convergence,wang2019global} require that all subproblems in ADMM be solved exactly. 
Our theoretical contributions are also of value to the optimization community.

Our experimental contributions demonstrate that partial model personalization can exacerbate client drift, while \fedapm effectively alleviates this issue.
Specifically, we validate \fedapm on several real-world datasets (including heterogeneous and multimodal data) using three partial model personalization strategies, four baselines, and six metrics.
We demonstrate that \fedapm improves the overall performance of the SOTA methods, resulting in an average improvement of 12.3\%, 16.4\%, and 18.0\% in testing accuracy, F1 score, and AUC, respectively. 
In terms of communication efficiency, \fedapm requires fewer communication rounds to converge to a lower loss.
Moreover, we explore the impact of hyperparameters on \fedapm and provide adjustment strategies based on empirical findings.


\section{Related Work}
Since partial model personalization is a special case of personalized FL approaches, we first review the existing personalized FL methods.
Subsequently, given that ADMM is a primal-dual-based optimization framework, we review the current research that utilizes primal-dual-based methods as solvers for FL.

\myparagraph{Personalized FL}
According to the model partitioning strategies, PFL can be categorized into two types: full model personalization and partial model personalization \cite{Krishna2022Partial}. 
Since the former is not the focus of this paper, interested readers can refer to \cite{Dinh2020pfedme, Li2021Ditto, zhu2024admm, Lin2022Personalized}. 
Here, we will only review partial model personalization methods.
Based on the iterative methods (Gauss-Seidel and Jacobi), partial model personalization can be divided into two categories: \fedalt \cite{Liam2021Exploiting, Singhal2021Federated} and \fedsim \cite{Liang2020Think, arivazhagan2019federated, hanzely2021personalized}.
These methods proposed different schemes for personalizing model layers.
\citet{Liang2020Think} and \citet{Liam2021Exploiting} proposed to personalize the input layers to learn a personalized representation, while \citet{arivazhagan2019federated} proposed learning a shared representation across various tasks through output layer personalization. 
However, regardless of which part of the model is personalized, the shared part still suffers from client drift, and this phenomenon can sometimes be exacerbated, leading to a decline in overall model performance.
In terms of theory, \citet{Krishna2022Partial} provided convergence guarantees for these methods, demonstrating a sublinear convergence rate under the assumptions of \textit{smoothness, bounded variance, and partial gradient diversity}, which are often too strong in practical scenarios.

\myparagraph{Primal-dual-based FL}
Existing FL frameworks can be classified into three categories based on the type of variables being solved: primal-based FL \cite{mcmahan2017communication, Li2020fedprox, Praneeth2020SCAFFOLD}, dual-based FL \cite{Ma2015Adding, Smith2017CoCoA, Smith2017Federated}, and primal-dual-based FL \cite{Zhou2023FedADMM, Gong2022FedADMM, Zhou2023FedGiA, zhang2021fedpd, Heejoo2024FedAND, Wang2022FedADMM, zhu2024admm}.
Primal-based FL solves the primal problem of FL, which makes it easy to implement \cite{zhu2024admm}. 
Consequently, it has become one of the most widely used FL frameworks \cite{Liu2021FATE, caldas2018leaf, Kallista2019Towards}.
In contrast, dual-based FL solves dual problems and has been shown to have better convergence than primal-based FL \cite{Shai2016Accelerated, Shai2013Stochastic}. 
However, dual-based FL is only applicable to convex problems.
In recent years, primal-dual-based FL has gained more attention due to its advantages from both primal-based and dual-based methods.
ADMM, as an exact penalty method \cite{boyd2011distributed, Zeng2021ADMM}, has been applied in FL optimization in recent years \cite{Zhou2023FedADMM}. 
It has been shown to offer advantages in terms of convergence \cite{Zhou2023FedGiA} and alleviating data heterogeneity \cite{Gong2022FedADMM}.
However, these studies do not explain why applying ADMM in FL outperforms other FL frameworks, a question we will discuss from an optimization perspective.

\section{Preliminaries}
This section presents the notations used throughout the paper and provides detailed definitions of both FL and ADMM.
Moreover, we present three types of partial model personalization methods.
\subsection{Notations}
\begin{table}
\centering
\setlength\tabcolsep{10pt}
\caption{Summary of notations}
\vspace{-1em}
\label{tab:notation}
\begin{tabular}{ll}
\toprule
\textbf{Notations} & \textbf{Description}\\
\midrule
$\bX_i$, $i\in[m]$ & The local dataset \\
$\bv_i$, $i\in[m]$ & The personalized model\\
$\bu_i$, $i\in[m]$ & The local model\\
$\bpi_i$, $i\in[m]$ & The dual variable\\
$\alpha_i$, $i\in[m]$ & The weight parameter\\
$\bu$ & The shared model\\
$\rho$ & The penalty parameter\\
$\mathcal{S}^t$ & The selected clients set in the $t$-th iteration\\
\bottomrule
\end{tabular}
\end{table}
We use different text-formatting styles to represent different mathematical concepts: plain letters for scalars, bold letters for vectors, and capitalized letters for matrices. For instance, $m$ represents a scalar, $\boldsymbol{v}$ represents a vector, and $\bV$ denotes a matrix. Without loss of generality, all training models in this paper are represented using vectors.
We use $[m]$ to represent the set $\{1, 2, ..., m\}$. We use ``$:=$" to indicate a definition, while $\mathbb{R}^d$ represents the $d$-dimensional Euclidean space. We represent the inner product of vectors, such as $\langle\boldsymbol{u},\boldsymbol{v}\rangle$, as the sum of the products of their corresponding elements. We use $||\cdot||$ to denote the Euclidean norm of a vector.
Table \ref{tab:notation} enumerates the notations used in this paper along with the description.

\subsection{Federated Learning}
In an FL scenario involving $m$ clients, each client $i$ holds a local dataset $\bX_i$ consisting of $n_i$ data samples drawn from the distribution $\mathcal{D}_i$. These clients collaborate through a central server to jointly train a model \( \bu \) that minimizes empirical risk as follows \cite{mcmahan2017communication}:
\begin{equation}\label{eq:problem of FL}
    \min_{\bu}\left\{\sum_{i=1}^{m}\alpha_i f_i(\bu):=\sum_{i=1}^{m}\alpha_i\mathbb{E}_{\bx\sim\mathcal{D}_i}[\ell_i(\bu;\bx)]\right\},
\end{equation}
where $\alpha_i$ is a weight parameter, typically chosen as either $1/m$ or $n_i/n$, where $n = \sum_{i=1}^m n_i$ represents the total data number, $\bx$ denotes a random sample from $\mathcal{D}_i$, $\ell_i(\bu; \bx)$ is the loss function for the model $\bu$ with respect to the sample $\bx$, and $f_i(\bu) := \mathbb{E}_{\bx \sim \mathcal{D}_i}[\ell_i(\bu; \bx)]$ represents the expected loss over the data distribution $\mathcal{D}_i$. In this paper, we consider $f_i(\bu)$ to be possibly non-convex.

\subsection{Alternating Direction Method of Multipliers}
ADMM is an optimization method within the augmented Lagrangian framework, ideally suited for addressing the following problem \cite{boyd2011distributed}:
\begin{align*}
    \min_{\bu\in\mathbb{R}^{r},\boldsymbol{v}\in\mathbb{R}^q}f(\bu)+g(\boldsymbol{v}),\quad\text{s.t. } A\bu+B\boldsymbol{v}-\boldsymbol{b}=\boldsymbol{0},
\end{align*}
where $A\in\mathbb{R}^{p\times r}$, $B\in\mathbb{R}^{p\times q}$, and $\boldsymbol{b}\in\mathbb{R}^{p}$. We directly give the augmented Lagrangian function of the problem as follows,
\begin{equation*}
    \ml(\bu,\boldsymbol{v},\bpi):=f(\bu)+g(\boldsymbol{v})+\langle\bpi,A\bu+B\boldsymbol{v}-\boldsymbol{b}\rangle+\frac{\rho}{2}||A\bu+B\boldsymbol{v}-\boldsymbol{b}||^2,
\end{equation*}
where $\bpi\in\mathbb{R}^p$ is the dual variable, and $\rho>0$ is the penalty parameter. After initializing the variables with $(\bu^0,\boldsymbol{v}^0,\bpi^0)$, ADMM iteratively performs the following steps:
\begin{equation*}
\left\{
\renewcommand{\arraystretch}{1.4} 
\begin{array}{l}
\bu^{t+1} = \operatorname{argmin}_{\bu \in \mathbb{R}^r} \mathcal{L}(\bu, \boldsymbol{v}^t, \boldsymbol{\pi}^t), \\
\boldsymbol{v}^{t+1} = \operatorname{argmin}_{\boldsymbol{v} \in \mathbb{R}^q} \mathcal{L}(\bu^{t+1}, \boldsymbol{v}, \boldsymbol{\pi}^t), \\
\boldsymbol{\pi}^{t+1} = \boldsymbol{\pi}^t + \rho (A \bu^{t+1} + B \boldsymbol{v}^{t+1} - \boldsymbol{b}).
\end{array}
\right.
\end{equation*}
ADMM offers effective distributed and parallel computing capabilities, efficiently addresses equality-constrained problems, and guarantees global convergence \cite{wang2019global}. This makes it especially suitable for large-scale optimization tasks and widely used in distributed computing \cite{boyd2011distributed} and machine learning \cite{Zeng2021ADMM}.

\begin{figure}
  \centering
  \includegraphics[width=\linewidth]{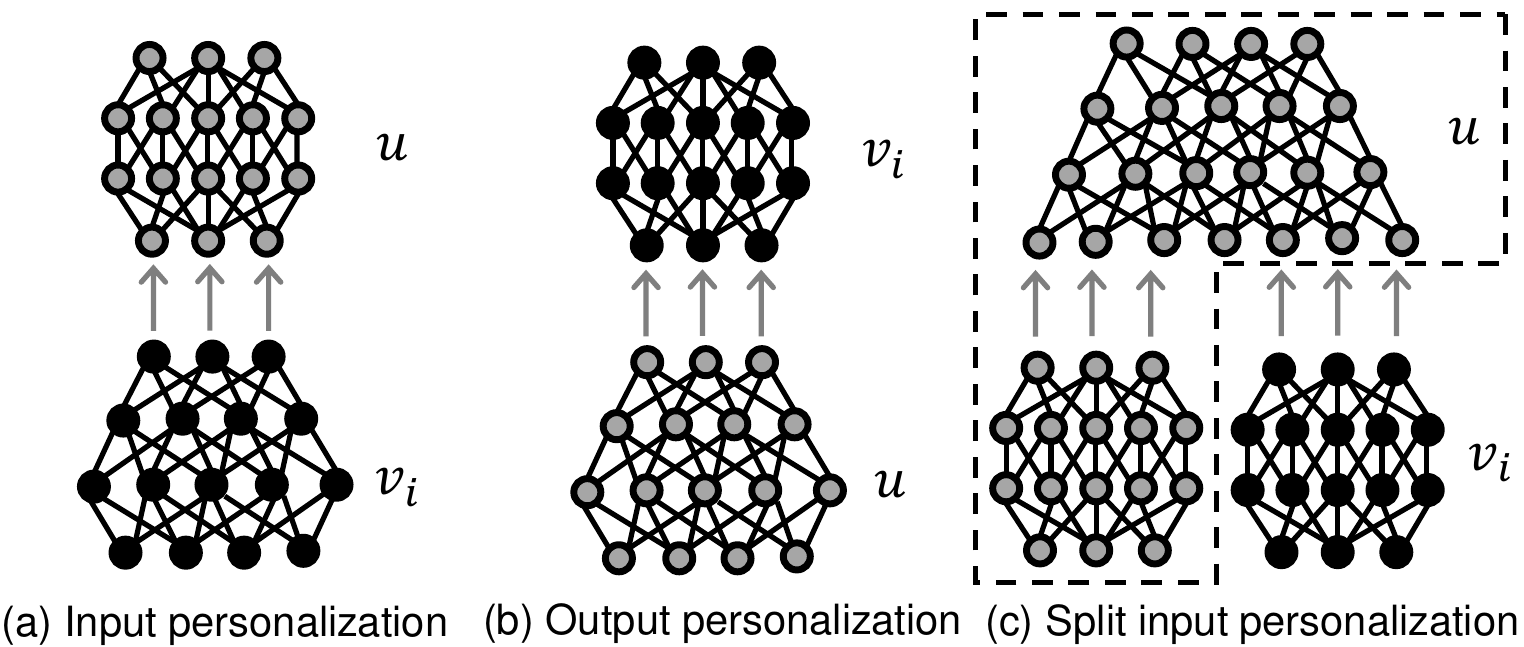}
  \caption{Three examples of partial model personalization in deep neural networks, where $\bv_i$ and $\bu$ represent the personalized and shared models, respectively.}
  \label{fig: three partial model}
\end{figure}

\subsection{Partial Model Personalization}
Partial model personalization allows clients to make personalized adjustments based on the shared global model.
\citet{Krishna2022Partial} categorized current methods for partial model personalization in deep neural networks into three types based on their applications:
\begin{itemize}
    \item \textbf{Input personalization}: The lower layers are trained locally, and the upper layers are shared among clients.
    \item \textbf{Output personalization}: The upper layers are trained locally, and the lower layers are shared among clients.
    \item \textbf{Split input personalization}: The input layers are divided horizontally into a shared and a personal part, which processes different portions of the input vector, and their outputs are concatenated before being passed to the upper layers of the model.
\end{itemize}
In Figure \ref{fig: three partial model}, we present examples of three partial model personalization strategies.
Input personalization can be seen as learning a shared representation among different clients \cite{Liang2020Think,Liam2021Exploiting}. 
Despite heterogeneous data with different labels, tasks may share common features, like those in various image types or word-prediction tasks \cite{Bengio2013Representation,Yann2015Deep}. 
After learning this representation, each client’s labels can be predicted with a linear classifier or shallow neural network \cite{Liam2021Exploiting}.
Output personalization is applicable to various tasks \cite{arivazhagan2019federated}, such as personalized image aesthetics \cite{Ren2017Personalized} and personalized highlight detection \cite{Molino2018PHD}, where explicit user features are absent from the data and must be inferred during training. As a result, the same input data may receive different labels from different users, indicating that personalized models must vary across users to accurately predict distinct labels for similar test data.
Split input personalization helps protect private user features by localizing their personalized embeddings on the device \cite{Krishna2022Partial}. Similar architectures have been proposed for context-dependent language models \cite{Mikolov2012Context}.

\section{Proposed \fedapm}
In this section, we begin by presenting the formulation of the optimization problem and defining the stationary points. 
Next, we offer a detailed algorithmic description of \fedapm. 
Finally, we discuss the advantages of \fedapm compared to other FL frameworks.
\subsection{Problem Formulation}
Let $\bV:=\{\bv_i\}_{i=1}^m$ be a set of personalized models and $\bu$ be the shared model. We consider solving the following optimization problem:
\begin{align}\label{eq: pfl problem1}
    \min_{\bV,\bu} \bigl\{f(\bV,\bu) := \sum_{i=1}^m \alpha_i f_i(\bv_i, \bu) \bigr\},
\end{align}
where $f_i(\bv_i, \bu):=\mathbb{E}_{\bx \sim \mathcal{D}_i}[\ell_i((\bv_i,\bu); \bx)]$ denotes the expected loss over the data distribution $\mathcal{D}_i$ for the combined model $(\bv_i, \bu)$.
The existing two types of methods for solving \eqref{eq: pfl problem1} are \fedalt and \fedsim \cite{Krishna2022Partial}, which are analogous to the Gauss-Seidel and Jacobi updates in numerical linear algebra \cite{Demmel1997Algebra}, respectively. However, we found that these two methods can exacerbate client drift in the shared model (see Figure \ref{fig: example_client_drift}, which will be validated in Section \ref{sec: experiments}).
To address this issue, we use ADMM to solve \eqref{eq: pfl problem1}, and we discuss the advantages of using ADMM in Section \ref{sec: discussion}.
We introduce auxiliary variables $\bU:=\{\bu_i\}_{i=1}^m$ to transform \eqref{eq: pfl problem1} into a separable form (with respect to a partition or splitting of the variable into multi-block variables):
\begin{equation}\label{eq: problem pfl2}
\begin{split}
    &\min_{\bV,\bU, \bu} \bigl\{f(\bV,\bU) := \sum_{i=1}^m \alpha_i f_i(\bv_i, \bu_i)\bigr\},\quad\text{ s.t. }\,\bu_i=\bu,\,i\in[m].
\end{split}
\end{equation}
Note that \eqref{eq: problem pfl2} is equivalent to \eqref{eq: pfl problem1} in the sense that the optimal solutions coincide (discussed in Section \ref{sec: stationary points}).
To apply ADMM for \eqref{eq: problem pfl2}, we define the augmented Lagrangian function as follows:
\begin{equation}\label{eq: lagrangian function}
\begin{split}
    \ml(\bV,\bU,\bPi,\bu)&:=\sum_{i=1}^m\ml_i(\bv_i,\bu_i,\bpi_i, \bu), \\
    \ml_i(\bv_i,\bu_i,\bpi_i,\bu)&:=\alpha_i f_i(\bv_i, \bu_i)+\langle \bpi_i, \bu_i-\bu\rangle + \frac{\rho}{2}\|\bu_i-\bu\|^2,
\end{split}
\end{equation} 
where $\Pi:=\{\bpi_i\}_{i=1}^{m}$ denotes the set of dual variables, and $\rho>0$ represents the penalty parameter. 
The ADMM algorithm to solve \eqref{eq: problem pfl2} can be outlined as follows: starting from arbitrary initialization $(\bV^0, \bU^0, \Pi^0, \bu^0)$, the update is iteratively performed for each $t \geq 0$:
\begin{align}\label{eq: admm update order}
\left\{
\begin{aligned}
\bv_i^{t+1}&=\operatorname{argmin}_{\bv_i}\ml_i(\bv_i,  \bu_i^t, \boldsymbol{\pi}^t_i,\bu^{t}),\\ 
\bu_i^{t+1} & =\operatorname{argmin}_{\bu_i} \ml_i(\bv_i^{t+1}, \bu_i, \boldsymbol{\pi}^t_i, \bu^{t}),\\ 
\boldsymbol{\pi}_i^{t+1} & =\boldsymbol{\pi}_i^t+\rho(\bu_i^{t+1}-\bu^{t}),\\
\bu^{t+1} & =\operatorname{argmin}_{\bu} \mathcal{L}(\bV^{t+1}, \bU^{t+1}, \Pi^{t+1}, \bu)\\
&=\frac{1}{m} \sum_{i=1}^m(\bu_i^{t+1}+\frac{1}{\rho}\boldsymbol{\bpi}_i^{t+1}). 
\end{aligned}
\right.
\end{align}
Due to the potentially non-convex nature of $f_i$, closed-form solutions for $\bu_i$ and $\bv_i$ may not exist. We will provide an approach to address this issue later in Section \ref{sec:algorithm design}.

\subsection{Stationary Points}\label{sec: stationary points}
We define the optimal conditions of \eqref{eq: problem pfl2} as follows:
\begin{definition}[Stationary point]\label{definition: optimal condition}
    A point $(\bV^*,\bU^*,\bu^*,\Pi^*)$ is a stationary point of \eqref{eq: problem pfl2} if it satisfies
\begin{align}\label{eq: optimal condition of pfl problem2}
\left\{\begin{aligned}
\alpha_i\nabla_{\bu_i} f_i(\bv_i^*,\bu_i^*)+\bpi_i^*+\rho(\bu_i^*-\bu^*)& =0, &  i \in[m],\\
\nabla_{\bv_i} f_i(\bv_i^*,\bu_i^*)& =0, & i \in[m], \\
\bu_i^* - \bu^*& =0, & i \in[m], \\
\sum_{i=1}^m \bpi_i^*& =0. &  \\
\end{aligned}\right.    
\end{align}
\end{definition}
If $f_i$ is convex with respect to $\bv_i$ and $\bu_i$ for any $i\in[m]$, then a point is a globally optimal solution if and only if it satisfies \eqref{eq: optimal condition of pfl problem2}.
Moreover, a stationary point $(\bV^*,\bU^*,\bu^*,\Pi^*)$ of \eqref{eq: problem pfl2} indicates that 
\vspace{-0.5em}
\begin{align}\label{eq: optimal condition of pfl problem1}
\left\{\begin{aligned}
\sum_{i=1}^m \alpha_i\nabla_{\bu} f_i(\bv_i^*,\bu^*)& =0, \\
\vspace{-5pt} 
\nabla_{\bv_i} f_i(\bv_i^*,\bu^*)& =0, & i \in[m], \\
\end{aligned}\right.    
\end{align}
That is, $(\bV^*,\bu^*)$ is also a stationary point of \eqref{eq: pfl problem1}.

\begingroup
\begin{algorithm}[t]
  \caption{\fedapm}
  \label{alg: fedapm}
  \SetAlgoLined
  \KwIn{$T:$ communication rounds, $\rho$: penalty parameter, $m$: number of clients, $\bX_i$: local dataset, $\sigma_i$: hyperparameter.}
  \KwOut{$\{\bv_i\}_{i=1}^m$ (personalized), $\{\bu_i\}_{i=1}^m$ (local), $\bu$ (shared).}
  \textbf{Initialize:} $\bv_i^0,\bu_i^0,\bpi_i^0, \bz_i^0=\bu_i^0+\frac{1}{\rho}\bpi_i^0, \xi_i^0, \mu_i\in(0,1), i\in[m]$.

  \For{$t=0,1,\cdots,T-1$}{
  \textcolor{gray}{\underline{\textit{Weights upload:}}} All clients send $\bz_i^t$ to the server \label{line 4}\;
  \textcolor{gray}{\underline{\textit{Weights average:}}} Server aggregates $\bz_i^t$ by \label{line 5}
\begin{align}
  \hspace{-3em}\bu^{t} = \frac{1}{m}\sum_{i=1}^m \bz_i^t.\hspace{-4.3em} \label{eq: solve u}
\end{align}

\textcolor{gray}{\underline{\textit{Weights feedback:}}} Broadcast $\bu^{t}$ to all clients\label{line 6}\;
\textcolor{gray}{\underline{\textit{Client selection:}}} Randomly select $s$ clients $\mathcal{S}^t\subset[m]$ \label{line 3}\;

  \For{each client $i\in\mathcal{S}^t$}{
  \textcolor{gray}{\underline{\textit{Local update:}}} client $i$ update its parameters as follows: \label{line 8}
    \begin{align}
        \hspace{-3em}&\bv_i^{t+1} = \operatorname{argmin}_{\bv_i} f_i(\bv_i,\bu_i^t) + \frac{\sigma_i}{2}\|\bv_i-\bv_i^t\|^2, \hspace{-1em} \label{eq: solve vi}  \\
        \hspace{-3em}&\xi_i^{t+1} \leq \mu_i \xi_i^t,\hspace{-1em}\label{eq: approximate solution}\\
        \hspace{-3em}&\text{Find a $\xi_i^{t+1}$-approximate solution $\bu_i^{t+1}$ (Definition \ref{definition: approximate solutions}),}\hspace{-1em}  \label{eq: solve ui}\\
        \hspace{-3em}& \bpi_i^{t+1} = \bpi_i^{t} + \rho (\bu_i^{t+1}-\bu^t),\hspace{-1em}\label{eq: solve pi}\\
        \hspace{-3em}& \bz_i^{t+1} = \bu_i^{t+1} + \frac{1}{\rho}\bpi_i^{t+1}\hspace{-1em}\label{eq: update message}
  \end{align}
  }
  \For{each client $i\notin\mathcal{S}^t$}{
  \textcolor{gray}{\underline{\textit{Local invariance:}}} client $i$ keep its parameters as follows: \label{line 10}
  \vspace{-1em}
  \begin{align}\label{eq: client unchange}
\hspace{-13em}(\zeta_i^{t+1}, \bv_i^{t+1}, \bu_i^{t+1}, \bpi_i^{t+1}, \bz_i^{t+1}) = (\zeta_i^{t}, \bv_i^{t}, \bu_i^{t}, \bpi_i^{t}, \bz_i^{t}).\hspace{-9.6em}
  \end{align}
  }
  }

\KwRet $\{\bv_i^T\}_{i=1}^m,\,\{\bu_i^T\}_{i=1}^m,\,\bu^T$.
\end{algorithm}    

\endgroup

\subsection{Algorithmic Design}\label{sec:algorithm design}
In Algorithm 1, we present the details of \fedapm. 
Figure \ref{fig: example} shows a running example of \fedapm.
The algorithm is divided into two parts, which are executed on the clients and the server, respectively.
\begin{itemize}
    \item \textbf{Server update:}
    First, the clients upload their locally computed parameters $\bz_i^{t}$ to the server (Line \ref{line 4}). 
    The server aggregates these parameters to obtain the shared model $\bu^{t}$ (Line \ref{line 5}), and then broadcasts $\bu^{t}$ to each client (Line \ref{line 6}). 
    Next, the clients are randomly divided into two groups (Line \ref{line 3}). 
    For the clients in $\mathcal{S}^t$, they perform a local update, while the clients not in $\mathcal{S}^t$ keep their local parameters unchanged.
    \item \textbf{Client update:} 
    The clients in $\mathcal{S}^t$ update the model parameters locally (Line \ref{line 8}). 
    Due to the possible non-convex nature of $f_i$, we propose using proximal update strategies for solving $\bv_i$ and $\bu_i$. 
    Specifically, we consider solving $\bv_i$ by adding a proximal term to the augmented Lagrangian function such that the new function is convex with respect to $\bv_i$, that is 
\begin{align}\label{eq: 15 solve v with proximal update}
    \bv_i^{t+1} = \operatorname{argmin}_{\bv_i}\ml_i(\bv_i,\bu_i^t,\bpi_i^t,\bu^t) + \frac{\sigma_i}{2}\|\bv_i-\bv_i^t\|^2,
\end{align}
where $\sigma_i$ is a hyperparameter employed to control the degree of approximation between $\bv_i$ and $\bv_i^t$.
The purpose of using proximal update strategies is to effectively stabilize the training process \cite{Zeng2019Global}.
For each subproblem involving $\bv_i$, we assume the minimizer can be achieved.
In a similar vein, to solve for $\bu_i$, we consider applying \textit{stochastic gradient descent} (SGD) to obtain an approximate solution for $\bu_i$, defined as follows:
    \begin{definition}[$\xi$-Approximate solution]\label{definition: approximate solutions}
        \!\! For $\xi_i^{t+1}\!\in\! (0,\!1)$, we say $\bu_i^{t+1}$ is a $\xi_i^{t+1}$-approximate solution of $\min_{\bu_i} \!\ml_i(\bv_i^{t+1}\!, \bu_i,\! \boldsymbol{\pi}^t_i, \!\bu^{t})$ if 
        \begin{align}\label{eq: approximate solution about u}
            \|\alpha_i \nabla_{\bu_i}f_i(\bv_i^{t+1}, \bu_i^{t+1})+\pi_i^t+\rho(\bu_i^{t+1}-\bu^t)\|^2\leq \xi_i^{t+1}.
        \end{align}
         Note that a smaller $\xi_i^{t+1}$ corresponds to higher accuracy.
    \end{definition}
SGD is commonly used to compute approximate solutions in a finite number of iterations and is applied in many FL frameworks \cite{Li2020fedprox, Zhou2023FedADMM, Dinh2020pfedme}.
Note that we set the local accuracy level at each iteration as $\xi_i^{t+1} \leq \mu_i \xi_i^{t}$. In the early stages of algorithm iteration, when the solution is far from stationary points, setting a larger $\xi_i$ can effectively reduce the number of iterations and improve efficiency.
As iterations increase, the solution gets closer to the stationary points. Therefore, reducing $\xi_i$ can improve the accuracy.
After computing $\bv_i$ and $\bu_i$, we follow from \eqref{eq: admm update order} to update the dual variable $\bpi_i$ and calculate the update messages $\bz_i$ that will be uploaded to the server.
For those clients that are not in $\mathcal{S}^t$, the local parameters remain unchanged (Line \ref{line 10}).

\end{itemize}

\begin{figure*}[t]
  \centering
  \includegraphics[width=0.7\linewidth]{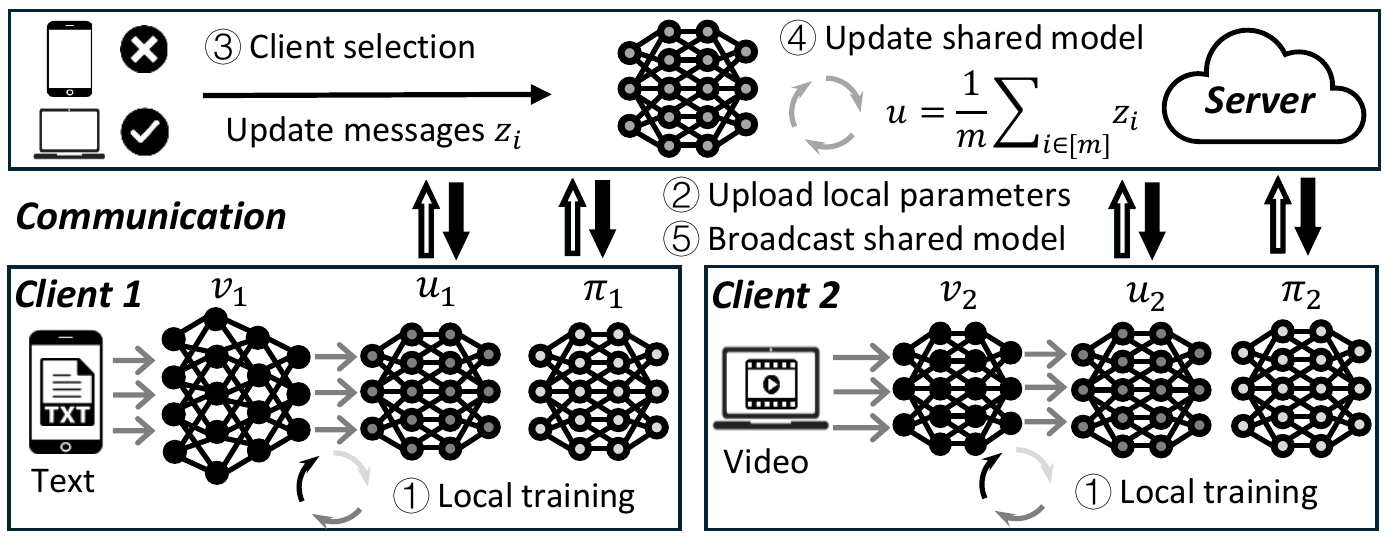}
  \caption{A running example of \fedapm. Different clients may possess heterogeneous or multimodal data. Each client updates its local parameters using local data and uploads $\bz_i$ to the server, which then updates and broadcasts $\bu$ to all clients.}
  \label{fig: example}
\end{figure*}

\subsection{Discussions}\label{sec: discussion}
In this section, we provide an intuitive discussion on why partial model personalization can exacerbate client drift in the shared model. 
Furthermore, we highlight the superiority of \fedapm in addressing client drift, as well as its improved convergence stability and efficiency compared to \fedalt and \fedsim.

In input personalization, the personalized module processes the raw input before it reaches the shared model. 
As a result, each client’s personalized transformation alters the input distribution seen by the shared model. 
Due to heterogeneity in client data and personalized modules, the shared model receives inconsistent and client-specific inputs, which leads to misaligned optimization objectives across clients.
This discrepancy causes the shared model to drift, as it is effectively optimizing for a different feature space on each client.
In output personalization, it is the backpropagated gradients from the personalized layers that influence the shared layers. 
Since these gradients are conditioned on personalized outputs (e.g., client-specific classification heads), they differ across clients even for similar inputs, thereby causing divergence in the shared representation layer updates.

\fedapm can be considered a general FL framework, while \fedalt and \fedsim are special cases of \fedapm. 
Specifically, when the dual variables \(\bpi_i,i\in[m]\) and the penalty parameter $\rho$ are set to $\boldsymbol{0}$, and the initial values of \(\bu_i\) are set to \(\bu^t\) in the corresponding subproblem, \fedapm reduces to \fedalt and \fedsim.
Therefore, \fedalt and \fedsim can be interpreted as FL approaches that solve \eqref{eq: problem pfl2} through penalty methods.
However, client drift arises in these methods: each client optimizes its own objective rather than the global one due to differing local objectives, which could hinder convergence or even lead to divergence. 
While increasing $\rho$ can mitigate this issue, it introduces another challenge: for ill-conditioned problems, an excessively large $\rho$ is required to enforce constraints effectively. 
This, in turn, may result in numerical instability during optimization, as demonstrated in \cite{Numerical2006Wright}.
This issue is particularly problematic in the FL optimization, where the number of constraints $m$ further increases the condition number of the objective function's Hessian matrix, exacerbating the risk of ill-conditioning.

In \fedapm, the second-order term provides a fixed correction for the update of $\bu_i$, causing the gradient descent direction of $\bu_i$ to shift towards $\bu$ by a fixed value $\rho(\bu_i - \bu)$, while the first-order term allows the gradient descent direction of $\bu_i$ to shift towards $\bu$ by a variable value $\bpi_i$, thereby providing compensation that brings $\bu_i$ closer to $\bu$ and effectively addressing the issue of client drift.
Moreover, the update of the dual variables corrects the bias in constraint violations, allowing ADMM to avoid relying on large \(\rho\) to enforce the constraints, thus preventing ill-conditioning.

\section{Convergence Analysis}\label{sec: convergence}
We aim to establish the global convergence and convergence rate for \fedapm, starting with the assumptions used in our analysis.
\subsection{Main Assumptions}
We present the definitions of \textit{graph}, \textit{semicontinuous}, \textit{real analytic}, and \textit{semialgebraic} functions, which are utilized in our assumptions.
\begin{definition}[Graph]
    Let $f:\mathbb{R}^p\to\mathbb{R}\cup\{+\infty\}$ be an extended real-valued function, its graph is defined by
\begin{align*}
    \text{Graph}(f):=\{(\boldsymbol{x},y)\in\mathbb{R}^p\times\mathbb{R}:y=f(\boldsymbol{x})\},
\end{align*}
and its domain is defined by $\text{dom}(f):=\{\boldsymbol{x}\in\mathbb{R}^p:f(\boldsymbol{x})<+\infty\}$. If $f$ is a proper function, i.e., $\text{dom}(f)\neq\emptyset$, then the set of its global minimizers is defined by $\argmin f:=\{\boldsymbol{x}\in\mathbb{R}^p:f(\boldsymbol{x})=\inf f\}$.
\end{definition}

\begin{definition}[Semicontinuous \cite{bochnak2013real}]
    A function $f:\mathcal{X}\to \mathbb{R}$ is lower semicontinuous if for any $x_0\in\mathcal{X}$, $\lim_{x\to x_0}\inf f(x)\geq f(x_0)$.
\end{definition}

\begin{definition}[Real analytic \cite{krantz2002primer}]\label{definition: real analytic}
    A function $f$ is real analytic on an open set $\mathcal{X}$ in the real line if for any $x_0\in \mathcal{X}$, $f(x)$ can be represented as $f(x)=\sum_{i=1}^{+\infty}a_i(x-x_0)^i$, where $\{a_i\}_{i=1}^{+\infty}$ are real numbers and the series is convergent to $f(x)$ for $x$ in a neighborhood of $x_0$.
\end{definition}

\begin{definition}[Semialgebraic set and function \cite{bochnak2013real}]\label{definition: Semialgebraic}
       \quad
\begin{itemize}
    \item[a)] A set $\mathcal{X}$ is called semialgebraic if it can be represented by  
    \begin{align*}
        \mathcal{X}=\cup_{i=1}^r\cap_{j=1}^s\{\bx\in\mathbb{R}^p: P_{ij}(\bx)=0,Q_{ij}(\bx)>0\},
    \end{align*}
    where $P_{ij}(\cdot)$ and $Q_{ij}(\cdot)$ are real polynomial functions for $1\leq i\leq r$ and $1\leq j\leq s$.
    \item[b)] A function $f$ is called semialgebraic if $\text{Graph}(f)$ is semialgebraic.
\end{itemize}
\end{definition}
As highlighted in \cite{law1965ensembles,bochnak2013real,shiota1997geometry}, semialgebraic sets are closed under several operations, including finite unions, finite intersections, Cartesian products, and complements. 
Typical examples include polynomial functions, indicator functions of a semialgebraic set, and Euclidean norm.
In the following, we outline the assumptions employed in our convergence analysis.
\begin{assumption}\label{assumption: either real analytic or semialgebraic}
Suppose the expected loss functions $f_i, i\!\in\![m]$ are
\begin{itemize}
    \item[a)] proper lower semicontinuous and nonnegative functions, and
    \item[b)] either real analytic or semialgebraic.
\end{itemize}
\end{assumption}

\begin{assumption}[Gradient Lipschitz continuity]\label{assmption lipschitz}
    For each $i\in[m]$, the expected loss function $f_i(\bv_i,\bu)$ is continuously differentiable. There exist constants $L_u$, $L_v$, $L_{uv}$, $L_{vu}$ such that for each $i\in[m]$:
\begin{itemize}
    \item $\nabla_{\bv_i}f_i(\bv_i,\bu_i)$ is $L_v$-Lipschitz with respect to $\bv_i$ and $L_{vu}$-Lipschitz with respect to $\bu_i$, i.e.
\begin{align*}
    \|\nabla_{\bv_i} f_i(\bv_1,\bu_i)-\nabla_{\bv_i} f_i(\bv_2,\bu_i)\| &\leq L_{v} \|\bv_1-\bv_2\|,\\
    \|\nabla_{\bv_i} f_i(\bv_i,\bu_1)-\nabla_{\bv_i} f_i(\bv_i,\bu_2)\| &\leq L_{vu} \|\bu_1-\bu_2\|,
\end{align*}
    \item $\nabla_{\bu_i}f_i(\bv_i,\bu_i)$ is $L_u$-Lipschitz with respect to $\bu_i$ and $L_{uv}$-Lipschitz with respect to $\bv_i$, i.e.
\begin{align*}
    \|\nabla_{\bu_i} f_i(\bv_i,\bu_1)-\nabla_{\bu_i} f_i(\bv_i,\bu_2)\| &\leq L_{u} \|\bu_1-\bu_2\|,\\
    \|\nabla_{\bu_i} f_i(\bv_1,\bu_i)-\nabla_{\bu_i} f_i(\bv_1,\bu_i)\| &\leq L_{uv} \|\bv_1-\bv_2\|,
\end{align*}
\end{itemize}

\end{assumption}

\begin{assumption}\label{assumption: coercive}
    The expected loss functions $f_i, i\in[m]$ are coercive\footnote{An extended-real-valued function $f:\mathbb{R}^p\to\mathbb{R}\cup\{+\infty\}$ is called coercive if and only if $f(\bx)\to +\infty$ as $\|\bx\|\to+\infty$.}. 
\end{assumption}
According to \cite{krantz2002primer,law1965ensembles,bochnak2013real,shiota1997geometry}, common loss functions such as squared, logistic, hinge, and cross-entropy losses can be verified to satisfy Assumption \ref{assumption: either real analytic or semialgebraic}. Assumption \ref{assmption lipschitz} is a standard assumption in the convergence analysis of FL, as noted in \cite{Krishna2022Partial}. Assumption \ref{assumption: coercive} is widely applied to establish the convergence properties of optimization algorithms, as shown in \cite{Zhou2023FedGiA, Zeng2021ADMM, Zeng2019Global}. Our assumptions are weaker than those used by \fedalt and \fedsim \cite{Krishna2022Partial}, which include \textit{gradient Lipschitz continuity, bounded variance, and bounded diversity}.

\subsection{Global Convergence}
Our global convergence analysis builds upon the analytical framework presented in \cite{Attouch2013Convergence}, which relies on four essential components: the establishment of \textit{sufficient descent} and \textit{relative error} conditions, along with the verification of the \textit{continuity condition} and the \textit{KŁ property} of the objective function.
Let $\mathcal{P}^t:=(\bV^t, \bU^t, \bPi^t, \bu^t)$, we define the Lyapunov function $\tilde{\ml}(\mathcal{P}^t)$  by adding to the original augmented Lagrangian the accuracy level of each $\bu_i^t$:
\begin{align} \label{eq: Lyapunov function}
    \tilde{\ml}(\mathcal{P}^t):=\ml(\mathcal{P}^t)+\sum_{i=1}^m \frac{29}{\rho(1-\mu_i)} \xi_i^{t}.
\end{align}
With the Lyapunov function established, we can demonstrate that the sequence generated by ADMM converges to the stationary points defined in \eqref{eq: optimal condition of pfl problem1} and \eqref{eq: optimal condition of pfl problem2} (see Theorem \ref{theorem: sequences convergence}).
Below, we present the two key lemmas, with additional details provided in the Appendix. 
Specifically, the verification of the KŁ property for FL training models satisfying Assumption \ref{assumption: either real analytic or semialgebraic} is outlined in Proposition \ref{proposition: kl property of L} (Appendix \ref{subappendix: Kurdyka-Łojasiewicz Property}), while the verification of the continuity condition is naturally satisfied by Assumption \ref{assumption: either real analytic or semialgebraic}. 
Using Lemmas \ref{lemma: sufficient descent} and \ref{lemma: relative error}, Proposition \ref{proposition: kl property of L}, and Assumption \ref{assumption: either real analytic or semialgebraic}, we establish Theorem \ref{theorem: sequences convergence}. 
\begin{lemma}[Sufficient Descent]\label{lemma: sufficient descent}
    \! Suppose that Assumption \ref{assmption lipschitz} holds. Let $\{\mathcal{P}^t\}$ denote the sequence generated by Algorithm \ref{alg: fedapm}; let each client set the hyperparameters such that $\max\{3\alpha_i L_u, 3\alpha_i L_{uv} \}\leq \rho \leq \frac{15}{8}\sigma_i$, then for all $t\geq 0$, it holds that
\begin{align*}
    &\tilde{\ml}(\mathcal{P}^{t})-\tilde{\ml}(\mathcal{P}^{t+1})\geq a \Sigma_p^{t+1},\,\text{where}\,\,a:=\min\{\frac{\rho}{60} , \, \frac{\sigma_i}{2}-\frac{4\rho}{15}\}\\
    &\text{and}\,\, \Sigma_p^{t+1} \!:=\!\sum_{i=1}^m(\|\bv_i^{t+1}\!-\!\bv_i^{t}\|^2\!+\!\|\bu^{t+1}_i\!-\!\bu_i^{t}\|^2\!+\!\|\bu^{t+1}\!-\!\bu^t\|^2).
\end{align*}

\end{lemma}

\begin{lemma}[Relative Error]\label{lemma: relative error}
    Suppose that Assumption \ref{assmption lipschitz} holds. Let $\{\mathcal{P}^t\}$ denote the sequence generated by Algorithm \ref{alg: fedapm}, let $\Xi^t:=\sum_{i=1}^m \xi_i^{t}$, $\tilde{\Xi}^{t+1}:=\Xi^{t+1}-\Xi^t$, and let each client set the hyperparameters such that $\sigma_i\geq \alpha_i L_v$ and $\max\{3\alpha_i L_u, 3\alpha_i L_{uv} \}\leq \rho \leq \frac{15}{8}\sigma_i$, then for all $t\geq 0$, the following result holds:
\begin{align}
    \text{dist}(\boldsymbol{0}, \partial \ml (\mathcal{P}^t))^2\leq  b (\Sigma_p^{t+1} + \tilde{\Xi}^{t+1}),
\end{align}
where $b:=\max\{\frac{22}{5}\sigma_i^2 + \frac{4}{3}\rho^2+\frac{8\rho}{15}, \frac{16}{3}\rho^2+2+\frac{8\rho}{15}, \frac{48+4\rho}{\rho(1-\mu_i)}\}$, $\text{dist}(0,\mathcal{C}):=\inf_{\boldsymbol{c}\in \mathcal{C}} \|\boldsymbol{c}\|$ for a set $\mathcal{C}$, and 
\begin{align*}
    \partial \ml (\mathcal{P}^t):= (\{\nabla_{\bv_i}{\ml}\}_{i=1}^m,\{\nabla_{\bu_i}{\ml}\}_{i=1}^m,\{\nabla_{\bpi_i}{\ml}\}_{i=1}^m,\{\nabla_{\bu}{\ml}\})(\mathcal{P}^t).
\end{align*}
\end{lemma}
Note that the relative error condition contains an error term $\tilde{\Xi}^{t+1}$ on the right-hand side of the inequality. The existing theoretical frameworks \cite{wang2019global, Zeng2021ADMM, Attouch2013Convergence} cannot account for this error term. 
In the following, we present an analysis framework of global convergence that incorporates this error term.

\begin{theorem}\label{theorem: convergence}
Suppose that Assumption \ref{assmption lipschitz} holds, let each client set the hyperparameters such that $\max\{3\alpha_i L_u, 3\alpha_i L_{uv} \}\leq \rho \leq \frac{15}{8}\sigma_i$ and $\sigma_i\geq \alpha_i L_v$, then the following results hold.
\begin{itemize}
    \item[a)] Sequence $\{\mathcal{P}^t\}$ is bounded.
    \item[b)] 
    The gradients of $f(\bV^{t+1},\bU^{t+1})$ and $f(\bV^{t+1},\bu^{t+1})$ with respect to each variable eventually vanish, i.e.,
    \begin{equation}
    \begin{split}
                 &\lim_{t\to \infty}\nabla_{\bV}f(\bV^{t+1},\bu^{t+1}) = \lim_{t\to \infty}\nabla_{\bV}f(\bV^{t+1},\bU^{t+1}) \to 0,\\
         &\lim_{t\to \infty}\nabla_{\bu}f(\bV^{t+1},\bu^{t+1})= \lim_{t\to \infty}\nabla_{\bU}f(\bV^{t+1},\bU^{t+1}) \to 0.
    \end{split}
    \end{equation}
    \item[c)] Sequences $\{\tilde{\ml}(\mathcal{P}^t)\}$, $\{\ml(\mathcal{P}^t)\}$, $\{f(\bV^t,\bU^t)\}$, and $\{f(\bV^t,\bu^t)\}$ converge to the same value, i.e.,
    \begin{align}
        \lim_{t\to \infty} \! \tilde{\ml}(\mathcal{P}^t) \!= \!\lim_{t\to \infty}\! \ml(\mathcal{P}^t) \!=\! \lim_{t\to \infty}\! f(\bV^t,\bU^t)   \!=\! \lim_{t\to \infty}\!f(\bV^t,\bu^t).
    \end{align}
\end{itemize}
\end{theorem}

Theorem \ref{theorem: convergence} establishes the convergence of the objective function. Next, we establish the convergence properties of the sequence $\{\mathcal{P}^t\}$.
\begin{theorem}\label{theorem: sequences convergence}
    Suppose that Assumptions \ref{assmption lipschitz} and \ref{assumption: coercive} hold, let each client $i$ set the hyperparameters such that $\max\{3\alpha_i L_u, 3\alpha_i L_{uv}, 2L_u \}\leq \rho \leq \frac{15}{8}\sigma_i$ and $\sigma_i\geq \alpha_i L_v$, then the following results hold:
    \begin{itemize}
        \item[a)] The accumulating point $\mathcal{P}^{\infty}$ of sequences $\{\mathcal{P}^t\}$ is a stationary point of \eqref{eq: problem pfl2}, and $(\bV^{\infty},\bu^{\infty})$ is a stationary point of \eqref{eq: pfl problem1}.
        \item[b)] Under Assumption \ref{assumption: either real analytic or semialgebraic}, the sequence $\{\mathcal{P}^t\}$ converges to $\mathcal{P}^{\infty}$.
    \end{itemize}
\end{theorem}
Note that the proof of Theorem \ref{theorem: sequences convergence} does not depend on the convexity of the loss function $f_i$. As a result, the sequence will reach a stationary point for \eqref{eq: problem pfl2} and \eqref{eq: pfl problem1}. Moreover, if we assume that $f_i$ is convex, the sequence will converge to the optimal solutions.

\subsection{Convergence Rate}
We establish the convergence rates when $\tilde{\ml}$ has \kl properties with a desingularizing function $\phi(x)=\frac{\sqrt{c}}{1-\theta}x^{1-\theta}$ (see Definition \ref{definition: Desingularizing Function} and \ref{definition: kl property} in Appendix \ref{subappendix: Kurdyka-Łojasiewicz Property}), where $c>0$ and $\theta\in[0,1)$. We elaborate in Proposition \ref{proposition: kl property of L} that most functions in FL are \kl functions.
\begin{theorem}\label{theorem: convergence rate based on kl property}
        Let $\{\mathcal{P}^t\}$ be the sequence generated by Algorithm \ref{alg: fedapm}, and $\mathcal{P}^{\infty}$ be its limit, then under Assumptions \ref{assumption: either real analytic or semialgebraic}, \ref{assmption lipschitz}, and \ref{assumption: coercive}, let each client set the hyperparameters $\max\{3\alpha_i L_u, 3\alpha_i L_{uv}, 2L_u \}\leq \rho \leq \frac{15}{8}\sigma_i$ and $\sigma_i\geq \alpha_i L_v$, the following results hold:
\begin{itemize}
    \item[a)] If $\theta=0$, then there exists a $t_1$ such that the sequence $\{\tilde{\ml}(\mathcal{P}^t)\}$, $t\geq t_1$ converges in a finite number of iterations.
    \item[b)] If $\theta\in(0,1/2]$, then there exists a $t_2$ such that for any $t\geq t_2$,
\begin{align*}
    \tilde{\ml}(\mathcal{P}^{t+1}) \!- \!\tilde{\ml}(\mathcal{P}^{\infty}) \!\leq \! \Bigl(\frac{bc}{a+bc}\Bigr)^{t-t_2+1}\Bigl(\tilde{\ml}(\mathcal{P}^{t_2})-f^*\Bigr) \! +  \!(a+bc) \Xi^{t_2-1} .
\end{align*}
   \item[c)] If $\theta \in(1/2,1)$, then there exists a $t_3$ such that for any $t
    \geq t_3$, 
\begin{align*}
    \tilde{\ml}(\mathcal{P}^{t+1})-\tilde{\ml}(\mathcal{P}^{\infty})\leq\Bigl(\frac{bc}{(2\theta-1)\kappa a(t-t_3)}\Bigr)^{\frac{1}{2\theta-1}},
\end{align*}
where $\kappa>0$ is a constant.
\end{itemize}
\end{theorem}

Theorem \ref{theorem: convergence rate based on kl property} demonstrates that when $\theta=0$, the convergence rate becomes constant. For $\theta \in (0,1/2]$, the rate is linear, while for $\theta \in (1/2,1)$, it is sublinear.
Note that although Theorem \ref{theorem: convergence rate based on kl property} relies on Assumptions \ref{assumption: either real analytic or semialgebraic}, \ref{assmption lipschitz} and \ref{assumption: coercive}, these conditions are satisfied by most loss functions and are less restrictive compared to those in \cite{Krishna2022Partial}.

\section{Numerical Experiments}\label{sec: experiments}
In this section, we compare the performance of \fedapm with SOTA methods and examine the effects of various parameters on \fedapm. 
Our codes are open-sourced\footnote{\textcolor{black}{\href{https://github.com/whu-totemdb/FedAPM}{https://github.com/whu-totemdb/FedAPM}}}.
We present key design details and results, with complete content provided in Appendix \ref{appendix: Experimental Details}.
\subsection{Setups}

\begin{table}[t]
\centering
\renewcommand{\arraystretch}{1.1}
\setlength{\tabcolsep}{5.5pt} 
\caption{An overview of the datasets (Acc and Gyro are abbreviations of Accelerometer and Gyroscope, respectively).}
\vspace{-1em}
\label{table: datasets and models}
\begin{tabular}{ccccc}
\toprule
\textbf{Ref.} & \textbf{Datasets} & \textbf{Modalities} & \textbf{\# samples} & \textbf{\# clients} \\ \midrule
\cite{krizhevsky2009learning}             & \cifar            & Image               & 60,000              & 20                 \\
\cite{Alam2018CrisisMMD}             & \crisis           & Image, Text         & 18,100              & 20                 \\
\cite{Niloy2021KU}             & \kuhar            & Acc, Gyro         & 10,300              & 63                 \\
\cite{Houwei2014CREMA}             & \crema            & Audio, VIdeo        & 4,798               & 72                 \\ \bottomrule
\end{tabular}\
\vspace{-0.5em}
\end{table}

\myparagraph{Datasets}
We evaluate \fedapm as well as other comparable algorithms on four datasets from common heterogeneous and multimodal FL benchmarks \cite{Li2022Experimental, Feng2023FedMultimodal}, which are \cifar \cite{krizhevsky2009learning}, \crisis \cite{Alam2018CrisisMMD}, \kuhar \cite{Niloy2021KU}, and \crema \cite{Houwei2014CREMA}. Among these, \cifar is unimodal, while the others are multimodal.
For each client, the training and testing data are pre-specified: $80\%$ of data is randomly extracted to construct a training set, keeping the remaining $20\%$ as the testing set.
Table \ref{table: datasets and models} presents an overview of the datasets.

\myparagraph{Partitions}
To incorporate data heterogeneity, we implement various data partitioning strategies as outlined in \cite{Li2022Experimental,Feng2023FedMultimodal}. For \cifar and \crisis, we allocate a portion of samples from each class to every client based on the Dirichlet distribution, which is commonly used for partitioning non-i.i.d. data \cite{Li2022Experimental}.
For \crema and \kuhar, we partition data by the speaker and participant IDs, respectively.

\myparagraph{Models}
We provide different classification models for various datasets. Specifically, we employ a CNN-only model architecture for the image modality, an RNN-only model for the video and text modalities, and a Conv-RNN model for other modalities. 

\myparagraph{Metrics and Baselines}
We use the \textit{accuracy, F1 score, AUC performance, training loss, gap between local and shared models, and communication rounds} as metrics to measure the performance of each competing method.
To demonstrate the empirical performance of our proposed method, we compare \fedapm with two state-of-the-art personalization methods, \fedalt and \fedsim \cite{Krishna2022Partial}, as well as two non-personalized methods, \fedavg \cite{mcmahan2017communication} and \fedprox \cite{Li2020fedprox}.

\begin{table}[t]
\caption{Performance comparison of various methods across multiple datasets. We report the mean and standard deviation for top-1 Accuracy, F1 Score, and AUC over 20 trials. 
Bold values highlight the best performance for each metric.}
\label{table: highest}
 \resizebox{\columnwidth}{!}{
\begin{tabular}{ccccc}
\toprule
\textbf{Datasets}                 & \textbf{Methods} & \textbf{Accuracy}                      & \textbf{F1 Score}             & \textbf{{AUC}}                   \\ \midrule
\multirow{5}{*}{\textbf{\cifar}}  & \fedapm    & \textbf{.738}$\boldsymbol{_{\pm.005}}$ & \textbf{.725}$\boldsymbol{_{\pm.005}}$ & \textbf{.873}$\boldsymbol{_{\pm.010}}$ \\
                         & \fedavg    & $.435_{\pm.012}$          & $.345_{\pm.010}$          & $.602_{\pm.011}$          \\
                         & \fedalt    & $.672_{\pm.006}$          & $.649_{\pm.005}$          & $.839_{\pm.011}$          \\
                         & \fedsim    & $.592_{\pm.005}$          & $.565_{\pm.004}$          & $.808_{\pm.010}$          \\
                         & \fedprox   & $.406_{\pm.010}$          & $.307_{\pm.005}$          & $.571_{\pm.011}$          \\ \midrule
\multirow{5}{*}{\textbf{\crisis}} & \fedapm    & $.357_{\pm.028}$          & \textbf{.294}$\boldsymbol{_{\pm.008}}$ & \textbf{.519}$\boldsymbol{_{\pm.015}}$ \\
                         & \fedavg    & $.374_{\pm.023}$          & $.288_{\pm.012}$          & $.510_{\pm.027}$          \\
                         & \fedalt    & $.364_{\pm.025}$          & $.289_{\pm.010}$          & $.513_{\pm.019}$          \\
                         & \fedsim    & $.364_{\pm.025}$          & $.291_{\pm.009}$          & $.514_{\pm.019}$          \\
                         & \fedprox   & \textbf{.380}$\boldsymbol{_{\pm.017}}$ & $.291_{\pm.009}$          & $.507_{\pm.024}$          \\ \midrule
\multirow{5}{*}{\textbf{\kuhar}}  & \fedapm    & \textbf{.437}$\boldsymbol{_{\pm.051}}$ & \textbf{.396}$\boldsymbol{_{\pm.061}}$ & \textbf{.595}$\boldsymbol{_{\pm.002}}$ \\
                         & \fedavg    & $.142_{\pm.004}$          & $.050_{\pm.004}$          & $.464_{\pm.023}$          \\
                         & \fedalt    & $.196_{\pm.017}$          & $.105_{\pm.022}$          & $.562_{\pm.006}$          \\
                         & \fedsim    & $.196_{\pm.012}$          & $.105_{\pm.017}$          & $.566_{\pm.007}$          \\
                         & \fedprox   & $.143_{\pm.005}$          & $.049_{\pm.005}$          & $.453_{\pm.023}$          \\ \midrule
\multirow{5}{*}{\textbf{\crema}}  & \fedapm    & {\textbf{.651}$\boldsymbol{_{\pm.007}}$} & {\textbf{.590}$\boldsymbol{_{\pm.009}}$} & {\textbf{.695}$\boldsymbol{_{\pm.018}}$} \\
                         & \fedavg    & $.434_{\pm.025}$          & $.295_{\pm.041}$          & $.604_{\pm.044}$          \\
                         & \fedalt    & $.444_{\pm.031}$          & $.304_{\pm.042}$          & $.642_{\pm.020}$          \\
                         & \fedsim    & $.444_{\pm.031}$          & $.304_{\pm.042}$          & $.643_{\pm.020}$          \\
                         & \fedprox   & $.438_{\pm.032}$          & $.297_{\pm.045}$          & $.600_{\pm.046}$          \\ \bottomrule
\end{tabular}}
\end{table}

\subsection{Comparison of Multiple Methods}
For each competing algorithm, different hyperparameters need to be tuned. 
We provide several candidates for each hyperparameter and perform a grid search on all possible combinations based on the accuracy performance on the validation dataset.
Further details on parameter settings can be found in Appendix \ref{appendix: Parameter Settings for the Algorithms}.
\subsubsection{Overall Comparison}
Table \ref{table: highest} reports the top-1 accuracy, F1 score, and AUC for various methods.
It can be observed that \fedapm outperforms all other methods across the datasets on most of the metrics. 
Specifically, compared to the best-performing method among the other four, \fedapm achieves an average improvement of 12.3\%, 16.4\%, and 18.0\% in accuracy, F1 score, and AUC, respectively.
This characteristic indicates that \fedapm is advantageous for FL applications involving heterogeneous or multimodal client data.

\subsubsection{Convergence and Communication Efficiency}
Due to space limitations, we show the results on two datasets here with the full results left in Figure \ref{fig: cifar10_crisis_mmd_ku_har_crema_d_training_loss} in Appendix \ref{sec: complete results}.
Figure \ref{fig: cifar10_crisis_mmd_training_loss} shows the variation in training loss across different methods as the communication rounds increase.
First, it can be observed that the methods based on partial model personalization (\fedapm, \fedalt, \fedsim) tend to converge more effectively compared to non-personalized methods (\fedavg and \fedprox). 
Second, \fedapm demonstrates lower training loss compared to other methods and achieves faster convergence, highlighting its superior convergence performance. 
Finally, for a given loss, \fedapm reaches this target with fewer communication rounds, emphasizing its communication efficiency over other methods.

\begin{figure}[t]
  \centering
  \includegraphics[width=\linewidth]{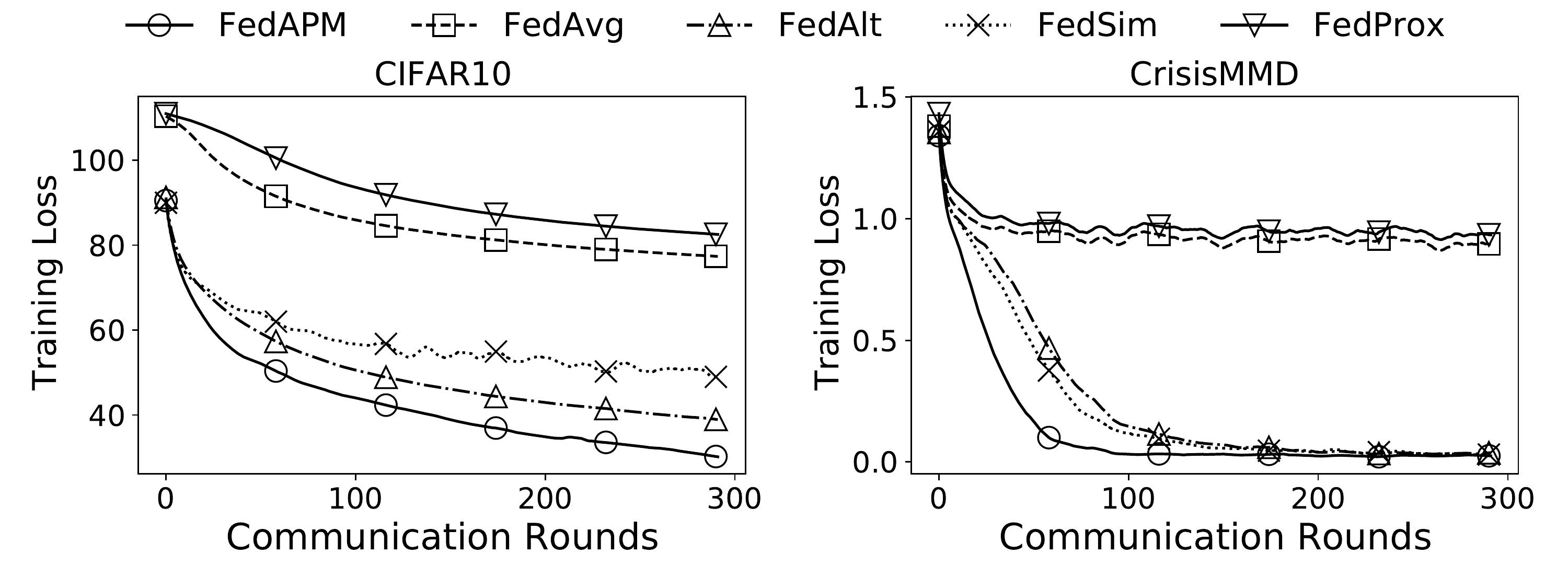}
  \caption{Comparison of training loss across various methods.}
  \label{fig: cifar10_crisis_mmd_training_loss}
\end{figure}

\subsubsection{Model Deviation}
Figure \ref{fig: grouped_norm_bar_chart} illustrates the distance comparison between local and shared models for different methods \footnote{For non-personalized methods, the shared model components we selected are the same as those in the personalized methods.}. 
First, partial model personalization can sometimes exacerbate client drift. In the \kuhar dataset, both \fedalt and \fedsim cause the local model to diverge further from the shared model compared to \fedavg and \fedprox.
Second, we observe that the local model trained with \fedapm is closer to the shared model than those trained with other methods, indicating that \fedapm can effectively address client drift.

\subsection{Comparison of Multiple Hyperparameters}
\subsubsection{Effect of Penalty Parameter}
We show the results on two datasets here with the full results shown in Figure \ref{fig: cifar10_crisis_mmd_ku_har_crema_d_FedAPM_loss_varying_rho} in Appendix \ref{sec: complete results}.
Figure \ref{fig: cifar10_crisis_mmd_FedAPM_loss_varying_rho} illustrates how the training loss of \fedapm changes with communication rounds under different values of $\rho$, where $\rho$ takes the values of $\{0.001, 0.01, 0.02, 0.05, 0.1\}$. 
The experimental results on \cifar and \crisis show that \fedapm effectively addresses ill-conditioning, with smaller $\rho$ leading to faster convergence without the need for an excessively large $\rho$. 
This supports the analysis presented in Section \ref{sec: discussion}, where we show that the dual variable effectively corrects the bias in constraint violations, allowing ADMM to avoid relying on large $\rho$ to enforce the constraints.


\begin{figure}[t]
  \centering
  \includegraphics[width=\linewidth]{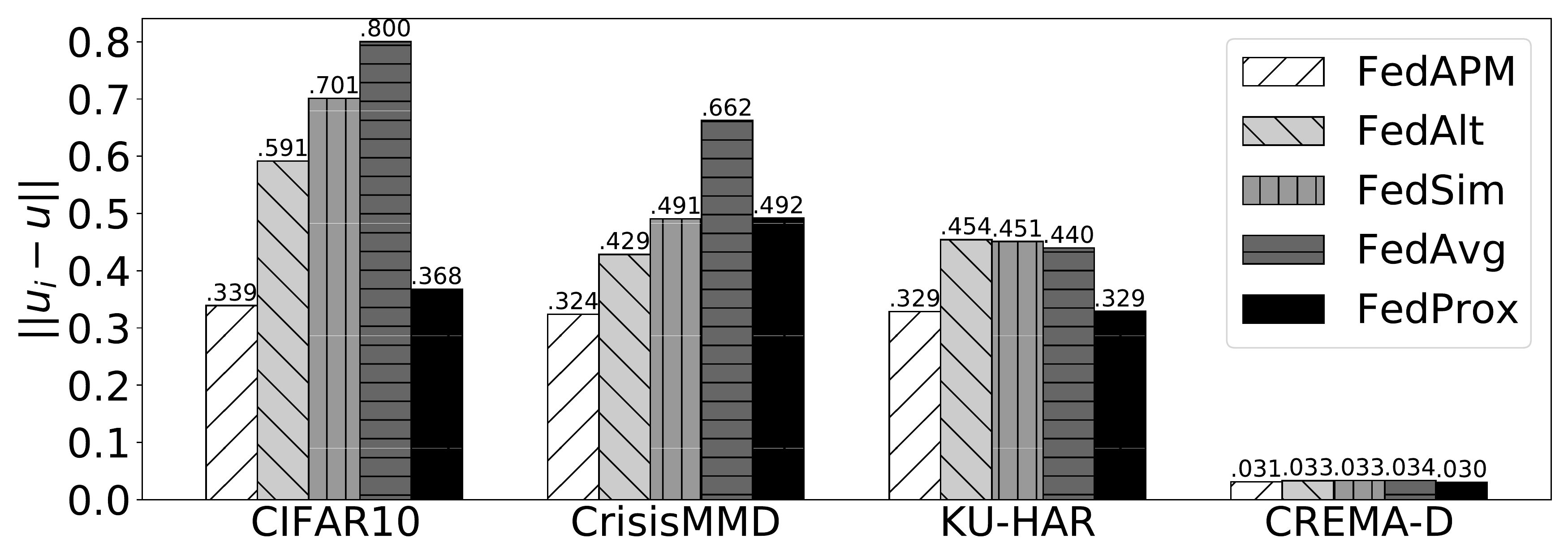}
  \caption{The distance between the local and the shared models across different FL methods. A larger distance indicates that the local model deviates further from the shared model.}
  \label{fig: grouped_norm_bar_chart}
\end{figure}
\subsubsection{Effect of Client Selection}
We show the results on two datasets here with the full results presented in Figure \ref{fig: cifar10_crisis_mmd_ku_har_crema_d_FedAPM_loss_varying_frac} in Appendix \ref{sec: complete results}.
Figure \ref{fig: cifar10_crisis_mmd_FedAPM_loss_varying_frac} illustrates the variation in the training loss of \fedapm with different numbers of selected clients, where the fraction of selected clients is $\{0.1, 0.2, 0.3, 0.5\}$. 
It can be observed that as the number of selected clients increases, the training loss decreases more rapidly. 
This is because a greater number of clients participating enables more frequent model updates, thereby enhancing model accuracy. 
Therefore, when communication capacity allows, increasing the number of selected clients can improve convergence speed.

\subsection{Insights on Experimental Results}
We have compared the performance of \fedapm with SOTA methods and validated the impact of hyperparameters on the performance of \fedapm. 
Our findings have led to the following insights:
\begin{itemize}
    \item \fedapm demonstrates superior accuracy, F1 score, AUC, convergence, and communication efficiency compared to state-of-the-art methods in heterogeneous and multimodal settings.
    \item Partial model personalization can exacerbate client drift, causing the local model to diverge further from the shared model. The proposed \fedapm effectively mitigates client drift, bringing the local model closer to the shared model.
    \item Selecting appropriate hyperparameters for \fedapm can effectively enhance convergence. Specifically, choosing a relatively small value for $\rho$ and increasing the number of clients participating in each training round can yield faster convergence.
\end{itemize}

\begin{figure}[t]
  \centering
  \includegraphics[width=\linewidth]{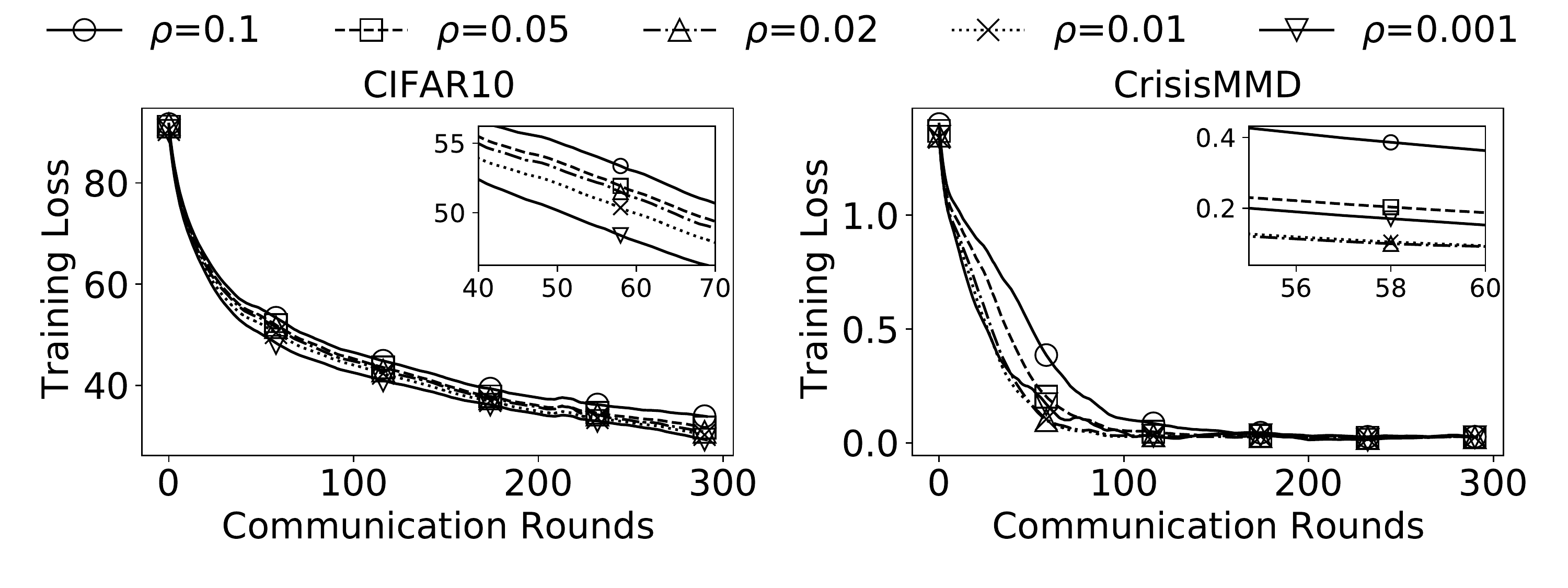}
  \caption{Convergence v.s. penalty parameter.}
    \vspace{2em}
  \label{fig: cifar10_crisis_mmd_FedAPM_loss_varying_rho}
  \includegraphics[width=\linewidth]{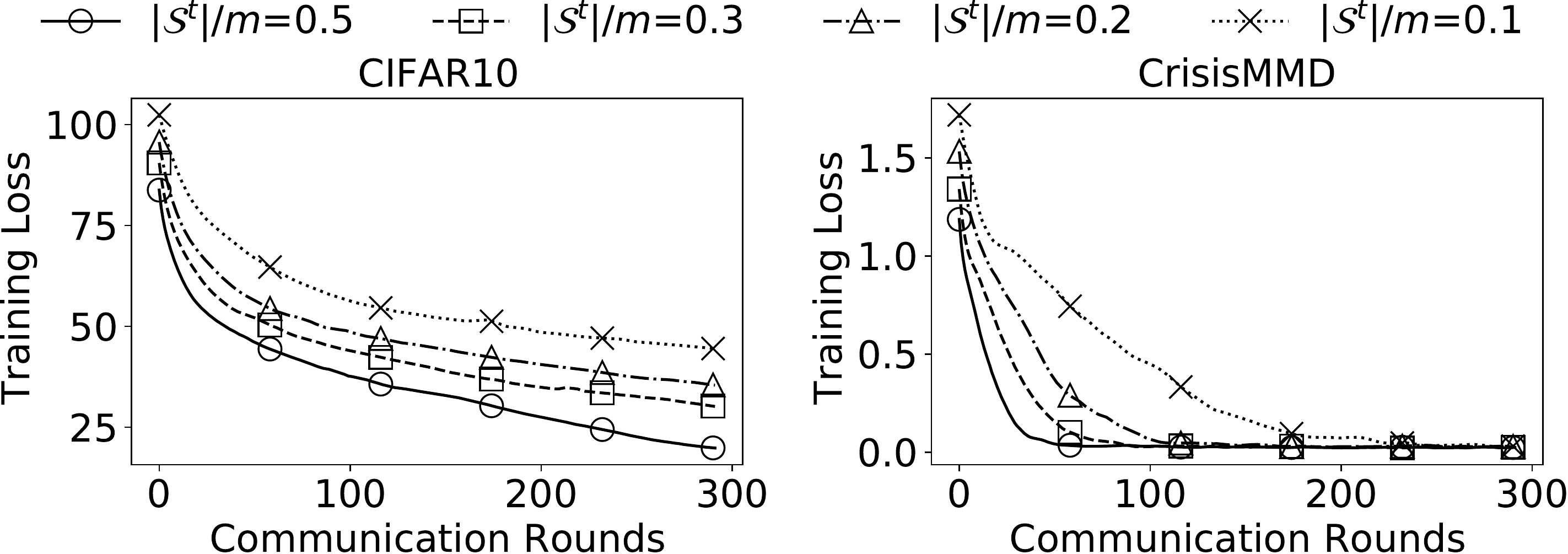}
  \caption{Convergence v.s. client selection.}
  \label{fig: cifar10_crisis_mmd_FedAPM_loss_varying_frac}
\end{figure}

\section{Conclusions and Future Work}
In this paper, we proposed an FL framework based on ADMM, called \fedapm, which performed well in heterogeneous and multimodal settings. 
\fedapm effectively addressed client drift and ill-conditioning issues in other FL frameworks.
We established global convergence for \fedapm with three convergence rates under weaker conditions than existing theoretical frameworks. 
Our experimental results demonstrated the superior performance of \fedapm across multiple datasets and metrics, and we also validated the impact of different hyperparameters, providing a strategy for their selection.

In future work, we plan to improve \fedapm with robust privacy-preserving mechanisms, such as encryption techniques or differential privacy, to prevent adversaries from inferring users' sensitive information from the shared model parameters. 
Additionally, we will explore the theoretical foundations and practical applications of federated learning and ADMM methods within foundation models.


\begin{acks}
This work was supported by the National Key R\&D Program of China (2023YFB4503600), National Natural Science Foundation of China (62202338), Key R\&D Program of Hubei Province (2023BAB081), and Ant Group through CCF-Ant Research Fund.
Jinshan Zeng was partially supported by the National Natural Science Foundation of China under Grant No.62376110, the Jiangxi Provincial Natural Science Foundation for Distinguished Young Scholars under Grant No. 20224ACB212004.
Yao Yuan was partially supported by the Research Grants Council (RGC) of Hong Kong, SAR, China (GRF-16308321) and the NSFC/RGC Joint Research Scheme Grant N\_HKUST635/20.
\end{acks}

\newpage

\balance
\bibliographystyle{ACM-Reference-Format}
\bibliography{sample}

\newpage
\onecolumn
\appendix
\begin{Huge}
    \textbf{Appendix}\\
\end{Huge}

\begin{Large}
    We provide a simple table of contents below for easier navigation of the appendix.\\
\end{Large}

\begin{LARGE}
    \textbf{CONTENTS}\\
\end{LARGE}

\begin{Large}
\textbf{Section \ref{appendix: convergence}: Convergence}\\

\quad \quad Section \ref{appendix: some useful}: Some Useful Properties\\

\quad \quad Section \ref{subappendix: Kurdyka-Łojasiewicz Property}: Kurdyka-Łojasiewicz Property in FL Training with ADMM\\

\quad \quad Section \ref{appendix: Sufficient Descent property}: Sufficient Descent\\

\quad \quad Section \ref{appendix: Relative Error property}: Relative Error\\

\quad \quad Section \ref{proof of theorem 1(theorem: convergence)}: Proof of Theorem \ref{theorem: convergence}\\

\quad \quad Section \ref{appendix: proof of theorem 2theorem: sequences convergence}: Proof of Theorem \ref{theorem: sequences convergence}\\

\textbf{Section \ref{appendix: Experimental Details}: Experimental Details and Extra Results}\\

\quad \quad Section \ref{sec: dataset details}: Full Details on the Datasets\\

\quad \quad Section \ref{sec: model details}: Neural Network Architecture for the Models used in Numerical Experiments\\

\quad \quad Section \ref{appendix: Parameter Settings for the Algorithms}: Parameter Settings for the Algorithms\\

\quad \quad Section \ref{appendix: Implementations on fedapm and Comparison Methods}: Implementations on \fedapm and Comparison Methods\\

\quad \quad Section \ref{sec: complete results}: Complete Results\\
\end{Large}

\newpage

\section{Convergence }\label{appendix: convergence}
\subsection{Some Useful Properties}\label{appendix: some useful}
We provide a list of useful properties and lemmas that are necessary for proving our theorems. Lemma \ref{jensen} presents commonly employed inequalities, while Lemma \ref{lemma: lipschitz} provides an exposition of the property of the smooth functions. Lemma \ref{lemma: pi estimate} will be employed in the proof of our main theorems.
\begin{lemma}\label{jensen}
For any vectors $\bv_1$, $\bv_2$ and $k>0$, we have
\begin{align}
    2\langle\bv_1,\bv_2\rangle&\leq k\|\bv_1\|^2+\frac{1}{k}\|\bv_2\|^2,\\
    \|\bv_1-\bv_2\|^2&\leq (1+k)\|\bv_1\|^2+(1+\frac{1}{k})\|\bv_2\|^2,\\
    \|\sum_{i=1}^k \bv_i\|^2&\leq k\sum_{i=1}^k\|\bv_i\|^2.
\end{align}
\end{lemma}

\begin{lemma}\label{lemma: lipschitz}
    Suppose that $f_i, i\in[m]$ satisfy Assumption \ref{assmption lipschitz}, then it holds that
\begin{align}
    f_i(\bv_i',\bu_i) - f_i(\bv_i,\bu_i) \leq \langle \nabla_{\bv} f_i(\bv_i,\bu_i), \bv_i'-\bv_i  \rangle + \frac{L_v}{2}\|\bv_i'-\bv_i\|^2,\\
    f_i(\bv_i,\bu_i') - f_i(\bv_i,\bu_i) \leq \langle \nabla_{\bu} f_i(\bv_i,\bu_i), \bu_i'-\bu_i  \rangle + \frac{L_u}{2}\|\bu_i'-\bu_i\|^2.
\end{align}
\end{lemma}



\begin{lemma} \label{lemma: pi estimate}
    Suppose that Assumption \ref{assmption lipschitz} holds, let each client set $\rho\geq \max\{3\alpha_i L_u, 3\alpha_i L_{uv}\}$, then for any client $i\in[m]$, it holds that 
    \begin{align}
        \|\bpi_i^{t+1}-\bpi_i^t\|^2\leq  \frac{24}{1-\mu_i}(\xi_i^{t} - \xi_i^{t+1} )+ \frac{4}{15}\rho^2\Bigl(\|\bu_i^{t+1} - \bu_i^{t}\|^2 + \|\bv_i^{t+1} - \bv_i^{t}\|^2 \Bigr)
    \end{align}
\end{lemma}
\begin{proof}
    If $i\in\mathcal{S}^t$, we define the local error term as
\begin{align}
    e_i^{t+1}&:=\alpha_i \nabla_{\bu_i}f_i(\bv^{t+1}_i, \bu_i^{t+1})+\bpi_i^t+\rho(\bu_i^{t+1}-\bu^t)\notag\\
    &=\alpha_i \nabla_{\bu_i}f_i(\bv^{t+1}_i, \bu_i^{t+1})+\bpi_i^{t+1}.\label{eq: error term}
\end{align}
where \eqref{eq: error term} follows from \eqref{eq: solve pi}. Then we obtain
\begin{align}
    \|\bpi_i^{t+1}-\bpi_i^t\|^2 = &\|e_i^{t+1} -e_i^{t} + \alpha_i \nabla_{\bu_i}f_i(\bv^{t}_i, \bu_i^{t}) - \alpha_i \nabla_{\bu_i}f_i(\bv^{t+1}_i, \bu_i^{t+1})\|^2\\
    \leq & 6\|e_i^{t+1} -e_i^{t} \|^2 +  \frac{6}{5}\|\alpha_i \nabla_{\bu_i}f_i(\bv^{t}_i, \bu_i^{t}) -\alpha_i \nabla_{\bu_i}f_i(\bv^{t+1}_i, \bu_i^{t+1})\|^2   \label{eq: 27}\\
    \leq & 12(\xi_i^{t+1} + \xi_i^{t} )+  \frac{6}{5}\|\alpha_i \nabla_{\bu_i}f_i(\bv^{t+1}_i, \bu_i^{t+1})- \alpha_i \nabla_{\bu_i}f_i(\bv^{t+1}_i, \bu_i^{t}) + \alpha_i \nabla_{\bu_i}f_i(\bv^{t+1}_i, \bu_i^{t})  - \alpha_i \nabla_{\bu_i}f_i(\bv^{t}_i, \bu_i^{t})\|^2  \label{eq: 28}\\
    \leq & \frac{12(1+\mu_i)}{1-\mu_i}(\xi_i^{t} - \xi_i^{t+1} )+ \frac{12}{5}\Bigl(\| \alpha_i  \nabla_{\bu_i}f_i(\bv^{t+1}_i, \bu_i^{t+1})- \alpha_i \nabla_{\bu_i}f_i(\bv^{t+1}_i, \bu_i^{t})\|^2 + \|\alpha_i \nabla_{\bu_i}f_i(\bv^{t+1}_i, \bu_i^{t})  - \alpha_i \nabla_{\bu_i}f_i(\bv^{t}_i, \bu_i^{t})\|^2\Bigr)\label{eq: 29}\\
    \leq & \frac{12(1+\mu_i)}{1-\mu_i}(\xi_i^{t} - \xi_i^{t+1} )+ \frac{12}{5}\Bigl(\alpha_i^2 L_{u}^2\|\bu_i^{t+1} - \bu_i^{t}\|^2 + \alpha_i^2 L_{uv} ^2\|\bv_i^{t+1} - \bv_i^{t}\|^2 \Bigr) \label{eq: 30}\\
    \leq & \frac{24}{1-\mu_i}(\xi_i^{t} - \xi_i^{t+1} )+ \frac{4}{15}\rho^2\Bigl(\|\bu_i^{t+1} - \bu_i^{t}\|^2 + \|\bv_i^{t+1} - \bv_i^{t}\|^2 \Bigr) \label{eq: 31}.
\end{align}
where \eqref{eq: 27}, \eqref{eq: 28}, and \eqref{eq: 29} follow from Lemma \ref{jensen}, \eqref{eq: 30} follows from Assumption \ref{assmption lipschitz}, and \eqref{eq: 31} follows from the fact that $\rho\geq \max\{3\alpha_i L_u, 3\alpha_i L_{uv}\}$.
If $i\notin\mathcal{S}^t$, we have $\|\bpi_i^{t+1}-\bpi_i^t\|^2 =0 \leq \frac{24}{1-\mu_i}(\xi_i^{t} - \xi_i^{t+1} )+ \frac{4}{15}\rho^2\Bigl(\|\bu_i^{t+1} - \bu_i^{t}\|^2 + \|\bv_i^{t+1} - \bv_i^{t}\|^2 \Bigr)$, which completes the proof.
\end{proof}

\subsection{Kurdyka-Łojasiewicz Property in FL Training with ADMM}\label{subappendix: Kurdyka-Łojasiewicz Property}
To prove Theorem 1, we need to demonstrate that the KŁ property holds for the considered Lagrangian function $\ml$.
We begin by presenting the definition of KŁ property and elucidating some properties of real analytic and semialgebraic functions.
\begin{definition}[Desingularizing Function]\label{definition: Desingularizing Function}
    A function $\phi\!:\![0,\theta)\!\to\!(0,+\infty)$ satisfying the following conditions is a desingularizing function:
\begin{itemize}
    \item[a)] $\phi$ is concave and continuously differentiable on $(0,\theta)$;
    \item[b)] $\phi$ is continuous at 0 and $\phi(0)=0$;
    \item[c)] For any $x\in(0,\theta), $ $\phi'(x)>0$.
\end{itemize}
\end{definition}

\begin{definition}[Fr\'{e}chet Subdifferential \cite{rockafellar2009variational} and Limiting Subdifferential \cite{Mordukhovich2006Variational}]
    For any $\bx,\boldsymbol{y}\in\text{dom}(h)$, the Fr\'{e}chet subdifferential of $h$ at $\bx$, represented by $\hat{\partial}h(\bx)$, is the set of vectors $\boldsymbol{z}$ which satisfies 
\begin{align*}
    \lim_{\boldsymbol{y}\neq \boldsymbol{x},}\inf_{\boldsymbol{y}\to \bx}\frac{h(\boldsymbol{y})-h(\bx)-\langle
    \boldsymbol{z},\boldsymbol{y}-\bx\rangle}{\|\bx-\boldsymbol{y}\|}\geq 0.
\end{align*}
When $\bx\notin \text{dom}(h)$, we set $\hat{\partial}h(\bx)=\emptyset$. The limiting subdifferential (or simply subdifferential) of $h$, represented by $\partial h(\bx)$ at $\bx\in\text{dom}(h)$ is defined by
\begin{align*}
    \partial h(\bx):=\{\boldsymbol{z}\in\mathbb{R}^p:\exists\bx^t\to\bx, h(\bx^t)\to h(\bx), \boldsymbol{z}^t\in\hat{\partial}h(\bx)\to\boldsymbol{z}\}.
\end{align*}
\end{definition}

\begin{definition}[Kurdyka-Łojasiewicz Property \cite{bolte2007lojasiewicz}]\label{definition: kl property}
    A function $h:\mathbb{R}^p\to\mathbb{R}\cup\{+\infty\}$ is said to have Kurdyka-Łojasiewicz (KŁ) property at $\bx^*\in \text{dom}(\partial h)$ if there exist a neighborhood $U$ of $\bx^*$, a constant $\eta$, and a desingularizing function $\phi$, such that for all $\bx\in U\cap \text{dom}(\partial h)$ and $h(\bx^*)<h(\bx)<h(\bx^*)+\eta$, it holds that,
\begin{align}\label{eq: kl inequality}
    \phi'(h(\bx)-h(\bx^*))\text{dist}(0,\partial h(\bx))\geq 1,
\end{align}
where $\text{dist}(0,\partial h(\bx)):=\inf\{\|\boldsymbol{z}\|:\boldsymbol{z}\in\partial h(\bx)\}$ represents the distance between zero to the set $\partial h(\bx)$.
If $h$ satisfies the KŁ property, then $f$ is called a KŁ function.
\end{definition}

The definition of the KŁ property means that the function under consideration is sharp up to a reparametrization \cite{Attouch2009convergence}. Particularly, when $h$ is smooth, finite-valued, and $h(\bx^*)=0$, then \eqref{eq: kl inequality} can be represented as $\|\nabla(\phi\circ h)(\bx)\|\geq 1$ for each $\bx\in\mathbb{R}^p$.
This inequality can be interpreted as follows: by reparametrizing the values of $h$ through $\phi$, we obtain a well-defined function. The function $\phi$ serves the purpose of transforming a singular region, where gradients are arbitrarily small, into a regular region, characterized by gradients bounded away from zero.
The class of KŁ functions encompasses a diverse range of function types, comprising real analytic functions (as defined in Definition \ref{definition: real analytic}), semialgebraic functions (as delineated in Definition \ref{definition: Semialgebraic}), as well as tame functions defined within certain o-minimal structures \cite{kurdyka1998gradients}. Additionally, it encompasses continuous subanalytic functions \cite{bolte2007lojasiewicz} and locally strongly convex functions \cite{Xu2013Block}.

In the following, we establish the KŁ property of the FL training  with ADMM, i.e. the function $\ml$ defined in \eqref{eq: lagrangian function} and $\tilde{\ml}$ defined in \eqref{eq: Lyapunov function}.
\begin{proposition}[\kl property in FL training with ADMM]\label{proposition: kl property of L}
    Suppose that Assumption \ref{assumption: either real analytic or semialgebraic} holds, then the Lagrangian function $\ml$ defined in \eqref{eq: lagrangian function} and $\tilde{\ml}$ defined in \eqref{eq: Lyapunov function} are \kl functions.
\end{proposition}

This proposition demonstrates that most FL training models with ADMM exhibit favorable geometric properties, specifically their sharpness across points in their domains.
To establish this proposition, we first introduce several lemmas, beginning with one that highlights key characteristics of real analytic functions.
\begin{lemma}[Properties of real analytic functions \cite{krantz2002primer}]
    The sums, products, and compositions of real analytic functions are real analytic functions.
\end{lemma}
We next outline key properties of semialgebraic sets and mappings, as detailed in \cite{bochnak2013real}.
\begin{lemma}[Properties of semialgebraic functions \cite{bochnak2013real}]
\quad
\begin{itemize}
    \item[a)] The finite union, finite intersection, and complement of semialgebraic sets are semialgebraic. The closure and the interior of a semialgebraic set are semialgebraic. 
    \item[b)] The composition $g\circ h$ of semialgebraic mappings $g: A\to B$ and $h: B\to C$ is semialgebraic.
    \item[c)] The sum of two semialgebraic functions is semialgebraic.
\end{itemize}
\end{lemma}
Given that our proof incorporates both real analytic and semialgebraic functions, we require the following lemma, whose statements are derived from \cite{shiota1997geometry}.
\begin{lemma}[Subanalytic functions \cite{shiota1997geometry}]\label{lemma: Subanalytic functions}
\quad
\begin{itemize}
    \item[a)] Both real analytic functions and semialgebraic functions are subanalytic.
    \item[b)] Let $f_1$ and $f_2$ be both subanalytic functions, then the sum of $f_1$ and $f_2$, i.e., $f_1+f_2$ is a subanalytic function if at least one of them maps a bounded set to a bounded set or if both of them are nonnegative.
\end{itemize}
\end{lemma}
Additionally, we require a key lemma from \cite{bolte2007lojasiewicz}, which demonstrates that the subanalytic function is a KŁ function.
\begin{lemma}[Property of subanalytic functions \cite{bolte2007lojasiewicz}]\label{lemma: Property of subanalytic functions}
    Let $h:\mathbb{R}^p\to \mathbb{R}\cup\{+\infty\}$ be a subanalytic function with closed domain, and assume that $h$ is continuous on its domain, then $h$ is a \kl function.
\end{lemma}
Next, we are ready to prove Proposition \ref{proposition: kl property of L}.
\begin{proof}[Proof of Proposition \ref{proposition: kl property of L}]
Following from \eqref{eq: lagrangian function} and \eqref{eq: Lyapunov function}, we have
\begin{equation}
\begin{split}
    \tilde{\ml}(\bV,\bU,\bPi,\bu)&:=\ml(\bV,\bU,\bPi,\bu)+\sum_{i=1}^m \frac{29}{\rho(1-\mu_i)} \xi_i,\\
    \ml(\bV,\bU,\bPi,\bu)&:=\sum_{i=1}^m\ml_i(\bv_i,\bu_i,\bpi,\bu), \\
    \ml_i(\bv_i,\bu_i,\bpi_i,\bu)&:=\alpha_i f_i(\bv_i, \bu_i)+\langle\bpi_i,\bu_i-\bu\rangle+\frac{\rho}{2}\|\bu_i-\bu\|^2,
\end{split}
\end{equation}     
which mainly includes the following types of terms, i.e.,
\begin{align*}
    f_i(\bv_i,\bu_i),\,\langle\bpi_i,\bu_i-\bu\rangle,\,\|\bu_i-\bu\|^2, \,\xi_i.
\end{align*}
Following from Assumption \ref{assumption: either real analytic or semialgebraic}, we have $f_i(\bv_i,\bu_i)$ is either real analytic or semialgebraic. Note that $\langle\bpi_i,\bu_i-\bu\rangle,\,\|\bu_i-\bu\|^2$ are all polynomial functions with variables $\bv_i,\,\bu_i,\bpi_i,$ and $\bu$, thus according to \cite{krantz2002primer} and \cite{bochnak2013real}, they are both real analytic and semialgebraic. $\xi$ is a constant, then it is either real analytic or semialgebraic.
Since each term of $\ml$ is either real analytic or semialgebraic, according to Lemma \ref{lemma: Subanalytic functions}, $\ml$ is a subanalytic function. 
Then we can infer that $\ml$ is a \kl function according to Lemma \ref{lemma: Property of subanalytic functions}.
\end{proof}

\subsection{Sufficient Descent}\label{appendix: Sufficient Descent property}
We restate the sufficient descent property in Lemma \ref{lemma: sufficient descent} as follows.
\begin{lemma}[Restatement of Lemma \ref{lemma: sufficient descent}] \label{lemma: sufficient descent restatement}
     Suppose that Assumption \ref{assmption lipschitz} holds. Let $\{\mathcal{P}^t\}$ denote the sequence generated by Algorithm \ref{alg: fedapm}, let each client $i$ set $\max\{3\alpha_i L_u, 3\alpha_i L_{uv} \}\leq \rho \leq \frac{15}{8}\sigma_i$, then for all $t\geq 0$, it holds that
\begin{align*}
    \tilde{\ml}(\mathcal{P}^{t})-\tilde{\ml}(\mathcal{P}^{t+1})\geq a \Sigma_p^{t+1},\,\text{where}\,\,
    \Sigma_p^{t+1} \!:=\!\sum_{i=1}^m(\|\bv_i^{t+1}\!-\!\bv_i^{t}\|^2\!+\!\|\bu^{t+1}_i\!-\!\bu_i^{t}\|^2\!+\!\|\bu^{t+1}\!-\!\bu^t\|^2),\, \text{and} \,\, a:=\min\{\frac{\rho}{60} , \, \frac{\sigma_i}{2}-\frac{4\rho}{15}\}.
\end{align*}
\end{lemma}
\begin{proof}
The descent quantity in Lemma \ref{lemma: sufficient descent restatement} can be developed by considering the descent quantity along the update of each variable:  
\begin{align}\label{eq:Lt+1-Lt}
    \ml(\mathcal{P}^{t+1})-\ml(\mathcal{P}^{t})&=\ml(\bV^{t+1},\bU^{t+1},\bPi^{t+1},\bu^{t+1})-\ml(\bV^{t},\bU^{t},\bPi^{t},\bu^{t}) \notag\\
    &=\underbrace{\ml(\bV^{t+1},\bU^{t+1},\bPi^{t+1},\bu^{t+1})-\ml(\bV^{t+1},\bU^{t+1},\bPi^{t+1},\bu^{t})}_{\Delta_1^t} + \underbrace{\ml(\bV^{t+1},\bU^{t+1},\bPi^{t+1},\bu^{t})-\ml(\bV^{t+1},\bU^{t+1},\bPi^{t},\bu^{t})}_{\Delta_2^t}+\notag\\
    &\quad\underbrace{\ml(\bV^{t+1},\bU^{t+1},\bPi^{t},\bu^{t})-\ml(\bV^{t+1},\bU^{t},\bPi^{t},\bu^{t})}_{\Delta_3^t}+\underbrace{\ml(\bV^{t+1},\bU^{t},\bPi^{t},\bu^{t})-\ml(\bV^{t},\bU^{t},\bPi^{t},\bu^{t})}_{\Delta_4^t}.
\end{align}
We consider individually estimating $\Delta_1^t$, $\Delta_2^t$, $\Delta_3^t$, and $\Delta_4^t$. 
We first estimate $\Delta_1^t$.
According to the Lagarangian function \eqref{eq: lagrangian function}, we have
\begin{align}
    \Delta_1^t=&\ml(\bV^{t+1},\bU^{t+1},\bPi^{t+1},\bu^{t+1})-\ml(\bV^{t+1},\bU^{t+1},\bPi^{t+1},\bu^{t}) \notag\\
    =&\sum_{i=1}^m \alpha_i f_i(\bv_i^{t+1},\bu_i^{t+1})+\langle\bpi_i^{t+1}, \bu_i^{t+1}-\bu^{t+1} \rangle + \frac{\rho}{2}\|\bu_i^{t+1}-\bu^{t+1}\|^2- \Biggl( \sum_{i=1}^m \alpha_i f_i(\bv_i^{t+1},\bu_i^{t+1})+\langle\bpi_i^{t+1}, \bu_i^{t+1}-\bu^{t} \rangle + \frac{\rho}{2}\|\bu_i^{t+1}-\bu^{t}\|^2 \Biggr) \\
    =&\sum_{i=1}^m\langle\bpi_i^{t+1}, \bu^t-\bu^{t+1} \rangle   + \frac{\rho}{2}\langle 2\bu_i^{t+1}-\bu^{t+1}-\bu^t, \bu^t-\bu^{t+1} \rangle \\
    =&\sum_{i=1}^m \langle\bpi_i^{t+1} + \rho\bu_i^{t+1}-\frac{\rho}{2}(\bu^{t+1}+u^t), \bu^t-\bu^{t+1} \rangle  \\
    =&-\frac{\rho m}{2} \|\bu^{t+1}-\bu^t\|^2.   \label{eq:estimate of delta1,e4}
\end{align}
where \eqref{eq:estimate of delta1,e4} follows from \eqref{eq: solve u}.
Next, we estimate the value of $\Delta_2^t$.
Based on the Lagrangian function \eqref{eq: lagrangian function}, we have
\begin{align}
    \Delta_2^t=&\ml(\bV^{t+1},\bU^{t+1},\bPi^{t+1},\bu^{t})-\ml(\bV^{t+1},\bU^{t+1},\bPi^{t},\bu^{t}) \notag\\
    =&\sum_{i=1}^m \alpha_i f_i(\bv_i^{t+1},\bu_i^{t+1})+\langle\bpi_i^{t+1}, \bu_i^{t+1}-\bu^{t} \rangle + \frac{\rho}{2}\|\bu_i^{t+1}-\bu^{t}\|^2- \Biggl( \sum_{i=1}^m \alpha_i f_i(\bv_i^{t+1},\bu_i^{t+1})+\langle\bpi_i^{t}, \bu_i^{t+1}-\bu^{t} \rangle + \frac{\rho}{2}\|\bu_i^{t+1}-\bu^{t}\|^2 \Biggr) \\
    =&\sum_{i=1}^m \langle\bpi_i^{t+1}-\bpi_i^t,\bu_i^{t+1}-\bu^{t}\rangle  \\
    =& \sum_{i=1}^m  \frac{1}{\rho} \|\bpi_i^{t+1}-\bpi_i^t\|^2  \label{eq: 39}\\
    \leq & \sum_{i=1}^m  \frac{24}{\rho (1-\mu_i)}(\xi_i^{t} - \xi_i^{t+1} )+ \frac{4}{15}\rho \Bigl(\|\bu_i^{t+1} - \bu_i^{t}\|^2 + \|\bv_i^{t+1} - \bv_i^{t}\|^2 \Bigr) \label{eq: estamite delta2 e5}
\end{align}
where \eqref{eq: 39} follows from \eqref{eq: solve pi}, and \eqref{eq: estamite delta2 e5} follows from Lemma \ref{lemma: pi estimate}. Next, we proceed to estimate $\Delta_3^t$:
\begin{align}
    \Delta_3^t=&\ml(\bV^{t+1},\bU^{t+1},\bPi^{t},\bu^{t})-\ml(\bV^{t+1},\bU^{t},\bPi^{t},\bu^{t}) \notag\\
    =&\sum_{i=1}^m \alpha_i f_i(\bv_i^{t+1},\bu_i^{t+1})+\langle\bpi_i^{t}, \bu_i^{t+1}-\bu^{t} \rangle + \frac{\rho}{2}\|\bu_i^{t+1}-\bu^{t}\|^2- \Biggl( \sum_{i=1}^m \alpha_i f_i(\bv_i^{t+1},\bu_i^{t})+\langle\bpi_i^{t}, \bu_i^{t}-\bu^{t} \rangle + \frac{\rho}{2}\|\bu_i^{t}-\bu^{t}\|^2 \Biggr) \\
    =&\sum_{i=1}^m \alpha_i \bigl(f_i(\bv_i^{t+1},\bu_i^{t+1}) - f_i(\bv_i^{t+1},\bu_i^{t})\bigr) + \langle\bpi_i^t,\bu_i^{t+1}-\bu_i^t \rangle + \frac{\rho}{2}\langle \bu_i^{t+1}+\bu_i^t-2\bu^t,\bu_i^{t+1}-\bu_i^t   \rangle \\
    \leq &\sum_{i=1}^m  \langle \alpha_i\nabla_{\bu_i}f_i(\bu_i^{t+1},\bv^{t+1}_i), \bu_i^{t+1}-\bu_i^t \rangle+ \frac{\alpha_i L_u}{2} \|\bu_i^{t+1}-\bu_i^t \|^2 + \langle\bpi_i^t,\bu_i^{t+1}-\bu_i^t \rangle + \frac{\rho}{2}\langle \bu_i^{t+1}+\bu_i^t-2\bu^t,\bu_i^{t+1}-\bu_i^t   \rangle  \label{eq: estimate delta3, lipschitz}  \\
    =&\sum_{i=1}^m  \langle \alpha_i\nabla_{\bu_i}f_i(\bu_i^{t+1},\bv^{t+1}_i) + \bpi_i^t + \rho (\bu_i^{t+1}-\bu^t), \bu_i^{t+1}-\bu_i^t \rangle - \frac{\rho - \alpha_i L_u }{2} \|\bu_i^{t+1}-\bu_i^t \|^2\\
    \leq &\sum_{i=1}^m  \frac{5}{\rho} \|e_i^{t+1}\|^2+ \frac{\rho}{20}\|\bu_i^{t+1}-\bu_i^t \|^2 - \frac{\rho}{3} \|\bu_i^{t+1}-\bu_i^t \|^2  \label{eq: 47 follows from lemma jensen}\\
    \leq &\sum_{i=1}^m  - \frac{17}{60}\rho\|\bu_i^{t+1}-\bu_i^t \|^2 + \frac{5}{\rho(1-\mu_i)}(\xi_i^{t}-\xi_i^{t+1}), \label{eq: estamite delta3 e7}
\end{align}
where \eqref{eq: estimate delta3, lipschitz} follows from Lemma \ref{lemma: lipschitz}, and \eqref{eq: 47 follows from lemma jensen} follows from Lemma \ref{jensen}.
Next, we proceed to estimate the value of $\Delta_4^t$
By Algorithm \ref{alg: fedapm}, $\bv_i^{t+1}$ is updated according to the following
\begin{align}\label{eq: vi update}
    \bv_i^{t+1} \leftarrow \operatorname{argmin}_{\bv_i} \Bigl\{\alpha_i f_i(\bv_i,\bu_i^{t}) + \frac{\sigma_i}{2} \|\bv_i-\bv_i^t\|^2   \Bigr\}.
\end{align}
Let $\bar{h}^{t+1}(\bv_i)=\alpha_i f_i(\bv_i,\bu_i^{t}) + \frac{\sigma_i}{2} \|\bv_i-\bv_i^t\|^2$. When $\sigma_i$ is sufficiently large such that $\bar{h}^{t+1}(\bv_i)$ becomes strongly convex with respect to $\bv_i$, the optimality of $\bv_i^{t+1}$ implies the following:
\begin{align}
    \bar{h}^{t+1}(\bv_i^{t}) \geq \bar{h}^{t+1}(\bv_i^{t+1}),
\end{align}
which implies
\begin{align}\label{eq: sufficient descent of vi}
   \alpha_i f_i(\bv_i^t,\bu_i^{t})\geq\alpha_i f_i(\bv_i^{t+1},\bu_i^{t}) +  \frac{\sigma_i}{2} \|\bv_i^{t+1}-\bv_i^t\|^2.
\end{align}
Next, according to the Lagrangian function, we obtain
\begin{align}
    \Delta_4^t=&\ml(\bV^{t+1},\bU^{t},\bPi^{t},\bu^{t})-\ml(\bV^{t},\bU^{t},\bPi^{t},\bu^{t}) \notag\\
    =&\sum_{i=1}^m \alpha_i f_i(\bv_i^{t+1},\bu_i^{t})+\langle\bpi_i^{t}, \bu_i^{t}-\bu^{t} \rangle + \frac{\rho}{2}\|\bu_i^{t}-\bu^{t}\|^2- \Biggl( \sum_{i=1}^m \alpha_i f_i(\bv_i^{t},\bu_i^{t})+\langle\bpi_i^{t}, \bu_i^{t}-\bu^{t} \rangle + \frac{\rho}{2}\|\bu_i^{t}-\bu^{t}\|^2 \Biggr) \\
    =&\sum_{i=1}^m \alpha_i \Bigl(f_i(\bv_i^{t+1},\bu_i^{t})-f_i(\bv_i^{t},\bu_i^{t})\Bigr)\\
    =&\sum_{i\in\mathcal{S}^t} \alpha_i \Bigl(f_i(\bv_i^{t+1},\bu_i^{t})-f_i(\bv_i^{t},\bu_i^{t})\Bigr) + \sum_{i\notin\mathcal{S}^t} \alpha_i \Bigl(f_i(\bv_i^{t+1},\bu_i^{t})-f_i(\bv_i^{t},\bu_i^{t})\Bigr)\\
    \leq&\sum_{i=1}^m-\frac{\sigma_i}{2}\|\bv_i^{t+1}-\bv_i^t\|^2, \label{eq: delta4 result}
\end{align}
where \eqref{eq: delta4 result} follows from \eqref{eq: sufficient descent of vi} and the fact that for each client $i\notin\mathcal{S}^t$,
\begin{align}\label{eq: fact of delta4 result}
    \sum_{i\notin\mathcal{S}^t} \alpha_i \Bigl(f_i(\bv_i^{t+1},\bu_i^{t})-f_i(\bv_i^{t},\bu_i^{t})\Bigr)=0=-\frac{\sigma_i}{2}\|\bv_i^{t+1}-\bv_i^t\|^2.
\end{align}
Next, substituting (\ref{eq:estimate of delta1,e4}), (\ref{eq: estamite delta2 e5}), (\ref{eq: estamite delta3 e7}), and (\ref{eq: delta4 result}) into (\ref{eq:Lt+1-Lt}), we obtain
\begin{align}
    \ml(\mathcal{P}^{t+1})-\ml(\mathcal{P}^{t})\leq &-\frac{\rho m}{2} \|\bu^{t+1}-\bu^t\|^2 + \sum_{i=1}^m \frac{24}{\rho (1-\mu_i)}(\xi_i^{t} - \xi_i^{t+1} )+ \frac{4}{15}\rho\Bigl(\|\bu_i^{t+1} - \bu_i^{t}\|^2 + \|\bv_i^{t+1} - \bv_i^{t}\|^2 \Bigr) \notag \\
    &- \frac{17}{60}\rho\|\bu_i^{t+1}-\bu_i^t \|^2 + \frac{5}{\rho(1-\mu_i)}(\xi_i^{t}-\xi_i^{t+1})-\frac{\sigma_i}{2}\|\bv_i^{t+1}-\bv_i^t\|^2 \\
    = &-\sum_{i=1}^m  \Bigl(\frac{\rho}{60} \|\bu_i^{t+1} - \bu_i^{t}\|^2 + \Bigl(\frac{\sigma_i}{2}-\frac{4\rho}{15}\Bigr)\|\bv_i^{t+1} - \bv_i^{t}\|^2 +\frac{\rho }{2} \|\bu^{t+1}-\bu^t\|^2 + \frac{29}{\rho(1-\mu_i)}(\xi_i^{t+1}-\xi_i^{t})\Biggr), \label{eq: sufficient descent final result}
    \end{align}
Then we can infer from \eqref{eq: sufficient descent final result} that
\begin{align}
    \tilde{\ml}(\mathcal{P}^{t}) - \tilde{\ml}(\mathcal{P}^{t+1})\geq a\sum_{i=1}^m \Bigl(\|\bu_i^{t+1} - \bu_i^{t}\|^2 + \|\bv_i^{t+1} - \bv_i^{t}\|^2 + \|\bu^{t+1}-\bu^t\|^2 \Bigr),
\end{align}
which completes the proof.
\end{proof}

\subsection{Relative Error}\label{appendix: Relative Error property}
\begin{lemma}[Restatement of Lemma \ref{lemma: relative error}]
    Suppose that Assumption \ref{assmption lipschitz} holds. Let $\{\mathcal{P}^t\}$ denote the sequence generated by Algorithm \ref{alg: fedapm}, and let $\tilde{\Xi}^{t+1}:=\sum_{i=1}^m(\xi_i^{t} - \xi_i^{t+1} )$, let each client $i$ choose $\sigma_i\geq \alpha_i L_v$ and $\max\{3\alpha_i L_u, 3\alpha_i L_{uv} \}\leq \rho \leq \frac{15}{8}\sigma_i$, then for all $t\geq 0$, it holds that
\begin{align}
    \text{dist}(\boldsymbol{0}, \partial \tilde{\ml} (\mathcal{P}^t))^2\leq b (\Sigma_p^{t+1} + \tilde{\Xi}^{t+1}),
\end{align}
where $b:=\max\{\frac{22}{5}\sigma_i^2 + \frac{4}{3}\rho^2+\frac{8\rho}{15}, \frac{16}{3}\rho^2+2+\frac{8\rho}{15}, \frac{48+4\rho}{\rho(1-\mu_i)}\}$, $\text{dist}(0,\mathcal{C}):=\inf_{\boldsymbol{c}\in \mathcal{C}} \|\boldsymbol{c}\|$ for a set $\mathcal{C}$, and 
\begin{align}
    \partial \tilde{\ml} (\mathcal{P}^t):= (\{\nabla_{\bv_i}{\tilde{\ml}}\}_{i=1}^m,\{\nabla_{\bu_i}{\tilde{\ml}}\}_{i=1}^m,\{\nabla_{\bpi_i}{\tilde{\ml}}\}_{i=1}^m,\{\nabla_{\bu}{\tilde{\ml}}\})(\mathcal{P}^t).
\end{align}
\end{lemma}
\begin{proof}
We will establish the relative error condition via bounding each term of $\partial \tilde{\ml} (\mathcal{P}^t)$.
First, we establish an upper bound for $\|\nabla_{\bv_i} {\tilde{\ml}}(\mathcal{P}^t)\|^2$.
By utilizing the Lagrangian function, we derive   
\begin{align}
    \|\nabla_{\bv_i} {\tilde{\ml}}(\mathcal{P}^t)\|^2 = & \|\alpha_i \nabla_{\bv_i} f_i(\bv_i^t,\bu_i^t) \|^2   \notag\\
    = & \|\alpha_i  \nabla_{\bv_i} f_i(\bv_i^t,\bu_i^t) -  \alpha_i  \nabla_{\bv_i} f_i(\bv_i^{t+1 },\bu_i^t) + \alpha_i  \nabla_{\bv_i} f_i(\bv_i^{t+1 },\bu_i^t) \|^2\\
    \leq & 2\|\alpha_i  \nabla_{\bv_i} f_i(\bv_i^t,\bu_i^t) -  \alpha_i  \nabla_{\bv_i} f_i(\bv_i^{t+1 },\bu_i^t)\|^2 +  2\|\alpha_i  \nabla_{\bv_i} f_i(\bv_i^{t+1 },\bu_i^t) \|^2   \label{eq: 63 follows from jensen}\\
    \leq & 2 \alpha_i^2 L_v^2 \|\bv_i^{t+1}-\bv_i^t\|^2 + 2\|\alpha_i  \nabla_{\bv_i} f_i(\bv_i^{t+1 },\bu_i^t)  + \sigma_i (\bv_i^{t+1}-\bv_i^t) - \sigma_i (\bv_i^{t+1}-\bv_i^t)\|^2  \label{eq: 64 follows from lipschitz}\\
    \leq & 2 \alpha_i^2 L_v^2 \|\bv_i^{t+1}-\bv_i^t\|^2 + \underbrace{12\|\alpha_i  \nabla_{\bv_i} f_i(\bv_i^{t+1 },\bu_i^t)  + \sigma_i (\bv_i^{t+1}-\bv_i^t) \|^2}_{=0} + \frac{12}{5}\| \sigma_i (\bv_i^{t+1}-\bv_i^t)\|^2  \label{eq: 66 follows from optimal condition}\\
    \leq & \frac{22}{5} \sigma_i^2 \|\bv_i^{t+1}-\bv_i^t\|^2,  \label{eq: eq: relative error vi e7}
\end{align}
where \eqref{eq: 63 follows from jensen} follows from Lemma \ref{jensen}, \eqref{eq: 64 follows from lipschitz} follows from Assumption \ref{assmption lipschitz}, \eqref{eq: 66 follows from optimal condition} follows from \eqref{eq: 15 solve v with proximal update}, and \eqref{eq: eq: relative error vi e7} follows from the fact that $\sigma_i\geq \alpha_i L_v$.
Next, we derive an upper bound for $\|\nabla_{\bu_i} {\tilde{\ml}}(\mathcal{P}^t)\|^2$. 
By utilizing the Lagrangian function, we derive   
\begin{align}
    \|\nabla_{\bu_i} {\tilde{\ml}}(\mathcal{P}^t)\|^2 =  & \| \alpha_i \nabla_{\bu_i} f_i(\bv_i^t,\bu_i^t) + \bpi_i^t + \rho (\bu_i^t - \bu^t) \|^2\\
    = &\|\alpha_i \nabla_{\bu_i} f_i(\bv_i^t,\bu_i^t) - \alpha_i \nabla_{\bu_i} f_i(\bv_i^{t+1},\bu_i^t) + \alpha_i \nabla_{\bu_i} f_i(\bv_i^{t+1},\bu_i^t) - \alpha_i \nabla_{\bu_i} f_i(\bv_i^{t+1},\bu_i^{t+1}) + \alpha_i \nabla_{\bu_i} f_i(\bv_i^{t+1},\bu_i^{t+1})     \notag\\
    & \quad \quad \quad  +\bpi_i^t + \rho (\bu_i^{t+1} - \bu^t) -\rho (\bu_i^{t+1} - \bu_i^{t})\|^2   \\
    \leq &  4\|\alpha_i \nabla_{\bu_i} f_i(\bv_i^t,\bu_i^t) - \alpha_i \nabla_{\bu_i} f_i(\bv_i^{t+1},\bu_i^t) \|^2 + 4 \|\alpha_i \nabla_{\bu_i} f_i(\bv_i^{t+1},\bu_i^t) - \alpha_i \nabla_{\bu_i} f_i(\bv_i^{t+1},\bu_i^{t+1}) \|^2\notag \\
    & \quad \quad \quad + 4 \| \alpha_i \nabla_{\bu_i} f_i(\bv_i^{t+1},\bu_i^{t+1}) +\bpi_i^t + \rho (\bu_i^{t+1} - \bu^t) \|^2 + 4 \|\rho (\bu_i^{t+1} - \bu_i^{t})\|^2  \label{eq: 69 follows from jensen}\\
    \leq &   4 \alpha_i^2 L_{uv}^2 \|\bv_i^{t+1}-\bv_i^t\|^2 + 4\alpha_i^2 L_u^2 \|\bu_i^{t+1}-\bu_i^t\|^2 + 4\xi_i^{t+1} + 4\rho^2 \|\bu_i^{t+1}-\bu_i^t\|^2 \label{eq: 70 follows from lipschitz} \\ 
    \leq & \frac{4}{3}\rho^2 \|\bv_i^{t+1}-\bv_i^t\|^2 + \frac{16}{3}\rho^2 \|\bu_i^{t+1}-\bu_i^t\|^2 + \frac{4}{1-\mu_i}(\xi_i^{t}-\xi_i^{t+1}), \label{eq: relative error ui}
\end{align}
where \eqref{eq: 69 follows from jensen} follows from Lemma \ref{jensen}, \eqref{eq: 70 follows from lipschitz} follows from Assumption \ref{assmption lipschitz}, and \eqref{eq: relative error ui} follows from the fact that $\rho \geq \max\{3\alpha_i L_u, 3\alpha_i L_{uv}\}$.
Next, we establish the upper bound of $\|\nabla_{\bpi_i} {\tilde{\ml}}(\mathcal{P}^t)\|^2$. By utilizing the Lagrangian function, we obtain   
\begin{align}
    \|\nabla_{\bpi_i} {\tilde{\ml}}(\mathcal{P}^t)\|^2 =  & \|\bu_i^t - \bu^t \|^2  \\
    =& \|\bu_i^t - \bu_i^{t+1} + \bu_i^{t+1}- \bu^t \|^2 \\
    \leq & 2\|\bu_i^{t+1} - \bu_i^t \|^2 + 2\|\bu_i^{t+1}- \bu^t \|^2 \label{eq: 73 follows from jensen} \\
    = & 2\|\bu_i^{t+1} - \bu_i^t \|^2 + \frac{2}{\rho}\|\bpi_i^{t+1}- \bpi_i^t \|^2 \label{eq: 74 follows from pi update}  \\
    \leq & 2\|\bu_i^{t+1} - \bu_i^t \|^2 + \frac{48}{\rho (1-\mu_i)}(\xi_i^{t} - \xi_i^{t+1} )+ \frac{8}{15}\rho \Bigl(\|\bu_i^{t+1} - \bu_i^{t}\|^2 + \|\bv_i^{t+1} - \bv_i^{t}\|^2 \Bigr)\label{eq: relative error pi}
\end{align}
where \eqref{eq: 73 follows from jensen} follows from Lemma \ref{jensen}, \eqref{eq: 74 follows from pi update} follows from \eqref{eq: solve pi}, and \eqref{eq: relative error pi} follows from Lemma \ref{lemma: pi estimate}, and 
Next, we solve the upper bound of $\|\nabla_{\bu} {\tilde{\ml}}(\mathcal{P}^t)\|^2$. By the definition of the Lagrangian function, we obtain   
\begin{align}
     \|\nabla_{\bu} {\tilde{\ml}}(\mathcal{P}^t)\|^2 =  \sum_{i=1}^m -\bpi_i^t + \rho(\bu^t - \bu_i^t) =0, \label{eq: relative error w}
\end{align}
where \eqref{eq: relative error w} follows from \eqref{eq: solve u}. Summing over \eqref{eq: eq: relative error vi e7}, \eqref{eq: relative error ui}, \eqref{eq: relative error pi}, and \eqref{eq: approximate solution about u}, we obtain
\begin{align}
    \text{dist}(\boldsymbol{0},\partial \ml (\mathcal{P}^t)) ^2 = &\|\nabla_{\bu} {\tilde{\ml}}(\mathcal{P}^t)\|^2+\sum_{i=1}^m\|\nabla_{\bv_i} {\tilde{\ml}}(\mathcal{P}^t)\|^2+\|\nabla_{\bu_i} {\tilde{\ml}}(\bu_i^t)\|^2+\|\nabla_{\bpi_i} {\tilde{\ml}}(\bpi_i^t)\|^2   \\
     \leq  &\sum_{i=1}^m\frac{22}{5} \sigma_i^2 \|\bv_i^{t+1}-\bv_i^t\|^2  +\frac{4}{3}\rho^2 \|\bv_i^{t+1}-\bv_i^t\|^2 + \frac{16}{3}\rho^2 \|\bu_i^{t+1}-\bu_i^t\|^2 + \frac{4}{1-\mu_i}(\xi_i^{t}-\xi_i^{t+1})  \notag  \\
    & \quad \quad+2\|\bu_i^{t+1} - \bu_i^t \|^2 + \frac{48}{\rho(1-\mu_i)}(\xi_i^{t} - \xi_i^{t+1} )+ \frac{8\rho}{15} \Bigl(\|\bu_i^{t+1} - \bu_i^{t}\|^2 + \|\bv_i^{t+1} - \bv_i^{t}\|^2 \Bigr) \\
    = &\sum_{i=1}^m \Biggl(\Bigl(\frac{22}{5}\sigma_i^2 + \frac{4}{3}\rho^2+\frac{8\rho}{15} \Bigr) \|\bv_i^{t+1}-\bv_i^t\|^2 + \Bigl(\frac{16}{3}\rho^2+2+\frac{8\rho}{15} \Bigr)\|\bu_i^{t+1} - \bu_i^t \|^2+\frac{48+4\rho}{\rho(1-\mu_i)}(\xi_i^{t} - \xi_i^{t+1} )  \Biggr)\\
    \leq &  \sum_{i=1}^m b \Bigl(\|\bu_i^{t+1} - \bu_i^{t}\|^2 + \|\bv_i^{t+1} - \bv_i^{t}\|^2 + \|\bu^{t+1}-\bu^t\|^2  + (\xi_i^{t} - \xi_i^{t+1} )  \Bigr),
\end{align}
which completes the proof.
\end{proof}

\subsection{Proof of Theorem \ref{theorem: convergence}}\label{proof of theorem 1(theorem: convergence)}
Before we formally prove Theorem \ref{theorem: convergence}, we present an important lemma.
\begin{lemma}\label{lemma: non increasing results}
    Let $\mathcal{P}^t$ be the sequence generated by Algorithm \ref{alg: fedapm}, let each client $i$ choose $\max\{3\alpha_i L_u, 3\alpha_i L_{uv} \}\leq \rho \leq \frac{15}{8}\sigma_i$ and $\sigma_i\geq \alpha_i L_v$, then the following results hold under Assumption \ref{assmption lipschitz}.
    \begin{itemize}
        \item[a)] Sequences $\{\tilde{L}(\mathcal{P}^t)\}$ is non-increasing.
        \item[b)] Suppose $\rho\geq 2L_u$, then $\tilde{L}(\mathcal{P}^t) - \frac{28}{\rho}\Xi^{t}\geq f(\bv^t,\bu^t)\geq f^*\geq-\infty$ for any $t\geq 1$.
        \item[c)] For any $i\in[m]$, the limits of the following terms are zero,
        \begin{align}\label{eq: 71 in lemma 9}
            (\xi^{t+1}_i-\xi_i^t,\bv_i^{t+1}-\bv_i^{t},\bu^{t+1}-\bu^{t},\bu_i^{t+1}-\bu_i^{t},\bu_i^{t+1}-\bu^{t},\bpi_i^{t+1}-\bpi_i^{t})\to \boldsymbol{0}.
        \end{align}
    \end{itemize}
\end{lemma}
\begin{proof}
a) Following from Lemma \ref{lemma: sufficient descent restatement}, we have
\begin{align}
    \tilde{\ml}(\mathcal{P}^{t})-\tilde{\ml}(\mathcal{P}^{t+1})\geq a \sum_{i=1}^m(\|\bv^{t+1}-\bv^{t}\|^2+\|\bu^{t+1}_i-\bu_i^{t}\|^2+\|\bu^{t+1}-\bu^t\|^2)\geq 0,
\end{align}
which implies $\{\tilde{\ml}(\mathcal{P}^t)\}$ is non-increasing.

b) By the definition of $\tilde{\ml}$, we obtain
\begin{align}
    \tilde{\ml}(\mathcal{P}^t) = &\sum_{i=1}^m  \Biggl( \alpha_i f_i(\bv_i^t, \bu_i^t)+\langle \bpi_i^t, \bu_i^t-\bu^t \rangle + \frac{\rho}{2}\|\bu_i^t -\bu^t\|^2 +  \frac{29}{\rho(1-\mu_i)}\xi_i^{t} \Biggr) \\
    =&\sum_{i=1}^m \Biggl( \alpha_i f_i(\bv_i^t, \bu_i^t)+\langle \nabla_{\bu_i}f_i(\bv_i^t,\bu_i^t), \bu^t - \bu_i^t \rangle + \langle e_i^t,  \bu_i^t-\bu^t  \rangle+ \frac{\rho}{2}\|\bu_i^t -\bu^t\|^2 +  \frac{29}{\rho(1-\mu_i)}\xi_i^{t} \Biggr)  \label{eq: 82 follows from eq: error term}\\
    \geq & \sum_{i=1}^m \Biggl( \alpha_i f_i(\bv_i^t, \bu^t) + \frac{\rho-2 L_u}{2}\|\bu_i^t -\bu^t\|^2 +\Bigl(\frac{29}{\rho(1-\mu_i)} - \frac{1}{2L_u} \Bigr)\xi_i^{t}  \Biggr)  \label{eq: 84 follows from lipschitz and jensen}\\
    \geq &  \sum_{i=1}^m \Biggl( \alpha_i f_i(\bv_i^t, \bu^t) + \Bigl(\frac{29}{\rho(1-\mu_i)} - \frac{1}{\rho} \Bigr)\xi_i^{t}  \Biggr) \label{eq: 85 follows from rho geq L_u}\\
    \geq &f^*+ \frac{28}{\rho}\Xi^{t} > -\infty,\label{eq: 87 geq infty}
\end{align}
where \eqref{eq: 82 follows from eq: error term} follows from \eqref{eq: error term}, \eqref{eq: 84 follows from lipschitz and jensen} follows from Lemma \ref{lemma: lipschitz} and Lemma \ref{jensen}, \eqref{eq: 85 follows from rho geq L_u} follows from the fact that $\rho\geq L_u$.

c) Following from the fact that $\xi^{t+1}\leq \mu_i\xi_i^t$ and $0<\mu_i<1$, then it can be easily indicated that $\xi_i^{t+1} - \xi_i^{t} \to 0$.
Next, following from Lemma \ref{lemma: sufficient descent restatement} and $\tilde{\ml(\mathcal{P}^t)}>-\infty$, we obtain
\begin{align}
    a\sum_{t=0}^{\infty}\Sigma_p^{t+1}\leq \sum_{t=0}^{\infty} \Bigl(\tilde{\ml}(\mathcal{P}^{t})-\tilde{\ml}(\mathcal{P}^{t+1})\Bigr)\leq\tilde{\ml}(\mathcal{P}^{0})-\tilde{\ml}(\mathcal{P}^{\infty}) <+\infty, \label{eq: 88 follows from eq 87}
\end{align}
where \eqref{eq: 88 follows from eq 87} follows from \eqref{eq: 87 geq infty}. This implies the sequences $\|\bu^{t+1}-\bu^{t}\|\to0$, $\|\bv_{i}^{t+1}-\bv_{i}^{t}\|\to0$, and $\|\bu_{i}^{t+1}-\bu_{i}^{t}\|\to 0$. 
Next, according to Lemma \ref{lemma: pi estimate}, we have 
\begin{align}
    \|\bpi_i^{t+1}-\bpi_i^t\|^2\leq \frac{24}{1-\mu_i}(\xi_i^{t} - \xi_i^{t+1} )+ \frac{4}{15}\rho^2\Bigl(\|\bu_i^{t+1} - \bu_i^{t}\|^2 + \|\bv_i^{t+1} - \bv_i^{t}\|^2 \Bigr),
\end{align}
which implies $\|\bpi_i^{t+1}-\bpi_i^{t}\|\to 0$.
Furthermore, following from the fact that $\boldsymbol{\pi}_i^{t+1}=\boldsymbol{\pi}_i^t+\rho(\bu_i^{t+1}-\bu^{t})$, we have
\begin{align}
    \|\bu_i^{t+1}-\bu^t\|=\frac{1}{\rho}\|\bpi_i^{t+1}-\bpi_i^{t}\|\to 0,
\end{align}
which completes the proof.
\end{proof}

\begin{theorem}[Restatement of Theorem \ref{theorem: convergence}] \label{theorem: restatement of theorem 1}
    Suppose that Assumption \ref{assmption lipschitz} holds, let each client $i$ choose $\max\{3\alpha_i L_u, 3\alpha_i L_{uv}, 2L_u \}\leq \rho \leq \frac{15}{8}\sigma_i$ and $\sigma_i\geq \alpha_i L_v$, then the following results hold.
\begin{itemize}
    \item[a)] Sequence $\{\mathcal{P}^t\}$ is bounded.
    \item[b)] 
    The gradients of $f(\bV^{t+1},\bU^{t+1})$ and $f(\bV^{t+1},\bu^{t+1})$ with respect to each variable eventually vanish, i.e.,
    \begin{equation}
    \begin{split}
        \lim_{t\to \infty}\nabla_{\bV}f(\bV^{t+1},\bu^{t+1}) = \lim_{t\to \infty}\nabla_{\bV}f(\bV^{t+1},\bU^{t+1}) =\lim_{t\to \infty}\nabla_{\bu}f(\bV^{t+1},\bu^{t+1})= \lim_{t\to \infty}\nabla_{\bU}f(\bV^{t+1},\bU^{t+1}) \to 0.
    \end{split}
    \end{equation}
    \item[c)] Sequences $\{\tilde{\ml}(\mathcal{P}^t)\}$, $\{\ml(\mathcal{P}^t)\}$, $\{f(\bV^t,\bu_i^t)\}$, and $\{f(\bV^t,\bu^t)\}$ converge to the same value, i.e.,
    \begin{align}
        \lim_{t\to \infty}\tilde{\ml}(\mathcal{P}^t)= \lim_{t\to \infty}\ml(\mathcal{P}^t)= \lim_{t\to \infty}f(\bV^t,\bu_i^t)   =\lim_{t\to \infty}f(\bV^t,\bu^t).
    \end{align}
\end{itemize}
\end{theorem}

\begin{proof}
a) Based on Lemma \ref{lemma: non increasing results}, we have $\tilde{\ml}(\mathcal{P}^1) - \sum_{i=1}^m\Bigl(\frac{29}{\rho(1-\mu_i)} - \frac{1}{2L_u} \Bigr)\xi_i^{1}\geq f(\bV^t,\bu^t)=\sum_{i=1}^{m}\alpha_i f_i(\bv_i^t, \bu^t)$. 
Along with the coercive property of $f_i$, we can infer that the sequences $\{\bV^t\}$ and $\{\bu^t\}$ are bounded. Following from the fact that $\|\bu_i^{t+1}-\bu^{t}\|\to 0$, we can infer that the sequence $\{\bu_i^t\}$ is bounded. Finally, following the fact that $\|\bpi_i^{t+1} - \bpi_i^t\| = \frac{1}{\rho}\|\bu_i^{t+1} - \bu^t\|\to 0$, we can infer that $\{\bpi_i^t\|$ is bounded.

b) We begin by defining
\begin{align}
    g_{\bv}(\bv_i,\bu_i):&=\nabla_{\bv_i} \Bigl \{\ml(\bv_i,\bu_i,\bpi_i,\bu) + \frac{\sigma_i}{2}\|\bv_i - \bv_i^t\|  \Bigr\} =\alpha_i\nabla_{\bv_i} f_i(\bv_i, \bu_i)+\sigma_i(\bv_i - \bv_i^t)  \label{eq: gv}\\
    g_{\bu}(\bv_i,\bu_i):&=\nabla_{\bu_i} \ml(\bv_i,\bu_i,\bpi_i,\bu) = \alpha_i \nabla_{\bu_i} f_i(\bv_i, \bu_i) +  \bpi_i +  \rho (\bu_i -\bu )  \label{eq: gu}.
\end{align}
Then we define 
\begin{align}
    g_{\bv}(\bv_i^{t+1},\bu_i^t)&:=\alpha_i\nabla_{\bv_i} f_i(\bv_i^{t+1},\bu_i^{t+1})+\sigma_i(\bv_i^{t+1} - \bv_i^t), \\
    g_{\bu}(\bv_i^{t+1},\bu_i^{t+1}) &:= \alpha_i \nabla_{\bu_i} f_i(\bv_i^{t+1}, \bu_i^{t+1}) +  \bpi_i^t +  \rho (\bu_i^{t+1} -\bu^t ) = \alpha_i \nabla_{\bu_i} f_i(\bv_i^{t+1}, \bu_i^{t+1}) +  \bpi_i^{t+1},
\end{align}
which implies $\|g_{\bv}(\bv_i^{t+1},\bu_i^t)\|=0$ and $\| g_{\bu}(\bv_i^{t+1},\bu_i^{t+1})\| \leq \sqrt{\xi_{i}^{t+1}}$. Next, we consider bounding the following terms:
\begin{align}
    \|\nabla_{\bV} f(\bV^{t+1}, \bu^{t+1})\| = & \|\sum_{i=1}^m \alpha_i \nabla_{\bv_i} f_i(\bv_i^{t+1}, \bu^{t+1}) - \alpha_i \nabla_{\bv_i} f_i(\bv_i^{t+1}, \bu^t) + \alpha_i \nabla_{\bv_i} f_i(\bv_i^{t+1}, \bu^t) - \alpha_i \nabla_{\bv_i} f_i(\bv_i^{t+1}, \bu_i^{t+1}) +\alpha_i \nabla_{\bv_i} f_i(\bv_i^{t+1}, \bu_i^{t+1}) \notag\\
    & \quad\quad - \alpha_i \nabla_{\bv_i} f_i(\bv_i^{t+1}, \bu_i^{t}) +  g_{\bv}(\bv_i^{t+1},\bu_i^t) - \sigma_i(\bv_i^{t+1} - \bv_i^t) \|\\
    \leq & \sum_{i=1}^m \Bigl( \|\alpha_i \nabla_{\bv_i} f_i(\bv_i^{t+1}, \bu^{t+1}) - \alpha_i \nabla_{\bv_i} f_i(\bv_i^{t+1}, \bu^t)\| +  \| \alpha_i \nabla_{\bv_i} f_i(\bv_i^{t+1}, \bu^t) - \alpha_i \nabla_{\bv_i} f_i(\bv_i^{t+1}, \bu_i^{t+1})\|   \notag\\
   &\quad\quad \|\alpha_i \nabla_{\bv_i} f_i(\bv_i^{t+1}, \bu_i^{t+1})-\alpha_i \nabla_{\bv_i} f_i(\bv_i^{t+1}, \bu_i^{t})\| +\| g_{\bv}(\bv_i^{t+1},\bu_i^t)\| +  \| \sigma_i(\bv_i^{t+1} - \bv_i^t) \| \Bigr)\\
    \leq & \sum_{i=1}^m \Bigl( \alpha_i L_{vu}\|\bu^{t+1}-\bu^t \| + \alpha_i L_{vu}\|\bu_i^{t+1}-\bu^t\|+  \alpha_i L_{vu}\|\bu_i^{t+1}-\bu_i^t\|+  \| g_{\bv}(\bv_i^{t+1},\bu_i^t)\| +  \| \sigma_i(\bv_i^{t+1} - \bv_i^t) \| \Bigr)\to 0  \label{eq: 98 follows from lipschitz}
\end{align}
\begin{align}
    \|\nabla_{\bV} f(\bV^{t+1}, \bU^{t+1})\| = & \|\sum_{i=1}^m \alpha_i \nabla_{\bv_i} f_i(\bv_i^{t+1}, \bu_i^{t+1}) -\alpha_i \nabla_{\bv_i} f_i(\bv_i^{t+1}, \bu_i^{t}) +  g_{\bv}(\bv_i^{t+1},\bu_i^t) - \sigma_i(\bv_i^{t+1} - \bv_i^t) \|   \\
    \leq &\sum_{i=1}^m  \Bigl( \|\alpha_i \nabla_{\bv_i} f_i(\bv_i^{t+1}, \bu_i^{t+1}) -\alpha_i \nabla_{\bv_i} f_i(\bv_i^{t+1}, \bu_i^{t}) \| + \| g_{\bv}(\bv_i^{t+1},\bu_i^t) \|   +  \| \sigma_i(\bv_i^{t+1} - \bv_i^t) \|  \Bigr)\\
    \leq & \sum_{i=1}^m \Bigl(\alpha_i  L_{vu} \|\bu_i^{t+1} - \bu_i^{t}\| + \| g_{\bv}(\bv_i^{t+1},\bu_i^t) \|   +  \| \sigma_i(\bv_i^{t+1} - \bv_i^t) \|  \Bigr)\to 0, \label{eq: 101 follows from lipschitz}
\end{align}
\begin{align}
    \|\nabla_{\bU} f(\bV^{t+1}, \bU^{t+1})\| = & \|\sum_{i=1}^m \alpha_i  g_{\bu}(\bv_i^{t+1},\bu_i^{t+1}) - \bpi_i^{t+1}\|\\
    \leq & \|\sum_{i=1}^m \alpha_i  g_{\bu}(\bv_i^{t+1},\bu_i^{t+1})\| + \|\sum_{i=1}^m  \bpi_i^{t+1}\|\\
    = &  \|\sum_{i=1}^m \alpha_i  g_{\bu}(\bv_i^{t+1},\bu_i^{t+1})\| + \rho \|\sum_{i=1}^m  (\bu^{t+1} - \bu_i^{t+1})\|\\
    = &  \|\sum_{i=1}^m \alpha_i  g_{\bu}(\bv_i^{t+1},\bu_i^{t+1})\| + \rho \|\sum_{i=1}^m  (\bu^{t+1} - \bu^t + \bu^t - \bu_i^{t+1})\|\\
    \leq & \|\sum_{i=1}^m \alpha_i g_{\bu}(\bv_i^{t+1},\bu_i^{t+1})\| + \rho \|\sum_{i=1}^m  (\bu^{t+1} - \bu^t )\| +\rho  \| \sum_{i=1}^m (\bu^t - \bu_i^{t+1})\| \to 0.\label{eq: 106 follows from sequence convergence}
\end{align}
\begin{align}
    \|\nabla_{\bU} f(\bV^{t+1}, \bu^{t+1})\| = & \|\sum_{i=1}^m \alpha_i \nabla_{\bu_i} f_i(\bv_i^{t+1}, \bu^{t+1}) -\alpha_i \nabla_{\bu_i} f_i(\bv_i^{t+1}, \bu^{t})+\alpha_i \nabla_{\bu_i} f_i(\bv_i^{t+1}, \bu^{t})-\alpha_i \nabla_{\bu_i} f_i(\bv_i^{t+1}, \bu_i^{t+1}) +g_{\bu}(\bv_i^{t+1},\bu_i^{t+1}) -\bpi_i^{t+1}\|\notag\\
    \leq & \sum_{i=1}^m \|\alpha_i \nabla_{\bu_i} f_i(\bv_i^{t+1}, \bu^{t+1}) -\alpha_i \nabla_{\bu_i} f_i(\bv_i^{t+1}, \bu^{t})\| + \|\alpha_i \nabla_{\bu_i} f_i(\bv_i^{t+1}, \bu^{t})-\alpha_i \nabla_{\bu_i} f_i(\bv_i^{t+1}, \bu_i^{t+1}) \| +\|g_{\bu}(\bv_i^{t+1},\bu_i^{t+1})\| + \|\bpi_i^{t+1}\| \notag\\
    \leq &\sum_{i=1}^m  \alpha_i L_{u} \Bigl( \|\bu^{t+1}-\bu^t\|+\|\bu_i^{t+1}-\bu^t\|   \Bigr)+\|g_{\bu}(\bv_i^{t+1},\bu_i^{t+1})\| + \rho \|\sum_{i=1}^m  (\bu^{t+1} - \bu^t )\| +\rho  \| \sum_{i=1}^m (\bu^t - \bu_i^{t+1})\| \to 0.\label{eq: 107 follows from lipschitz}
\end{align}
where \eqref{eq: 98 follows from lipschitz}, \eqref{eq: 101 follows from lipschitz}, \eqref{eq: 106 follows from sequence convergence}, and \eqref{eq: 107 follows from lipschitz} follow from Assumption \ref{assmption lipschitz} and Lemma \ref{lemma: non increasing results}.

c) Lemma \ref{lemma: non increasing results} indicates that $\{\tilde{\ml}(\mathcal{P}^{t+1})\}$ is non-increasing and lower bounded. Then it is not hard to indicate that $\tilde{\ml}(\mathcal{P}^{t+1})\to \ml(\mathcal{P}^{t+1})$ due to the fact that $\xi_i^{t+1}\to 0$. By definition, we have
\begin{align}
    \ml(\mathcal{P}^{t+1})-f(\bV^{t+1},\bU^{t+1})
    =&\sum_{i=1}^m\Bigl(\langle\bpi_i^{t+1},\bu_i^{t+1} - \bu^{t+1}   \rangle + \frac{\rho}{2}\|\bu_i^{t+1} - \bu^{t+1} \|^2 \Bigr)\\
    = & \sum_{i=1}^m\Bigl( \langle e_i^{t+1}, \bu_i^{t+1} - \bu^{t+1}  \rangle  -   \langle\nabla_{\bu_i}\alpha_i f_i(\bv_i^{t+1}, \bu_i^{t+1}), \bu_i^{t+1} - \bu^{t+1}  \rangle+ \frac{\rho}{2}\|\bu_i^{t+1} - \bu^{t+1}\|^2 \Bigr) \\
    \leq &\sum_{i=1}^m\Bigl( \langle e_i^{t+1}, \bu_i^{t+1} - \bu^{t+1}  \rangle  -   \langle\nabla_{\bu_i}\alpha_i  f_i(\bv_i^{t+1}, \bu_i^{t+1}), \bu_i^{t+1} - \bu^{t+1}  \rangle+ \rho (\|\bu_i^{t+1} - \bu^{t}\|^2  + \|\bu^{t+1} - \bu^{t}\|^2 )\Bigr).\label{eq: 110 follows from jensen}
\end{align}
where \eqref{eq: 110 follows from jensen} follows from Lemma \ref{jensen}. Using the above condition, we can obtain that 
\begin{align}
    |\ml(\mathcal{P}^{t+1})-f(\bV^{t+1},\bU^{t+1})|\leq \sum_{i=1}^m\Bigl( |\langle e_i^{t+1}, \bu_i^{t+1} - \bu^{t+1}  \rangle|  +|  \langle\alpha_i\nabla_{\bu_i}f_i(\bv_i^{t+1}, \bu_i^{t+1}), \bu_i^{t+1} - \bu^{t+1}  \rangle|+\rho (\|\bu_i^{t+1} - \bu^{t}\|^2  + \|\bu^{t+1} - \bu^{t}\|^2 )\Bigr)\to 0, \label{eq: 111 follows from many zero limit}
\end{align}
where \eqref{eq: 111 follows from many zero limit} follows from \eqref{eq: 71 in lemma 9}. Next, we derive
\begin{align}
    \ml(\mathcal{P}^{t+1})-f(\bV^{t+1},\bu^{t+1})
    =&\sum_{i=1}^m\Bigl(\alpha_i f_i(\bv_i^{t+1},\bu_i^{t+1}) -\alpha_i f_i(\bv_i^{t+1},\bu^{t+1})  +\langle\bpi_i^{t+1},\bu_i^{t+1} - \bu^{t+1}   \rangle + \frac{\rho}{2}\|\bu_i^{t+1} - \bu^{t+1} \|^2 \Bigr)\\
    \leq &\sum_{i=1}^m\Bigl( \langle \alpha_i\nabla_{\bu_i} f_i(\bv_i^{t+1} ,\bu_i^{t+1})+ \bpi_i^{t+1},  \bu_i^{t+1} - \bu^{t+1}\rangle  + \frac{\rho+L_u}{2}\|\bu_i^{t+1} - \bu^{t+1} \|^2 \Bigr),\label{eq: L and f to 0}
\end{align}
where \eqref{eq: L and f to 0} follows from Lemma \ref{lemma: lipschitz}. Using the above condition, we can obtain that 
\begin{align}
    |\ml(\mathcal{P}^{t+1})-f(\bV^{t+1},\bu^{t+1})|\leq & \sum_{i=1}^m\Bigl(\langle \alpha_i \nabla_{\bu_i} f_i(\bv_i^{t+1} ,\bu_i^{t+1})+ \bpi_i^{t+1},  \bu_i^{t+1} - \bu^{t+1}\rangle  + \frac{\rho+L_u}{2}\|\bu_i^{t+1} - \bu^{t+1} \|^2 \Bigr)\\
    \leq &\sum_{i=1}^m\Bigl(2 \xi_i^{t+1} + 2 \| \bu_i^{t+1}  -  \bu^{t} + \bu^{t}- \bu^{t+1}\|^2 + \frac{\rho+L_u}{2}\|\bu_i^{t+1}  -  \bu^{t} + \bu^{t}- \bu^{t+1} \|^2 \Bigr)\\
     = &\sum_{i=1}^m\Bigl(2 \xi_i^{t+1} +4 \| \bu_i^{t+1}  -  \bu^{t} \|^2+ \|\bu^{t}- \bu^{t+1}\|^2 + (\rho+L_u)\|\bu_i^{t+1}  -  \bu^{t}\|^2 +\| \bu^{t}- \bu^{t+1} \|^2 \Bigr)\to 0,
\end{align}
which completes the proof.
\end{proof}

\subsection{Proof of Theorem \ref{theorem: sequences convergence}}\label{appendix: proof of theorem 2theorem: sequences convergence}
\begin{theorem}[Restatement of Theorem \ref{theorem: sequences convergence}]\label{theorem: restatement of theorem 2}
    Suppose that Assumptions \ref{assmption lipschitz} and \ref{assumption: coercive} hold, let each client set $\max\{3\alpha_i L_u, 3\alpha_i L_{uv}, 2L_u \}\leq \rho \leq \frac{15}{8}\sigma_i$ and $\sigma_i\geq \alpha_i L_v$, then the following results hold.
    \begin{itemize}
        \item[a)] The accumulating point $\mathcal{P}^{\infty}$ of sequences $\{\mathcal{P}^t\}$ is a stationary point of Problem \eqref{eq: problem pfl2}, and $(\bV^{\infty},\bu^{\infty})$ is a stationary point of Problem \eqref{eq: pfl problem1}.
        \item[b)] Under Assumption \ref{assumption: either real analytic or semialgebraic}, the sequence $\{\mathcal{P}^t\}$ converges to $\mathcal{P}^{\infty}$.
    \end{itemize}
\end{theorem}
\begin{proof}
a) Let $(\bV^{\infty}, \bU^{\infty}, \bPi^{\infty}, \bu^{\infty})$ be any accumulating point of the sequence $\{(\bV^t,\bU^t,\bPi^t,\bu^t)\}$. Then we have 
\begin{align}
    g_v(\bv_i^{\infty},\bu_i^{\infty})=\alpha_i \nabla_{\bv_i} f_i(\bv_i^{\infty}, \bu_i^{\infty}) =& 0 ,\label{eq: stationary point v}\\
    g_u(\bv_i^{\infty},\bu_i^{\infty}) = \alpha_i \nabla_{\bu_i} f_i(\bv_i^{\infty}, \bu_i^{\infty}) +  \bpi_i^{\infty} +  \rho (\bu_i^{\infty} -\bu^{\infty} ) =& 0, \label{eq: stationary point ui}\\
    \bu_i^{\infty} - \bu^{\infty}  = & 0, \label{eq: stationary point u}\\
    \sum_{i=1}^{m} \bpi_i^{\infty} = \rho \sum_{i=1}^m(\bu^{\infty}-\bu_i^{\infty}) =& 0,\label{eq: stationary point pi}
\end{align}
where \eqref{eq: stationary point v} and \eqref{eq: stationary point ui} follow from \eqref{eq: gv} and \eqref{eq: gu}, \eqref{eq: stationary point u} and \eqref{eq: stationary point pi} follow from Lemma \ref{lemma: non increasing results} and \eqref{eq: solve u}. 
Combining \eqref{eq: stationary point v}, \eqref{eq: stationary point ui}, \eqref{eq: stationary point u} and \eqref{eq: stationary point pi} yields the stationary point condition in \eqref{eq: optimal condition of pfl problem2}.
Moreover, $(\bV^{\infty}, \bu^{\infty})$ is a stationary point of \eqref{eq: pfl problem1}, which completes the proof.

b) Following from Proposition \ref{proposition: kl property of L}, it holds that $\tilde{\ml}$ is a \kl function.
According to Definition \ref{definition: kl property}, we have
\begin{align}
    \phi'\Bigl(\tilde{\ml}(\mathcal{P}^t)-\tilde{\ml}(\mathcal{P}^{\infty})\Bigr)\text{dist}\Bigl(0,\,\partial\tilde{\ml}(\mathcal{P}^t)\Bigr)\geq 1,
\end{align}
which implies 
\begin{align}
    \phi'\Bigl(\tilde{\ml}(\mathcal{P}^t)-\tilde{\ml}(\mathcal{P}^{\infty})\Bigr)\geq\frac{1}{\text{dist}\Bigl(0,\,\partial\tilde{\ml}(\mathcal{P}^t)\Bigr)}\geq \frac{1}{\sqrt{b (\Sigma_p^{t+1}+\tilde{\Xi}^{t+1})}}.\label{eq: 122 relative error form 2}
\end{align}

Since $\phi$ is a concave function, we have
\begin{align}
    \phi\Bigl(\tilde{\ml}(\mathcal{P}^{t+1})-\tilde{\ml}(\mathcal{P}^{\infty})\Bigr)-\phi\Bigl(\tilde{\ml}(\mathcal{P}^t)-\tilde{\ml}(\mathcal{P}^{\infty})\Bigr)&\leq\phi'\Bigl(\tilde{\ml}(\mathcal{P}^t)-\tilde{\ml}(\mathcal{P}^{\infty})\Bigr)\Bigl(\tilde{\ml}(\mathcal{P}^{t+1})-\tilde{\ml}(\mathcal{P}^t)\Bigr) \\
    &\leq-\frac{a \Sigma_p^{t+1}}{\sqrt{b (\Sigma_p^{t+1}+\tilde{\Xi}^{t+1})}}, \label{eq: 123 follows from lemma 1 and 122}
\end{align}
where \eqref{eq: 123 follows from lemma 1 and 122} follows from Lemma \ref{lemma: sufficient descent} and \eqref{eq: 122 relative error form 2}. Under the above condition, we obtain
\begin{align}
    \sqrt{\Sigma_p^{t+1}}&\leq \sqrt{\frac{\sqrt{b (\Sigma_p^{t+1}+\tilde{\Xi}^{t+1})}}{a}\Biggl(\phi\Bigl(\tilde{\ml}(\mathcal{P}^t)-\tilde{\ml}(\mathcal{P}^{\infty})\Bigr)-\phi\Bigl(\tilde{\ml}(\mathcal{P}^{t+1})-\tilde{\ml}(\mathcal{P}^{\infty})\Bigr)\Biggr)} \\
    &\leq\frac{1}{2}\frac{\sqrt{b (\Sigma_p^{t+1}+\tilde{\Xi}^{t+1})}}{a}+\frac{1}{2}\Biggl(\phi\Bigl(\tilde{\ml}(\mathcal{P}^t)-\tilde{\ml}(\mathcal{P}^{\infty})\Bigr)-\phi\Bigl(\tilde{\ml}(\mathcal{P}^{t+1})-\tilde{\ml}(\mathcal{P}^{\infty})\Bigr)\Biggr)\\
    &\leq\frac{1}{2}\frac{\sqrt{b \Sigma_p^{t+1}}}{a }+\frac{1}{2}\frac{\sqrt{b \tilde{\Xi}^{t+1}}}{a}+\frac{1}{2}\Biggl(\phi\Bigl(\tilde{\ml}(\mathcal{P}^t)-\tilde{\ml}(\mathcal{P}^{\infty})\Bigr)-\phi\Bigl(\tilde{\ml}(\mathcal{P}^{t+1})-\tilde{\ml}(\mathcal{P}^{\infty})\Bigr)\Biggr)\label{eq: 79}.
\end{align}
Summing over \eqref{eq: 79} from 0 to $+\infty$ yields
\begin{align}
    \sum_{t=0}^{+\infty}\sqrt{\Sigma_p^{t+1}}\leq\sum_{t=0}^{+\infty}\frac{1}{2}\frac{\sqrt{b \Sigma_p^{t+1}}}{a}+\frac{1}{2}\frac{\sqrt{b \tilde{\Xi}^{t+1}}}{a}+ \frac{1}{2}\phi\Bigl(\tilde{\ml}(\mathcal{P}^0)-\tilde{\ml}(\mathcal{P}^{\infty}\Bigr),
\end{align}
which implies 
\begin{align}
    \sum_{t=0}^{+\infty}\sqrt{\Sigma_p^{t+1}}\leq \frac{a}{2a-\sqrt{b}}\phi\Bigl(\tilde{\ml}(\mathcal{P}^0)-\tilde{\ml}(\mathcal{P}^{\infty}\Bigr)+\sum_{t=1}^{+\infty}\frac{\sqrt{b\tilde{\Xi}^{t+1}}}{2a-\sqrt{b}}\leq \infty.
\end{align}
Following from the definition of $\Sigma_p^{t+1}$, we have
\begin{align}\label{eq: 80}
    \Sigma_p^{t+1}=\sum_{i=1}^m(\|\bu^{t+1}-\bu^{t}\|^2+\|\bu^{t+1}_i-\bu_i^{t}\|^2+\|\bv_i^{t+1}-\bv_i^t\|^2)<\infty.
\end{align}
Then we can infer from \eqref{eq: 80} that the sequence $\{\bV^t,\bU^t,\bu^t\}$ is convergent. Next, following from the fact that $\|\bpi_i^{t+1}-\bpi_i^t\| = \rho \|\bu_i^{t+1} - \bu^t \| $, we can infer that $\{\Pi^t\}$ is convergent. Overall, the sequence $\{\mathcal{P}^t\}$ is convergent, which complete the proof.
\end{proof}

\subsection{Proof of Theorem \ref{theorem: convergence rate based on kl property}}
\begin{theorem}[Restatement of Theorem \ref{theorem: convergence rate based on kl property}]\label{theorem: restatement of theorem 3}
    Let $\{\mathcal{P}^t\}$ be the sequence generated by Algorithm \ref{alg: fedapm}, and $\mathcal{P}^{\infty}$ be its limit, let $\phi(x)=\frac{\sqrt{c}}{1-\theta}x^{1-\theta}$ be a desingularizing function, where $c>0$ and $\theta\in[0,1)$, then under Assumptions \ref{assumption: either real analytic or semialgebraic} and \ref{assmption lipschitz}, let each client $i$ set $\max\{3\alpha_i L_u, 3\alpha_i L_{uv}, 2L_u \}\leq \rho \leq \frac{15}{8}\sigma_i$ and $\sigma_i\geq \alpha_i L_v$, the following results hold.
\begin{itemize}
    \item[a)] If $\theta=0$, then there exists a $t_1$ such that the sequence $\{\tilde{\ml}(\mathcal{P}^t)\}$, $t\geq t_1$ converges in a finite number of iterations.
    \item[b)] If $\theta\in(0,1/2]$, then there exists a $t_2$ such that for any $t\geq t_2$,
\begin{align*}
    \tilde{\ml}(\mathcal{P}^{t+1})-\tilde{\ml}(\mathcal{P}^{\infty})\leq\Bigl(\frac{bc}{a+bc}\Bigr)^{t-t_2+1}\Bigl(\tilde{\ml}(\mathcal{P}^{t_2})-f^*\Bigr) + (a+bc) \Xi^{t_2-1} .
\end{align*}
   \item[c)] If $\theta \in(1/2,1)$, then there exists a $t_3$ such that for any $t
    \geq t_3$, 
\begin{align*}
    \tilde{\ml}(\mathcal{P}^{t+1})-\tilde{\ml}(\mathcal{P}^{\infty})\leq\Bigl(\frac{bc}{(2\theta-1)\kappa a(t-t_3)}\Bigr)^{\frac{1}{2\theta-1}},
\end{align*}
where $\mu>0$ is a constant.
\end{itemize}
\end{theorem}

\begin{proof}
Given the desingularizing function $\phi(x)=\frac{\sqrt{c}}{1-\theta}x^{1-\theta}$, we have
\begin{align}
    1&\leq\phi'\Bigl(\tilde{\ml}(\mathcal{P}^{t+1})-\tilde{\ml}(\mathcal{P}^{\infty})\Bigr)^2 \text{dist}\Bigl(0,\,\partial\tilde{\ml}(\mathcal{P}^t)\Bigr)^2\\
    &\leq c \Bigl(\tilde{\ml}(\mathcal{P}^{t+1})-\tilde{\ml}(\mathcal{P}^{\infty})\Bigr)^{-2\theta} b (\Sigma_p^{t+1}+\tilde{\Xi}^{t+1})\label{eq:103 follows from lemma 1}\\
    &\leq  c \Bigl(\tilde{\ml}(\mathcal{P}^{t+1})-\tilde{\ml}(\mathcal{P}^{\infty})\Bigr)^{-2\theta}b \Bigl(\frac{\tilde{\ml}(\mathcal{P}^t)-\tilde{\ml}(\mathcal{P}^{t+1})}{a}+\tilde{\Xi}^{t+1}\Bigr)\label{eq: 104 follows from lemma 2},
\end{align}
where \eqref{eq:103 follows from lemma 1} follows from Lemma \ref{lemma: sufficient descent} and \eqref{eq: 104 follows from lemma 2} follows from Lemma \ref{lemma: relative error}. Under the above conditions, we can obtain
\begin{align}\label{eq: kl final}
    \frac{a}{b c} \Bigl(\tilde{\ml}(\mathcal{P}^{t+1})-\tilde{\ml}(\mathcal{P}^{\infty})\Bigr)^{2\theta}\leq\Bigl(\tilde{\ml}(\mathcal{P}^t)-\tilde{\ml}(\mathcal{P}^{\infty})\Bigr)-\Bigl(\tilde{\ml}(\mathcal{P}^{t+1})-\tilde{\ml}(\mathcal{P}^{\infty})\Bigr)+a \tilde{\Xi}^{t+1}.
\end{align}
Next, we consider the three cases with respect to $\theta$.
\begin{itemize}
    \item If $\theta=0$, then according to \eqref{eq: kl final}, it holds that,
\begin{align}\label{eq:conflict}
    \Bigl(\tilde{\ml}(\mathcal{P}^t)-\tilde{\ml}(\mathcal{P}^{\infty})\Bigr)-\Bigl(\tilde{\ml}(\mathcal{P}^{t+1})-\tilde{\ml}(\mathcal{P}^{\infty})\Bigr)+a \tilde{\Xi}^{t+1}\geq \frac{a}{bc}.
\end{align}

However, $\tilde{\ml}(\mathcal{P}^t)-\tilde{\ml}(\mathcal{P}^{\infty})\to 0$ and $\tilde{\Xi}^{t+1}\to 0$, which are conflict with \eqref{eq:conflict}. Therefore, we can deduce that there must exist $t\geq t_1>0$ such that $\tilde{\ml}(\mathcal{P}^t)-\tilde{\ml}(\mathcal{P}^{\infty})= 0$.
    \item If $\theta\in (0,1/2]$, there must exist $t\geq t_2>0$, such that $0\leq\tilde{\ml}(\mathcal{P}^{t+1})-\tilde{\ml}(\mathcal{P}^{\infty})\leq 1$. Then according to \eqref{eq: kl final}, we have
\begin{align}
    \Bigl(\tilde{\ml}(\mathcal{P}^t)-\tilde{\ml}(\mathcal{P}^{\infty})\Bigr)-\Bigl(\tilde{\ml}(\mathcal{P}^{t+1})-\tilde{\ml}(\mathcal{P}^{\infty})\Bigr)+a \tilde{\Xi}^{t+1}\geq\frac{a}{bc} \Bigl(\tilde{\ml}(\mathcal{P}^{t+1})-\tilde{\ml}(\mathcal{P}^{\infty})\Bigr)^{2\theta}\geq \frac{a}{bc} \Bigl(\tilde{\ml}(\mathcal{P}^{t+1})-\tilde{\ml}(\mathcal{P}^{\infty})\Bigr),
\end{align}
which implies 
\begin{align}
    \tilde{\ml}(\mathcal{P}^{t+1})-\tilde{\ml}(\mathcal{P}^{\infty})&\leq\frac{bc}{a+bc}\Bigl(\tilde{\ml}(\mathcal{P}^t)-\tilde{\ml}(\mathcal{P}^{\infty})+a \tilde{\Xi}^{t+1}\Bigr)\\
    &\leq\frac{bc}{a+bc}\Bigl(\tilde{\ml}(\mathcal{P}^t)-\tilde{\ml}(\mathcal{P}^{\infty})\Bigr) + \frac{bc}{a+bc}a \Xi^{t}   \label{eq: 135 follows from fact xi geq tilde xi}\\
    &\leq\Bigl(\frac{bc}{a+bc}\Bigr)^{t-t_2+1}\Bigl(\tilde{\ml}(\mathcal{P}^{t_2})-\tilde{\ml}(\mathcal{P}^{\infty})\Bigr)  +\Bigl(\frac{bc}{a+bc}\Bigr)^{t-t_2+1} a \Xi^{t_2-1} +\cdots+ \Bigl(\frac{bc}{a+bc}\Bigr)^2 a \Xi^{t-1}+ \frac{bc}{a+bc}a \Xi^{t}\\
    &\leq\Bigl(\frac{bc}{a+bc}\Bigr)^{t-t_2+1}\Bigl(\tilde{\ml}(\mathcal{P}^{t_2})-\tilde{\ml}(\mathcal{P}^{\infty})\Bigr)  +a \Xi^{t_2-1}\Biggl( \Bigl(\frac{bc}{a+bc}\Bigr)^{t-t_2+1} +\cdots+\Bigl(\frac{bc}{a+bc}\Bigr)^2+\frac{bc}{a+bc} \Biggr) \\
    &\leq\Bigl(\frac{bc}{a+bc}\Bigr)^{t-t_2+1}\Bigl(\tilde{\ml}(\mathcal{P}^{t_2})-f^*-\frac{28}{\rho}\Xi^{t_2} \Bigr) + (a+bc) \Biggl( 1-\Bigl(\frac{bc}{a+bc}\Bigr)^{t-t_2+1} \Biggr)\Xi^{t_2-1} \label{eq: 139 follows from lemma 14}\\
    &\leq\Bigl(\frac{bc}{a+bc}\Bigr)^{t-t_2+1}\Bigl(\tilde{\ml}(\mathcal{P}^{t_2})-f^*\Bigr) + (a+bc) \Xi^{t_2-1} 
\end{align}
where \eqref{eq: 135 follows from fact xi geq tilde xi} follows from the fact that $\Xi^t\geq \tilde{\Xi}^{t+1}$ and \eqref{eq: 139 follows from lemma 14} follows from Lemma \ref{lemma: non increasing results}.

\item If $\theta\in(1/2,1)$, we define a function $\phi(z):=\frac{bc}{a(1-2\theta)}z^{1-2\theta}$. Let $1>\frac{1}{R}>\min_i{\mu_i}>0$ be a constant, we consider two cases, if $\Bigl(\tilde{\ml}(\mathcal{P}^t)-\tilde{\ml}(\mathcal{P}^{\infty})\Bigr)^{-2\theta}\geq\Bigl(\tilde{\ml}(\mathcal{P}^{t+1})-\tilde{\ml}(\mathcal{P}^{\infty})\Bigr)^{-2\theta}/R$, we obtain
\begin{align}\label{eq: geq 0}
    &\phi\Bigl(\tilde{\ml}(\mathcal{P}^t)-\tilde{\ml}(\mathcal{P}^{\infty})\Bigr)-\phi\Bigl(\tilde{\ml}(\mathcal{P}^{t+1})-\tilde{\ml}(\mathcal{P}^{\infty})\Bigr)\\
    =&\int_{\tilde{\ml}(\mathcal{P}^{t+1})-\tilde{\ml}(\mathcal{P}^{\infty})}^{\tilde{\ml}(\mathcal{P}^t)-\tilde{\ml}(\mathcal{P}^{\infty})}\phi'(z)dz=\int_{\tilde{\ml}(\mathcal{P}^{t+1})-\tilde{\ml}(\mathcal{P}^{\infty})}^{\tilde{\ml}(\mathcal{P}^t)-\tilde{\ml}(\mathcal{P}^{\infty})}\frac{bc}{a}z^{-2\theta}dz\\
    \geq&\frac{bc}{a}\Bigl(\tilde{\ml}(\mathcal{P}^t)-\tilde{\ml}(\mathcal{P}^{t+1})\Bigr)\Bigl(\tilde{\ml}(\mathcal{P}^t)-\tilde{\ml}(\mathcal{P}^{\infty})\Bigr)^{-2\theta}\\
    \geq&\frac{bc}{Ra}\Bigl(\tilde{\ml}(\mathcal{P}^t)-\tilde{\ml}(\mathcal{P}^{t+1})\Bigr)\Bigl(\tilde{\ml}(\mathcal{P}^{t+1})-\tilde{\ml}(\mathcal{P}^{\infty})\Bigr)^{-2\theta}\\
    =&\frac{bc}{Ra}\Bigl(\tilde{\ml}(\mathcal{P}^t)-\tilde{\ml}(\mathcal{P}^{t+1})+a \tilde{\Xi}^{t+1}\Bigr)\Bigl(\tilde{\ml}(\mathcal{P}^{t+1})-\tilde{\ml}(\mathcal{P}^{\infty})\Bigr)^{-2\theta}-\frac{bc}{Ra}a \tilde{\Xi}^{t+1}\Bigl(\tilde{\ml}(\mathcal{P}^{t+1})-\tilde{\ml}(\mathcal{P}^{\infty})\Bigr)^{-2\theta}\\
    \geq&\frac{1}{R}-\frac{bc}{R}\frac{\tilde{\Xi}^{t+1}}{\bigl(\tilde{\ml}(\mathcal{P}^{t+1})-\tilde{\ml}(\mathcal{P}^{\infty})\bigr)^{2\theta}}\label{eq: 145 follows from 132}\\
    \geq&\frac{1}{R}-\frac{bc}{R}\frac{\Xi^{t}}{\bigl(\tilde{\ml}(\mathcal{P}^{t+1})-\tilde{\ml}(\mathcal{P}^{\infty})\bigr)^{2\theta}},
\end{align}
where \eqref{eq: 145 follows from 132} follows from \eqref{eq: kl final}. We proceed by considering three cases for the analysis of $\tilde{\Xi}^{t+1}$: whether it is a lower-order infinitesimal, a higher-order infinitesimal, or an infinitesimal of the same order relative to $\Bigl(\tilde{\ml}(\mathcal{P}^{t+1})-\tilde{\ml}(\mathcal{P}^{\infty})\Bigr)^{2\theta}$.
\begin{itemize}
    \item[i)]
If $\Xi^{t}$ is a lower-order infinitesimal of $\Bigl(\tilde{\ml}(\mathcal{P}^{t+1})-\tilde{\ml}(\mathcal{P}^{\infty})\Bigr)^{2\theta}$, then we can obtain
\begin{align}\label{eq: 116 lower-order infinitesimal}
    \infty = \lim_{t\to +\infty}\frac{\Xi^{t}}{\bigl(\tilde{\ml}(\mathcal{P}^{t+1})-\tilde{\ml}(\mathcal{P}^{\infty})\bigr)^{2\theta}}= \lim_{t\to +\infty}\frac{\sum_{i=1}^m\xi_i^{t}}{\bigl(\tilde{\ml}(\mathcal{P}^{t+1})-\tilde{\ml}(\mathcal{P}^{\infty})\bigr)^{2\theta}} \leq \lim_{t\to +\infty}\frac{\sum_{i=1}^m \mu_i^t\xi_i^{0}}{\bigl(\tilde{\ml}(\mathcal{P}^{t+1})-\tilde{\ml}(\mathcal{P}^{\infty})\bigr)^{2\theta}}.
\end{align}
Based on \eqref{eq: 116 lower-order infinitesimal}, we can derive the convergence rate of $\bigl(\tilde{\ml}(\mathcal{P}^{t+1})-\tilde{\ml}(\mathcal{P}^{\infty})\bigr)^{2\theta}\to 0$ is faster than $\sum_{i=1}^m\mu_i^t\xi^{0}\to 0$, which implies 
\begin{align}\label{eq:117 if confilit with our assumption}
    \bigl(\tilde{\ml}(\mathcal{P}^{t+1})-\tilde{\ml}(\mathcal{P}^{\infty})\bigr)^{2\theta}<\min_i \mu_i \times \bigl(\tilde{\ml}(\mathcal{P}^{t})-\tilde{\ml}(\mathcal{P}^{\infty})\bigr)^{2\theta}<\frac{1}{R} \bigl(\tilde{\ml}(\mathcal{P}^{t})-\tilde{\ml}(\mathcal{P}^{\infty})\bigr)^{2\theta},
\end{align}
Note that \eqref{eq:117 if confilit with our assumption} is conflict with $\Bigl(\tilde{\ml}(\mathcal{P}^t)-\tilde{\ml}(\mathcal{P}^{\infty})\Bigr)^{-2\theta}\geq\Bigl(\tilde{\ml}(\mathcal{P}^{t+1})-\tilde{\ml}(\mathcal{P}^{\infty})\Bigr)^{-2\theta}/R$, then $\Xi^{t}$ cannot be a lower-order infinitesimal of $\Bigl(\tilde{\ml}(\mathcal{P}^{t+1})-\tilde{\ml}(\mathcal{P}^{\infty})\Bigr)^{2\theta}$.
\item[ii)]
If $\Xi^{t}$ is a higher-order infinitesimal of $\Bigl(\tilde{\ml}(\mathcal{P}^{t+1})-\tilde{\ml}(\mathcal{P}^{\infty})\Bigr)^{2\theta}$, then we obtain
\begin{align}
    \lim_{t\to +\infty}\frac{\Xi^{t}}{\bigl(\tilde{\ml}(\mathcal{P}^{t+1})-\tilde{\ml}(\mathcal{P}^{\infty})\bigr)^{2\theta}}=0.
\end{align}

If $\Xi^{t}$ and $\Bigl(\tilde{\ml}(\mathcal{P}^{t+1})-\tilde{\ml}(\mathcal{P}^{\infty})\Bigr)^{2\theta}$ are same-order infinitesimals, then we obtain
\begin{align}
    \lim_{t\to +\infty}\frac{\Xi^{t}}{\bigl(\tilde{\ml}(\mathcal{P}^{t+1})-\tilde{\ml}(\mathcal{P}^{\infty})\bigr)^{2\theta}}=\alpha\neq0.
\end{align}

Then there exists $Q>0$ such that $\frac{\Xi^{t}}{\bigl(\tilde{\ml}(\mathcal{P}^{t+1})-\tilde{\ml}(\mathcal{P}^{\infty})\bigr)^{2\theta}}\leq Q$. Therefore, we obtain that
\begin{align}\label{eq: 152 to 158}
    \phi\Bigl(\tilde{\ml}(\mathcal{P}^t)-\tilde{\ml}(\mathcal{P}^{\infty})\Bigr)-\phi\Bigl(\tilde{\ml}(\mathcal{P}^{t+1})-\tilde{\ml}(\mathcal{P}^{\infty})\Bigr)\geq\frac{1}{R}-\frac{bcQ}{R}.
\end{align}
\end{itemize}


Consider another case, if $\Bigl(\tilde{\ml}(\mathcal{P}^t)-\tilde{\ml}(\mathcal{P}^{\infty})\Bigr)^{-2\theta}<\Bigl(\tilde{\ml}(\mathcal{P}^{t+1})-\tilde{\ml}(\mathcal{P}^{\infty})\Bigr)^{-2\theta}/R$, then we can infer
\begin{align}
    R^{\frac{1}{2\theta}}\Bigl(\tilde{\ml}(\mathcal{P}^{t+1})-\tilde{\ml}(\mathcal{P}^{\infty})\Bigr)< \tilde{\ml}(\mathcal{P}^{t})-\tilde{\ml}(\mathcal{P}^{\infty}),
\end{align}
which implies
\begin{align}
    \Bigl(\tilde{\ml}(\mathcal{P}^{t+1})-\tilde{\ml}(\mathcal{P}^{\infty})\Bigr)^{1-2\theta}&>\overline{R} \Bigl(\tilde{\ml}(\mathcal{P}^{t})-\tilde{\ml}(\mathcal{P}^{\infty})\Bigr)^{1-2\theta},\\
    \Bigl(\tilde{\ml}(\mathcal{P}^{t+1})-\tilde{\ml}(\mathcal{P}^{\infty})\Bigr)^{1-2\theta}-\Bigl(\tilde{\ml}(\mathcal{P}^{t})-\tilde{\ml}(\mathcal{P}^{\infty})\Bigr)^{1-2\theta}&>(\overline{R}-1)\Bigl(\tilde{\ml}(\mathcal{P}^{t})-\tilde{\ml}(\mathcal{P}^{\infty})\Bigr)^{1-2\theta},
\end{align}
where $\overline{R}=R^{\frac{2\theta-1}{\theta}}> 1$. Since $\overline{R}-1>0$ and $\tilde{\ml}(\mathcal{P}^{t})-\tilde{\ml}(\mathcal{P}^{\infty})\to 0^{+}$, then there exists $\overline{\kappa}>0$ such that $(\overline{R}-1)\Bigl(\tilde{\ml}(\mathcal{P}^{t})-\tilde{\ml}(\mathcal{P}^{\infty})\Bigr)^{1-2\theta}>\overline{\kappa}$ for all $t\geq t_3$. 
Therefore, we obtain that
\begin{align}
    \Bigl(\tilde{\ml}(\mathcal{P}^{t+1})-\tilde{\ml}(\mathcal{P}^{\infty})\Bigr)^{1-2\theta}-\Bigl(\tilde{\ml}(\mathcal{P}^{t})-\tilde{\ml}(\mathcal{P}^{\infty})\Bigr)^{1-2\theta}\geq\overline{\kappa}>0
\end{align}
for all $t\geq t_3$. Then we bound $\phi\Bigl(\tilde{\ml}(\mathcal{P}^t)-\ml(\mathcal{P}^{\infty})\Bigr)-\phi\Bigl(\tilde{\ml}(\mathcal{P}^{t+1})-\tilde{\ml}(\mathcal{P}^{\infty})\Bigr)$ as 
\begin{align}
    \phi\Bigl(\tilde{\ml}(\mathcal{P}^t)-\tilde{\ml}(\mathcal{P}^{\infty})\Bigr)-\phi\Bigl(\tilde{\ml}(\mathcal{P}^{t+1})-\tilde{\ml}(\mathcal{P}^{\infty})\Bigr)
    =&\frac{bc}{(1-2\theta)a}\Biggl(\Bigl(\tilde{\ml}(\mathcal{P}^t)-\tilde{\ml}(\mathcal{P}^{\infty})\Bigr)^{1-2\theta}-\Bigl(\tilde{\ml}(\mathcal{P}^{t+1})-\tilde{\ml}(\mathcal{P}^{\infty})\Bigr)^{1-2\theta}\Biggr) \\
    \geq&\frac{\overline{\kappa}bc}{(2\theta-1)a}.\label{eq: geq 1}
\end{align}

If we define $\kappa:=\min\{\frac{1}{R}-\frac{bcQ}{R},\frac{\overline{\kappa}bc}{(2\theta-1)a}\}>0$, one can combine \eqref{eq: 152 to 158} and \eqref{eq: geq 1} to obtain that
\begin{align}\label{eq: to sum inequality}
    \phi\Bigl(\tilde{\ml}(\mathcal{P}^t)-\tilde{\ml}(\mathcal{P}^{\infty})\Bigr)-\phi\Bigl(\tilde{\ml}(\mathcal{P}^{t+1})-\tilde{\ml}(\mathcal{P}^{\infty})\Bigr)\geq\kappa
\end{align}
for all $t\geq t_3$. By summing \eqref{eq: to sum inequality} from $t_3$ to some $t$ greater than $t_3$, we obtain
\begin{align}
    \phi\Bigl(\tilde{\ml}(\mathcal{P}^{t_3})-\tilde{\ml}(\mathcal{P}^{\infty})\Bigr)-\phi\Bigl(\tilde{\ml}(\mathcal{P}^{t+1})-\tilde{\ml}(\mathcal{P}^{\infty})\Bigr)\geq (t-t_3)\kappa,
\end{align}
which implies 
\begin{align}\label{eq: 160 substitute}
    \phi\Bigl(\tilde{\ml}(\mathcal{P}^{t+1})-\tilde{\ml}(\mathcal{P}^{\infty})\Bigr)\leq -(t-t_3)\kappa.
\end{align}
Substituting $\phi(z):=\frac{bc}{a(1-2\theta)}z^{1-2\theta}$ into \eqref{eq: 160 substitute} yields
\begin{align}
    \tilde{\ml}(\mathcal{P}^{t+1})-\tilde{\ml}(\mathcal{P}^{\infty})\leq\Bigl(\frac{(2\theta-1)\kappa a}{bc}(t-t_3)\Bigr)^{\frac{1}{1-2\theta}},
\end{align}
which completes the proof.
\end{itemize}
\end{proof}

\newpage

\section{Experimental Details and Extra Results}\label{appendix: Experimental Details}
\subsection{Full Details on the Datasets}\label{sec: dataset details}
We introduce three commonly used multimodal datasets for training in FL: \crema, KU-HAR, and \crisis, along with \cifar, a well-known image classification dataset. 
Below is a description of these datasets:
\begin{itemize}
    \item \cifar is a widely used image classification dataset, consisting of 60,000 32×32 color images categorized into 10 different classes, with 6,000 images per class. 
While it is not multimodal, \cifar serves as a fundamental benchmark for image classification tasks and is commonly used in computer vision research.
    \item \crisis is a dataset that focuses on analyzing the role of social media during disasters and emergencies, where platforms like Twitter are crucial for disseminating information about the disaster's impact, such as property damage, injuries, and fatalities, as well as urgent needs for help.
\crisis includes 18,100 tweets, paired with visual and textual information, related to seven major natural disasters, such as the 2017 California Wildfires. 
This dataset is used to assess the impact of disasters, like utility damage and casualties, and integrates text and image modalities for multimodal learning.
    \item \kuhar is a dataset focused on human activity recognition (HAR), using wearable sensor data such as accelerometers and gyroscopes. 
Due to the rising popularity of wearable technology, HAR has gained significant interest in FL. 
The dataset contains data from 90 participants (75 male, 15 female) across 18 different activities. 
In our experiments,  six activities are selected, along with the addition of jumping and running.
The accelerometer and gyroscope data in \kuhar are treated as two distinct modalities, making it a multimodal dataset.
    \item \crema is primarily designed for emotion recognition (ER) tasks, which have broad applications in virtual assistants, human behavior analysis, and AI-enhanced education. 
The dataset consists of 7,442 audio-visual clips recorded by 91 actors. Each actor is instructed to recite 12 sentences while expressing six distinct emotions.
\end{itemize}




\subsubsection{Data Partitioning} In processing different modalities, the data is first partitioned based on specific characteristics. 
The first partitioning method is through unique client identifiers. 
For example, datasets related to audio, such as \crema, contain speech or speech-visual data organized by speaker IDs. 
Therefore, it is logical to partition client data in FL by speaker ID. Similarly, we apply the same approach to the \kuhar dataset. 
On the other hand, for datasets like \cifar and \crisis, partitioning is conducted using a Dirichlet distribution.

\subsubsection{Feature Processing} The following outlines how we handle different data modalities. 
In our experiments, we utilize pre-trained models as backbone networks to extract features for downstream model training \cite{Feng2023FedMultimodal}. 
The following is a detailed introduction for each modality:
\begin{itemize}
\item Visual modalities: We use MobileNetV2 \cite{howard2017mobilenets} for feature extraction. 
This lightweight model, with 4.3M parameters, is designed for resource-constrained devices, offering efficient computation without sacrificing accuracy. 
Its low computational cost and mobile-friendly architecture make it ideal for federated learning, ensuring fast and effective visual data processing.

\item Text modalities: For text data, we utilize MobileBERT \cite{sun2020mobilebert} as the backbone for feature extraction. 
MobileBERT significantly reduces the original BERT \cite{devlin2018bert} model's size from 340M to 25M by employing a bottleneck architecture, making it much more efficient for environments with limited computational resources.

\item Audio modalities: We use Mel-frequency cepstral coefficients (MFCCs) \cite{chen2023exploring} for audio data, a widely adopted feature extraction method in speech recognition. 
MFCCs capture critical audio characteristics efficiently, making them well-suited for federated learning on devices with limited computational resources.

\item Other modalities: For other data types, including accelerometer and gyroscope, we process the raw data directly. 
This approach simplifies the pipeline while retaining the necessary information for multimodal federated learning.
\end{itemize}

\subsection{Neural Network Architecture for the Models used in Numerical Experiments}\label{sec: model details}
\subsubsection{CNN for \cifar}
We utilize a convolutional neural network (CNN) for image-based classification tasks on \cifar dataset. 
The architecture consists of two convolutional layers, each followed by a ReLU activation function and max-pooling layers to extract spatial features. 
Finally, fully connected layers are used to map the features to class scores. The configuration of each layer is as follows:
\begin{itemize}
    \item Convolutional layer 1: Input channels = 3, output channels = 16, kernel size = 3, padding = 1.
    \item Convolutional layer 2: Input channels = 16, output channels = 32, kernel size = 3, padding = 1.
    \item Fully-connected layer 1: Input features = 32 * 8 * 8, output features = 128.
    \item Fully-connected layer 2: Input features = 128, output features = 10 (corresponding to the 10 \cifar classes). 
    For the convolutional layers, we apply a max-pooling operation with a kernel size of 2 and a stride of 2 after each ReLU activation function. 
    Both fully connected layers also apply ReLU as the activation function.
    \item Output personalization: In this model, the personalization is applied to the fully-connected layer 1 and fully-connected layer 2.
\end{itemize}

\subsubsection{ImageTextClassifier for \crisis}
For multimodal input combining images and text, we use a model that incorporates both CNN and RNN components. 
The image features are projected through a linear layer followed by ReLU, while text data is processed through a GRU\cite{Cho2014Learning}. 
For certain configurations, we use attention-based fusion to allow the model to combine the multimodal input. 
The detailed layer structure is as follows:
\begin{itemize}
    \item Image projection layer: Input dimension = 1280, output dimension = 64, followed by ReLU and Dropout.
    \item Text processing layer: GRU with input size = 512, hidden size = 64, number of layers = 1, bidirectional = false.
    \item Attention module: we use the attention module with d\_hid = 64 and d\_head = 4 to enable the model to attend to crucial multimodal features.
    \item Classifier head: after fusing image and text features, the classifier consists of two layers with an output size of 64 units and 10 classes.
    \item Input personalization: in this model, the Image Projection Layer and Text Processing Layer are personalized, adapting the input to the specific multimodal features of the dataset.
\end{itemize}

\subsubsection{HARClassifier for \kuhar}
For \kuhar dataset, which consists of accelerometer and gyroscope data, we employ a convolutional neural network (Conv1D) and GRU-based model. 
The architecture is structured to process both sensor modalities separately before fusing the features for classification. 
We use the attention-based fusion mechanism to combine information from both sensor types. 
The model’s layers are configured as follows:
\begin{itemize}
    \item Accelerometer convolutional layers: Three 1D convolutional layers with the following configurations:
    
        Conv1: input channels = 3, output channels = 32, kernel size = 5, padding = 2.

        Conv2: input channels = 32, output channels = 64, kernel size = 5, padding = 2.
        
        Conv3: input channels = 64, output channels = 128, kernel size = 5, padding = 2.
    \item Gyroscope convolutional layers: same configuration as accelerometer layers.
    \item RNN layers: separate GRUs for both accelerometer and gyroscope data with input size = 128 and hidden size = 128.
    \item Attention module: We apply the attention module to fuse accelerometer and gyroscope features with d\_hid = 128 and d\_head = 6.
    \item Classifier head: we combine features from accelerometer and gyroscope data into two layers, followed by 64 units and an output size of 10 classes.
    \item Split input personalization: the personalization in this model applies to the accelerometer convolutional layer. 
    
\end{itemize}

\subsubsection{MMActionClassifier for \crema}
For audio-visual data from the \crema dataset, we employ a multimodal classifier that processes both audio and video features. 
The architecture combines 1D convolutional layers for the audio modality and GRU layers for video frames, along with the attention-based fusion mechanisms. 
The detailed layer configurations are as follows:
\begin{itemize}
    \item Audio convolutional layers: three Conv1D layers to process audio features:
        
        Conv1: output channels = 80, output channels = 32, kernel size = 5, padding = 2.
        
        Conv2: output channels = 32, output channels = 64, kernel size = 5, padding = 2.
        
        Conv3: input channels = 64, output channels = 128, kernel size = 5, padding = 2.

    \item Video RNN layer: GRU with input size = video\_input\_dim, hidden size = 128, number of layers = 1.
    \item Audio RNN layer: GRU with Input size = 128, hidden size = 128, number of layers = 1.
    \item Attention module: we use the attention module with d\_hid = 128 and d\_head = 6 to combine audio and video modalities.
    \item Classifier head: after concatenating audio and video features, a fully connected layer is used to predict one of the 10 classes.
    \item Input personalization: the personalization in this model applies to the input layers, specifically the Audio Convolutional Layers, video RNN Layer, and audio RNN Layer.

\end{itemize}

\subsection{Parameter Settings for the Algorithms}\label{appendix: Parameter Settings for the Algorithms}
For each algorithm, we tune various hyperparameters by selecting several candidate values. We set the fraction of clients selected in each round to 30\%. 
Since all methods use SGD as the solver, we set a fixed number of 3 epochs for each method. 
The learning rate is selected from the set $\{0.05, 0.1, 0.5, 1\}$. 
For both \fedapm and \fedprox, the penalty parameter $\rho$ is selected from the set $\{0.001, 0.01, 0.02, 0.05, 0.1\}$. 
We list and offer a brief description of these methods below.
\begin{itemize}    
    \item \fedavg \cite{mcmahan2017communication}: which learns a shared model by averaging the model updates computed locally in each communication round. It is a baseline algorithm in the FL literature. However, it is likely to be affected by the statistical heterogeneity among clients.
    \item \fedprox \cite{Li2020fedprox}: which introduces a proximal regularization term in each client's objective function, which helps balance local training on local data while maintaining proximity to the global model. 
    This approach ensures that updates do not deviate significantly from the global model, leading to improved overall performance in heterogeneous settings.
    \item \fedalt \cite{Krishna2022Partial}: which is a partial model personalization framework in FL, utilizes Gauss-Seidel iteration to solve for personalized and shared models.
    In this approach, clients first update their personal parameters while keeping the received shared parameters fixed, and then they update the shared parameters while fixing the new personal parameters.
    \item \fedsim \cite{Krishna2022Partial}: which is similar to \fedalt, differs primarily in that it employs Jacobi iteration to solve for the personalized and shared models.
    During each local iteration, both the shared and personalized parameters are updated simultaneously.
\end{itemize}

\subsection{Implementations on \fedapm and Comparison Methods}\label{appendix: Implementations on fedapm and Comparison Methods}
Our algorithms were executed on a computational platform comprising two Intel Xeon Gold 5320 CPUs with 52 cores, 512 GB of RAM, four NVIDIA A800 with 320 GB VRAM, and operating on the Ubuntu 22.04 environment. The software implementation was realized in Python 3.9 and Pytorch 1.9 and open-sourced (\textcolor{black}{\href{https://anonymous.4open.science/r/FedAPM-A921}{https://anonymous.4open.science/r/FedAPM-A921}}).

\subsection{Complete Results}\label{sec: complete results}

\subsubsection{Comparison of Multiple Methods}
Figure \ref{fig: cifar10_crisis_mmd_ku_har_crema_d_training_loss} presents the complete results for the variation in training loss over communication rounds.
It can be observed that \fedapm converges to a lower loss at a faster rate.
\begin{figure}[t]
  \centering
  \includegraphics[width=\linewidth]{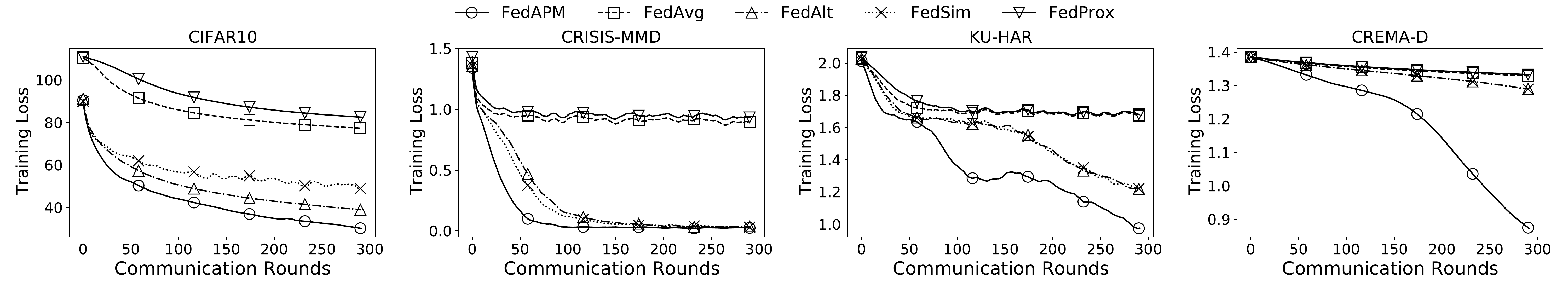}
  \caption{Comparison of training loss across various methods.}
  \label{fig: cifar10_crisis_mmd_ku_har_crema_d_training_loss}
\end{figure}

\subsubsection{Complete Comparison of Multiple Parameters}
Figure \ref{fig: cifar10_crisis_mmd_ku_har_crema_d_FedAPM_loss_varying_rho} presents the complete results for the variation of \fedapm's training loss across communication rounds under different values of $\rho$. It can be observed that when $\rho$ is set to 0.01, \fedapm converges the fastest, while selecting either too large ($\rho=0.1$) or too small ($\rho=0.001$) values reduces the convergence performance of the algorithm.
In Figure \ref{fig: cifar10_crisis_mmd_ku_har_crema_d_FedAPM_loss_varying_frac}, the complete results are presented, showing how the training loss of \fedapm evolves across communication rounds when different numbers of clients are selected for participation in training. It can be observed that selecting more clients accelerates the convergence of the algorithm.
\begin{figure}[t]
  \centering
  \includegraphics[width=\linewidth]{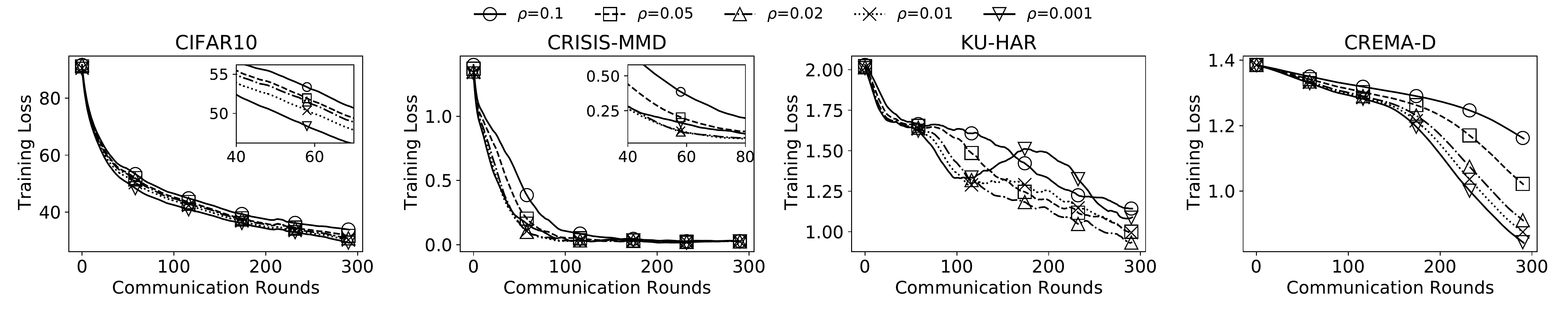}
  \caption{Convergence v.s. penalty parameter.}
  \label{fig: cifar10_crisis_mmd_ku_har_crema_d_FedAPM_loss_varying_rho}
  \includegraphics[width=\linewidth]{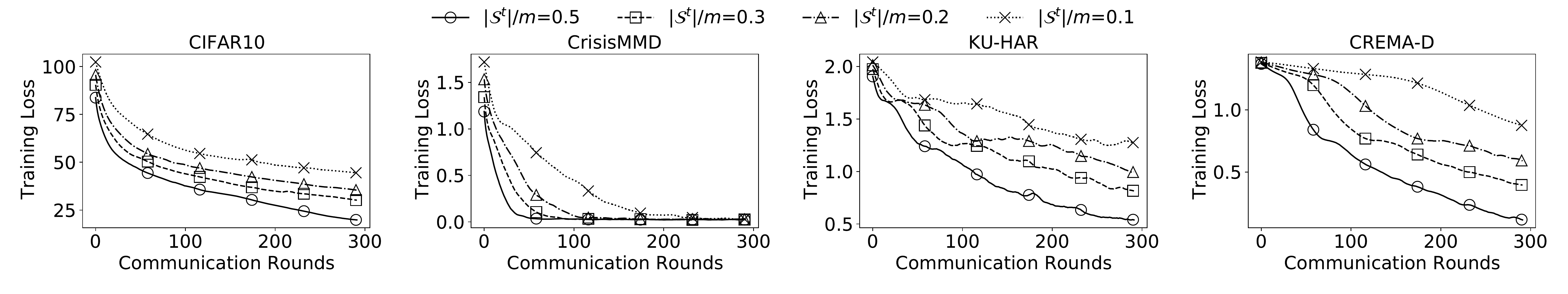}
  \caption{Convergence v.s. client selection.}
  \label{fig: cifar10_crisis_mmd_ku_har_crema_d_FedAPM_loss_varying_frac}
\end{figure}

\subsubsection{Additional Experiments Suggested by the Reviewers}
In response to the reviewers' suggestions, we have expanded our experimental evaluation by including additional baselines and datasets. The new baselines incorporate recent federated learning methods, such as pFedMe \cite{Dinh2020pfedme}, DITTO \cite{Li2021Ditto}, lp-proj \cite{Lin2022Personalized}, FedADMM \cite{Zhou2023FedADMM}, FedALA \cite{Zhang2023FedALA}, FedCP \cite{Zhang2023FedCP}, and FedROD \cite{Chen2022Bridging}.

Additionally, we introduced two new datasets, CIFAR100 and TinyImageNet trained on ResNet18. These modifications allow for a broader and more thorough comparison across different FL methods.

\begin{table*}[t]
\caption{Overall comparison of various methods.}
\label{table: Overall comparison of various methods}
\vspace{-1em}
\resizebox{\textwidth}{!}{
\begin{tabular}{cccccccccccccccccccccccc}
\toprule
\multirow{2}{*}{\textbf{Methods}} & \multicolumn{3}{c}{\textbf{CIFAR10}} &  & \multicolumn{3}{c}{\textbf{CrisisMMD}} &  & \multicolumn{3}{c}{\textbf{KU-HAR}} &  & \multicolumn{3}{c}{\textbf{Crema-D}} &  & \multicolumn{3}{c}{\textbf{CIFAR100}} &  & \multicolumn{3}{c}{\textbf{TinyImageNet}} \\ \cmidrule{2-4} \cmidrule{6-8} \cmidrule{10-12} \cmidrule{14-16} \cmidrule{18-20} \cmidrule{22-24} 
                         & \textbf{ACC}     & \textbf{F1}      & \textbf{AUC}     &  & \textbf{ACC}      & \textbf{F1}       & \textbf{AUC}     &  & \textbf{ACC}     & \textbf{F1}      & \textbf{AUC}    &  & \textbf{ACC}     & \textbf{F1}      & \textbf{AUC}     &  & \textbf{ACC}      & \textbf{F1}      & \textbf{AUC}     &  & \textbf{ACC}       & \textbf{F1}        & \textbf{AUC}      \\ \midrule
\texttt{FedAPM}                 & .738    & .725    & .873    &  & .357     & .294     & .519    &  & .437    & .396    & .595   &  & .651    & .590    & .695    &  & .328     & .319    & .606    &  & .246      & .184      & .588     \\
\texttt{FedAvg}                   & .435    & .345    & .602    &  & .374     & .288     & .510    &  & .142    & .050    & .464   &  & .434    & .295    & .604    &  & .190     & .131    & .562    &  & .129      & .135      & .459     \\
\texttt{FedAlt}                   & .672    & .649    & .839    &  & .364     & .289     & .513    &  & .196    & .105    & .562   &  & .444    & .304    & .642    &  & .317     & .292    & .592    &  & .231      & .183      & .582     \\
\texttt{FedSim}                   & .592    & .565    & .808    &  & .364     & .291     & .514    &  & .196    & .105    & .566   &  & .444    & .304    & .643    &  & .304     & .285    & .588    &  & .154      & .119      & .543     \\
\texttt{FedProx}                  & .406    & .307    & .571    &  & .380     & .291     & .507    &  & .143    & .049    & .453   &  & .438    & .297    & .600    &  & .188     & .128    & .560    &  & .130      & .134      & .458     \\
\texttt{FedADMM}                  & .307    & .312    & .709    &  & .382     & .294     & .505    &  & .220    & .135    & .342   &  & .541    & .486    & .697    &  & .238     & .191    & .586    &  & .162      & .130      & .464     \\
\texttt{pFedMe-G}                 & .228    & .209    & .663    &  & .372     & .315     & .502    &  & .259    & .168    & .301   &  & .403    & .272    & .525    &  & .348     & .342    & .584    &  & .190      & .146      & .553     \\
\texttt{pFedMe-P}                 & .416    & .303    & .642    &  & .462     & .307     & .501    &  & .258    & .159    & .303   &  & .421    & .287    & .531    &  & .307     & .286    & .647    &  & .185      & .134      & .479     \\
\texttt{DITTO-G}                  & .512    & .558    & .835    &  & .378     & .330     & .502    &  & .260    & .163    & .317   &  & .417    & .276    & .587    &  & .364     & .378    & .612    &  & .192      & .156      & .539     \\
\texttt{DITTO-P}                  & .649    & .632    & .809    &  & .471     & .324     & .532    &  & .220    & .126    & .315   &  & .422    & .285    & .589    &  & .310     & .274    & .536    &  & .197      & .144      & .456     \\
\texttt{lp-proj}                  & .629    & .620    & .799    &  & .348     & .293     & .501    &  & .242    & .139    & .347   &  & .488    & .371    & .640    &  & .338     & .290    & .572    &  & .205      & .191      & .546     \\
\texttt{FedALA}                   & .652    & .684    & .833    &  & .443     & .292     & .501    &  & .271    & .189    & .322   &  & .513    & .420    & .655    &  & .366     & .379    & .574    &  & .225      & .201      & .566     \\
\texttt{FedRoD}                   & .634    & .661    & .810    &  & .437     & .302     & .511    &  & .240    & .144    & .302   &  & .525    & .443    & .691    &  & .363     & .366    & .580    &  & .191      & .151      & .442     \\
\texttt{FedCP}                    & .668    & .692    & .839    &  & .470     & .325     & .533    &  & .277    & .191    & .323   &  & .577    & .535    & .694    &  & .391     & .380    & .612    &  & .247      & .192      & .593     \\ \bottomrule
\end{tabular}}
\end{table*}

\end{document}